\documentclass{article}
\usepackage[utf8]{inputenc}
\usepackage[T1]{fontenc}

\usepackage[main,final]{neurips_2026}

\usepackage{microtype}
\usepackage{graphicx}
\graphicspath{{figures/}}
\usepackage{subcaption}
\usepackage{float}
\newcommand{\gridw}{0.32\linewidth}
\usepackage{booktabs} %
\usepackage{titletoc}
\usepackage{paralist}
\usepackage{url}
\usepackage{nicefrac}
\usepackage{xcolor}
\usepackage{hyperref}
\hypersetup{hidelinks}
\definecolor{codeblue}{HTML}{003399}

\usepackage{algorithm}
\usepackage{algorithmic}
\usepackage{bbm}
\usepackage{bm}
\usepackage{enumitem}
\usepackage{mathrsfs}

\newcommand{\ab}{\mathbf{a}}

\newcommand{\vb}{\mathbf{v}}
\newcommand{\wb}{\mathbf{w}}
\newcommand{\xb}{\mathbf{x}}
\newcommand{\ubb}{\mathbf{u}}
\newcommand{\yb}{\mathbf{y}}
\newcommand{\zb}{\mathbf{z}}

\newcommand{\poly}{\text{poly}}

\newcommand{\bb}[1]{\mathbb{#1}}

\newcommand{\wh}[1]{\widehat{#1}}

\newcommand{\bfa}[1]{\boldsymbol{#1}}

\DeclareMathAlphabet{\pazocal}{OMS}{zplm}{m}{n}
\newcommand{\ca}[1]{\pazocal{#1}}

\newcommand{\nnb}{\nonumber}

\newcommand{\wha}[1]{\wh{\bfa #1}}

\newcommand{\lef}{\left}
\newcommand{\rig}{\right}

\newcommand{\tda}[1]{\tilde{\bfa{#1}}}

\newcommand{\vect}{\text{vec}}

\newcommand{\pta}{\partial}

\newcommand{\mbbm}[1]{\mathbbm{#1}}

\newcommand{\tde}{\tilde}

\newcommand{\la}{\langle}
\newcommand{\ra}{\rangle}

\newcommand{\ub}[1]{\underbrace{#1}}

\newcommand{\tr}{\text{Tr}}

\newcommand{\argmin}{\mathop{\mathrm{argmin}}}

\usepackage{amsmath}
\usepackage{amssymb}
\usepackage{mathtools}
\usepackage{amsthm}

\usepackage[capitalize,noabbrev]{cleveref}

\theoremstyle{plain}
\newtheorem{theorem}{Theorem}[section]

\newtheorem{lemma}[theorem]{Lemma}

\theoremstyle{definition}
\newtheorem{definition}{Definition}
\newtheorem{assumption}{Assumption}
\theoremstyle{remark}
\newtheorem{remark}{Remark}
\newtheorem{condition}{Condition}

\title{A Theory on Flow Matching with Neural Networks}

\author{%
  Yihan He\thanks{Equal contribution.}\\
  Princeton University\\
  \texttt{yihan.he@princeton.edu} \\
  \And
  Qishuo Yin\footnotemark[1]\\
  Princeton University\\
  \texttt{qy1448@princeton.edu} \\
  \And
  Yuan Cao\\
  The University of Hong Kong\\
  \texttt{yuancao@hku.hk}\\
  \And
  Jianqing Fan\\
  Princeton University\\
  \texttt{jqfan@princeton.edu}\\
  \And
  Han Liu\\
  Northwestern University\\
  \texttt{hanliu@northwestern.edu}
}

\begin{document}

\maketitle

\begin{abstract}
In this work, we develop theoretical foundation for flow matching with neural-network–parameterized conditional velocity fields. We establish convergence guarantees for gradient descent in the over-parameterized 2-layered ReLU neural network regime. We derive generalization bounds for the conditional velocity-field matching objective.  Building on these results, we provide Wasserstein-distance guarantees for the samples generated by the induced flow. Our analysis is based on generalization bound for multi-task representation learning with unbounded losses, which may be of independent interest beyond flow-based generative modeling. These theoretical results are validated through extensive experiments on both synthetic and real-world image benchmarks. Our code is available at \textcolor{codeblue}{\url{https://github.com/qishuoyin/FlowMatchNN}}.

\end{abstract}

\section{Introduction}\label{sect:1}
Flow matching \citep{lipman2022flow,albergo2022building} has emerged as a state-of-the-art approach within diffusion-based generative modeling for learning and sampling high-dimensional data, including images \citep{esser2024scaling}, and has been applied broadly across a range of other domains \citep{lipman2024flow}.

Flow matching is attractive in part because it cleanly separates three design choices: (i) the \emph{path} of intermediate distributions (e.g., linear interpolation in data space / straight-line couplings in rectified flow \citep{liu2022rectifiedflow}, variance-preserving/variance-exploding diffusion schedules \citep{ho2020ddpm,song2021scoreSDE}, or transport maps motivated by optimal transport \citep{lipman2023flowmatching,tong2024minibatchot}), (ii) the \emph{target velocity} associated with that path \citep{lipman2023flowmatching,tong2024minibatchot}, and (iii) the \emph{numerical solver} used at sampling time \citep{chen2018neuralode,lipman2023flowmatching,song2021scoreSDE}. This modularity has led to a growing ecosystem of variants, including conditional flow matching constructions that leverage conditional distributions and couplings to define tractable target velocities \citep{tong2024minibatchot}. Empirically, these methods have been applied to high-dimensional generation tasks such as images, audio, and video \citep{lipman2023flowmatching,lipman2024flowmatchingguide,guan2024lafma,jin2025pyramidalflow}, and have also been used for conditional generation and editing, where one integrates an ODE while enforcing constraints (e.g., classifier guidance, inverse problems, or conditioning variables) \citep{dhariwal2021diffusionbeatgans,ho2022classifierfree,song2022inversemedical,chung2023dps}. Beyond generation, learned continuous-time transports are increasingly used as primitives for data interpolation, domain translation, and simulation of complex dynamical systems \citep{liu2022rectifiedflow,tong2024minibatchot,chen2018neuralode}.

Despite this empirical success, a theoretical understanding of flow matching remains incomplete. There are three intertwined sources of error that ultimately determine the quality of generated samples: (1) \emph{optimization error} from training neural networks with nonconvex objectives using first-order methods, (2) \emph{statistical error} from learning the velocity field from finitely many samples, and (3) \emph{sampling error} from approximating the continuous-time dynamics with a neural network learned velocity field.

Hence, the questions that we address in this work can be summarized as follows:
\begin{compactenum}
    \item \emph{What is the convergence guarantee for the neural network training with gradient descent?}
    \item \emph{What is the generalization guarantee for the velocity field matching ?}
    \item \emph{What is the distributional guarantee for samples generated by neural-network-based flow matching?}
\end{compactenum}

\paragraph{Contributions.} To address the three questions above, we summarize our main contributions as follows.
\begin{compactenum}
\item We establish convergence guarantees for approximating high-dimensional conditional velocity fields using gradient-descent-trained two-layer neural networks. Our analysis connects the training dynamics to the neural tangent kernel (NTK) regime \citep{du2017gradient,jacot2018neural} and addresses the challenges posed by high dimensionality, including regimes where the data dimension $d$ is of the same order as the sample size $n$.
\item We derive $L_2$ statistical error bounds for the velocity field learning problem by introducing a general high-dimensional multi-task learning generalization bound for unbounded losses. This result complements existing representation learning theory \citep{maurer2016vector,maurer2016benefit} and may be of independent interest.
\item We provide sampling error guarantees for flow matching in Wasserstein distance. The proof combines a geometric lemma that upgrades $L_1$ control to $L_\infty$ control on a compact high-dimensional subset with a refined analysis of the induced ODE dynamics. To the best of our knowledge, this is the first sampling error guarantee for gradient-descent-trained, neural-network-based flow matching algorithms.
\item Our theoretical results are complemented by numerical experiments on both the synthetic and real-world images that validate the convergence and sampling error bounds. These experiments provide insights into the practical implications of our theoretical findings and highlight potential avenues for future research.
\end{compactenum}

\paragraph{Organization.} The remainder of this work is organized as follows. Section \ref{sect:2} reviews related work and positions our contributions relative to existing literature. Section \ref{sect:3} introduces the main definitions and summarizes our results. Section \ref{sect:7} reports numerical experiments that complement the theoretical analysis. Finally, Section \ref{sect:8} discusses limitations and directions for future work.

\section{Related Works}\label{sect:2}
Our work is closely related to two branches of recent theoretical research on generative modeling and overparameterized neural networks. In the following, we first review recent theoretical developments for flow matching models, and then discuss the techniques most directly related to those used in this work.

\paragraph{Theoretical works on generative modeling algorithms.}
Recent works \citep{fukumizu2024flow,su2025high,su2025theoretical} study statistical convergence rates for variants of flow matching. In particular, \citet{fukumizu2024flow} shows that flow matching achieves near-optimal rates in the $p$-Wasserstein distance, comparable to diffusion models. However, their estimator is not neural-network-based, whereas neural networks are the primary object of study in our work. \citet{su2025high,su2025theoretical} investigate the learning error of variants of neural-network-based flow matching, but their analysis focuses on the theoretical limits of certain neural network classes and does not address achievability by a concrete optimization procedure. In contrast, the neural network solution we study is achievable via gradient descent with polynomial-time complexity.
\citet{han2024neural} analyzes optimization and generalization for score matching in diffusion models using shallow neural networks. Their work does not provide distributional guarantees for samples generated from the learned score. Moreover, their two-layer network treats only the first layer weights $\{\wb_r\}_{r\in[n]}$ as trainable parameters, whereas our analysis allows weights from both the first and second layers (a.k.a. $\{\ab_r\}_{r\in[n]}$ and $\{\wb_r\}_{r\in[n]}$) to be trainable. This distinction substantially complicates the analysis and renders several arguments in \citet{han2024neural} inapplicable in our setting.

Overall, much of the existing theory for flow matching is largely algorithmically agnostic, leaving a gap between statistical guarantees and practical learning procedures with polynomial-time complexity. In comparison, our analysis provides performance guarantees for gradient descent and bridges this gap by establishing algorithmic achievability.

\paragraph{Convergence and generalization guarantees for overparameterized neural networks.}
A substantial body of work studies training dynamics and generalization in the overparameterized regime. Part of our analysis is motivated by recent advances connecting trained overparameterized neural networks to the neural tangent kernel (NTK) \citep{jacot2018neural}. In particular, a series of works establish convergence guarantees for gradient descent when training neural networks in the NTK regime \citep{li2018learning,du2018gradient,allen2018convergence,du2018gradientdeep,chizat2018note,zou2019gradient,zou2019improved,lee2019wide}. Closely related studies also develop NTK-based generalization bounds for overparameterized neural networks \citep{allen2018learning,arora2019fine,arora2019exact,cao2019generalizationsgd,ji2020polylogarithmic,chen2021much}.
Another line of work leverages mean-field techniques to analyze the training of overparameterized shallow networks \citep{mei2018mean,chizat2018global,mei2019mean,wei2018regularization}, deep networks \citep{lu2020mean,fang2021modeling,ding2021meanfield1,ding2022meanfield2}, and transformers \citep{gao2024global}. While these studies offer valuable insights and theoretical tools for deep learning, they are largely developed for standard supervised learning settings. In comparison, our work studied the generative modeling problem.

\section{Major Results}\label{sect:3}
This section presents the formal statements of the main results summarized in Section \ref{sect:1}. In Section \ref{prelim}, we provide additional details on the problem formulation, key definitions, and the assumptions used throughout this work. In Section \ref{mainthms}, we state our three main theorems.
\paragraph{Notations.}
Throughout, we use the following notation conventions. Vector-valued quantities are denoted by boldfaced characters. We write $[n]:={1,\ldots,n}$ and $[i:j]:={i,i+1,\ldots, j}$ for $i<j$. Universal constants are denoted by $C$ and may change from line to line. For a vector $\bfa v$, we denote its $\ell_2$ norm by $\Vert\bfa v\Vert_2$. For a matrix $\bfa A\in\bb R^{m\times n}$, we denote its operator norm by $\Vert\bfa A\Vert_2:=\sup_{\bfa v\in\bb S^{n-1}}\Vert\bfa A\bfa v\Vert_2$. Given two sequences $a_n$ and $b_n$, we write $a_n\lesssim b_n$ or $a_n=O(b_n)$ if $\limsup_{n\to\infty}\big|\frac{a_n}{b_n}\big|<\infty$, and $a_n=o(b_n)$ if $\limsup_{n\to\infty}\big|\frac{a_n}{b_n}\big|=0$. We use $\mbbm 1$ to denote the indicator function.
\subsection{Preliminaries, Definitions, and Assumptions}\label{prelim}
In this section, we begin by introducing the problem setup and some preliminaries. Then we introduce the assumptions made in this work.
\subsubsection{Problem Setup.} In the generative modeling problem, we are given a dataset $\{\xb_{1,i}\}_{i\in[n]}$ consisting of $n$ data points in $\bb R^d$, assumed to be drawn from a high-dimensional distribution $\bb P_1$. The goal is to generate additional samples $\wh \xb_1$ whose distribution is close to $\bb P_1$ using $\{\xb_{1,i}\}_{i\in[n]}$. We assume that the $\xb_{1,i}$ are i.i.d.
\paragraph{Flow Matching.} Flow matching is a state-of-the-art approach to generative modeling. It enables sampling without explicitly estimating the target distribution $\bb P_1$. Instead, it specifies a path of intermediate distributions that connects a tractable, easy-to-sample base distribution $\bb P_0$ (e.g., a Gaussian) to the target data distribution $\bb P_1$. Concretely, one draws a collection of i.i.d. samples $\{\xb_{0,i}\}_{i\in[n]}$ from $\bb P_0$ as source data. The objective of flow matching is to learn a mapping $\ubb$ that satisfies $\ubb(\xb_0)\sim\bb P_1$ given $\xb_0\sim \bb P_0$
through interpolating the two datasets $\{\xb_{0,i}\}_{i\in[n]}$ and $\{\xb_{1,i}\}_{i\in[n]}$ as
$    \ubb(\xb_{0,i}) = \xb_{1,i}$, $ \forall i\in[n]$.

The {flow matching} algorithm learns a transport path between $\xb_{0,i}$ and $\xb_{1,i}$, $i\in[n]$ by constructing a conditional velocity field $\vb_{\tau}\in\bb R^d$ such that, for all $\tau\in[0,1]$,
\begin{align}\label{ode}
    &\frac{d\xb_{\tau}}{d\tau}=\vb_{\tau}(\xb_{\tau}|\xb_0) := \sigma_{\tau}^\prime\cdot\xb_0 + \frac{\mu_{\tau}^\prime}{\mu_{\tau}}(\xb_{\tau} - \sigma_{\tau}\xb_0),
\end{align}
where $\sigma$ and $\mu$ satisfying $\sigma_0 = \mu_1 = 1$, $\sigma_1 = \mu_0 = 0$ are scalar functions with derivatives $\sigma^\prime$ and $\mu^\prime$, respectively. We emphasize that the conditional velocity field $\vb$ also depends on $\xb_0$ and $\tau$.
Under this parametrization, the target mapping $\ubb$ can be obtained by solving the ODE in \eqref{ode} to produce $\xb_1$ from the initial condition $\xb_0$. In practice \citep{lipman2022flow,lipman2024flowmatchingguide}, the velocity field $\vb_{\tau}$ is frequently represented by a neural network. In this work, we consider the following overparameterized two-layer neural network:
\begin{align*}
    \bfa f_{\bfa\theta}(\xb_{\tau},\tau,\xb_0)=\frac{1}{\sqrt m}\sum_{r=1}^m\ab_r\sigma\lef(\wb_r^\top[\xb_{\tau}^\top,\tau,\xb_0^\top]^\top\rig)
\end{align*}
where $m$ denotes the width of the neural network. For $r\in[m]$, the neural network parameters $\ab_r$ and $\wb_r$ lie in $\bb R^d$ and $\bb R^{2d+1}$ respectively. The nonlinearity $\sigma$ is taken to be the ReLU activation. We collect all network parameters into $\bfa\theta =[\ab_1^{\top},\ldots,\ab_m^{\top},\wb_1^\top,\ldots,\wb_m^\top]^\top$.
We emphasize that this architecture is typically studied in the \emph{overparameterized} regime, where $m$ can be much larger than $n$.
To train the neural network on the samples $\{\xb_{0,i}\}_{i\in[n]}$, we minimize the following empirical loss after drawing $\tau_i\overset{\mathrm{i.i.d.}}{\sim}\mathrm{Uniform}([0,1])$:
\small
\begin{align}\label{loss}
    L(\bfa\theta)=\frac{1}{n}\sum_{i=1}^n\lef\Vert\bfa f_{\bfa\theta}(\xb_{\tau_i,i},\tau_i,\xb_{0,i})-\vb_{\tau_i}(\xb_{\tau_i,i}|\xb_{0,i})\rig\Vert_2^2.
\end{align}
\normalsize
In this work, we consider the standard \emph{gradient descent} algorithm as the optimization procedure, which is given by
\begin{align*}
    \bfa\theta^{(t+1)} = \bfa\theta^{(t)} - \eta\cdot\nabla_{\bfa\theta}L(\bfa\theta^{(t)}),
\end{align*}
where $\eta$ is the step size.
The network parameters $\bfa\theta^{(0)}$ are randomly initialized. In particular, we set
\begin{align*}
    \wb_r^{(0)}\overset{i.i.d.}{\sim} N(0, \kappa^2I_{2d+1}),\qquad\ab_r^{(0)}\overset{i.i.d.}{\sim} N(0,I_d).
\end{align*}
where $\kappa\in\bb R^+$ controls the magnitude of the network parameters at initialization.
Given the gradient-descent-trained neural network model $\bfa f_{\wha\theta}$ with parameters $\wha\theta$, we generate a new random variable $\wh\xb_1$ by first sampling $\wh\xb_0\sim \bb P_0$ and then solving the following ODE using standard numerical solvers \citep{butcher2016numerical}
\begin{align}\label{odee}
    \frac{d\wh\xb_\tau}{d\tau} = \bfa f_{\wha\theta}\lef(\wh\xb_{\tau},\tau,\wh\xb_{0}\rig).
\end{align}

\begin{algorithm}[H]
\caption{Data generating procedure}
\begin{algorithmic}[1]\label{alg:data}
    \STATE \textbf{input:} number of samples $n$, dataset $\{\xb_{1,i}\}_{i\in[n]}$
    \FOR {$i=1,2,\ldots, n$}
        \STATE Sample $\tau_i \sim \mathrm{Uniform([0,1])}$
        \STATE Sample $\xb_{0,i}\sim \bb P(\xb_0\mid\xb_{1,i})$.
        \STATE Compute $\xb_{\tau,i} = \sigma_{\tau_i}\cdot\xb_{0,i} +\mu_{\tau_i}\cdot\xb_{1,i}$
    \ENDFOR
    \STATE \textbf{return} $\{(\tau_i, \xb_{\tau_i,i},\xb_{0,i})\}_{i=1}^n$
\end{algorithmic}
\end{algorithm}

\begin{remark}
The flow matching pipeline considered in this work comprises two components: Algorithm~\ref{alg:data} specifies the data generation procedure, and Algorithm~\ref{alg:learning} specifies the corresponding learning and sampling procedure. A key methodological choice in Algorithm~\ref{alg:data} is that we draw the source sample $\xb_0$ from the conditional distribution $\bb P(\xb_0\mid \xb_1)$, rather than from its marginal $\bb P_0$ as is customary in standard flow matching formulations \citep{lipman2022flow,liu2022rectifiedflow}. In the experiments, we use $\xb_0\mid\xb_1\sim N(\xb_1,I_d)$: each source sample is its paired data point plus independent standard Gaussian noise. This conditional coupling is adopted for theoretical reasons: it renders the population-level objective identifiable, in the sense that the target conditional velocity field $\vb_\tau$ is the unique minimizer of the expected loss. For a general coupling, squared-loss regression identifies the conditional mean of the velocity label given the network input; recovering the intended field requires checking identifiability. Independent couplings therefore require a separate population analysis. Minibatch-OT matching additionally introduces dependence between matched pairs, beyond our i.i.d. argument, although the fixed-sample optimization analysis can be reused when its kernel conditions hold.
\end{remark}

Along the matched interpolation, either endpoint can be recovered from the other and $(\xb_\tau,\tau)$ whenever the corresponding schedule coefficient is nonzero. The endpoint-conditioned velocity expressions therefore agree along this path away from singular schedule endpoints; our network takes $(\xb_\tau,\tau,\xb_0)$ as input.

\paragraph{Neural Tangent Kernel.}
 The neural tangent kernel (NTK) \citep{jacot2018neural} is a widely used framework that connects overparameterized neural networks to reproducing kernel Hilbert spaces (RKHSs). We define the following Gram matrix for $\wb_r\sim N(0,\kappa^2 I_{2d+1})$:
\begin{align*}
    \bfa H^{(\infty)}_{ij}&:=I_d\bb E\Big[[\xb_{\tau_j}^\top, \tau_j,\xb_0^\top](\wb_r\wb_r^\top + I_{2d+1})[\xb_{\tau_i}^\top,\tau_i,\xb_{0}^\top]^\top\\
&\mbbm 1\{[\xb_{\tau_i}^\top,\tau_i,\xb_{0,j}^\top]\wb_r > 0,[\xb_{\tau_j}^\top,\tau_j,\xb_{0,j}^\top]\wb_r > 0\}\Big],\\
    \bfa H^{(\infty)} &:=\begin{bmatrix}
        \bfa H^{(\infty)}_{i,j}
    \end{bmatrix}_{i,j\in[n]}.
\end{align*}
We note that each block $\bfa H^{(\infty)}_{ij}$ of the Gram matrix can be decomposed as the sum of two matrices, $\bfa B^{(\infty)}$ and $\bfa C^{(\infty)}$, where
\begin{align*}
    \bfa B^{(\infty)}_{ij}&:= I_d\bb E\Big[[\xb_{\tau_j}^\top, \tau_j,\xb_0^\top]\wb_r\wb_r^\top [\xb_{\tau_i}^\top,\tau_i,\xb_{0}^\top]^\top\mbbm 1\{[\xb_{\tau_i}^\top,\tau_i,\xb_{0,j}^\top]\wb_r > 0,[\xb_{\tau_j}^\top,\tau_j,\xb_{0,j}^\top]\wb_r > 0\}\Big],\\
    \bfa C^{(\infty)}_{ij}&:= I_d\bb E\Big[[\xb_{\tau_j}^\top, \tau_j,\xb_0^\top][\xb_{\tau_i}^\top,\tau_i,\xb_{0}^\top]^\top\mbbm 1\{[\xb_{\tau_i}^\top,\tau_i,\xb_{0,j}^\top]\wb_r > 0,[\xb_{\tau_j}^\top,\tau_j,\xb_{0,j}^\top]\wb_r > 0\}\Big].
\end{align*}
We can similarly define $\bfa B^{(\infty)}$ and $\bfa C^{(\infty)}$. The Gram matrix plays a central role in characterizing the optimization landscape of gradient descent. In particular, in the parameter-limit regime $m\to\infty$, it can be interpreted as an effective Hessian for the neural network. The matrices $\bfa B^{(\infty)}$ and $\bfa C^{(\infty)}$ correspond to the respective contributions from the trainable parameters $\wb_r$ and $\ab_r$. We also establish the following result on the minimum eigenvalue of the Gram matrix.
\begin{lemma}[Property of the Gram Matrix]
The Gram Matrix satisfies $\lambda_{\min}(\bfa H^{(\infty)}) >0$.
\begin{remark}
    We note that this is a standard assumption in the NTK literature, and can be proved based on various assumptions on the training data \citet{jacot2018neural,allen2018learning,arora2019fine}. In our problem setup, we can also establish this property rigorously. %
\end{remark}
\end{lemma}
\begin{algorithm}[htbp]
\caption{Learning and Sampling}
\begin{algorithmic}[1]\label{alg:learning}
    \STATE \textbf{input:} augmented datasets $\{\xb_{0,i},\tau_i,\xb_{1,i}\}_{i\in[n]}$, step size $\eta$, some constant $B_1$.
    \STATE Initialize neural network parameters $\bfa\theta^{(0)}$ by
    $\wb_r\overset{i.i.d.}{\sim} N(0, \kappa^2I_{2d+1}),\quad \ab_r\overset{i.i.d.}{\sim} N(0, I_d)$.
    \FOR {$t=1,2,\ldots, T$}
        \STATE Let $\bfa\theta^{(t)}\gets \bfa\theta^{(t-1)} -\eta\nabla_{\bfa\theta^{(t-1)}}L(\bfa\theta^{(t-1)})$ where $L$ is defined in \eqref{loss}.
    \ENDFOR
    \STATE Sample a new random variable $\wh\xb_0$.
        \STATE Using the Euler Method to solve the ODE defined by \eqref{odee} with parameters $\bfa\theta^{(t)}$ and obtain $\wh\xb_1$.
    \STATE Let $\tde B_1 = \sqrt d B_1$
    \STATE Truncate $
        \wh\xb_1 \gets \wh\xb_1\mbbm 1_{\Vert\wh\xb_1\Vert\leq 2 \tde B_1} +\mbbm 1_{\Vert\wh\xb_1\Vert_2 > 2 \tde B_1}\argmin_{\xb:\Vert\xb\Vert_2 = 2\tde B_1}\Vert\xb - \wh\xb_1\Vert_2$.
    \STATE \textbf{return} $\wh\xb_1$.
\end{algorithmic}
\end{algorithm}

\subsubsection{Assumptions.} We introduce a few assumptions considered in this work.
\begin{assumption}[Weak]\label{assump1}
    We assume that the distribution of $\bb P_1$ is sub-Gaussian with
$
        \bb E[\exp(\bfa t^\top\xb_{1,i})]\leq\exp\lef({Cd\Vert\bfa t\Vert_2^2}/{2}\rig).
$\end{assumption}
We note that the above assumption provides tail guarantees for $\Vert\xb_1\Vert_2$ and will be used in the proofs of our generalization and gradient descent upper bounds. This sub-Gaussian assumption is standard in high-dimensional statistics \citep{wainwright2019high}. For the sampling bound, we will require a stronger condition, given by
\begin{assumption}[Strong]\label{assump2}
    We assume that the distribution of $\bb P_1$ has bounded support with $\Vert\xb_1\Vert_2\leq B_1\sqrt d$ for some large constant $B_1$.
\end{assumption}
And for the distribution of $\xb_1$ we need an anti-concentration bound given as follows.
\begin{assumption}[Anti-Concentration]\label{assump3}
    We assume that the random vector $\xb_1\sim\bb P_1$ satisfies
$         \bb E\lef[\frac{1}{\Vert\xb_1\Vert_2}\rig]\lesssim \frac{1}{\sqrt{d}}$ and $ \bb E\Big[\frac{1}{\Vert\xb_1\Vert_2^2}\Big]\lesssim \frac{1}{d}$.

\end{assumption}
Intuitively, the anti-concentration bound ensures that the training samples do not collapse onto a single point as the sample size grows.

For the functions $\mu$ and $\sigma$, this work requires the following upper bounds:
\begin{assumption}\label{assump4}
    We assume that for all $\tau\in[0,1]$, $\max\{|\mu_{\tau}|,|\sigma_{\tau}|,|\mu_{\tau}^\prime|,|\sigma_{\tau}^\prime|\} < B_0 = \Theta( 1 )$.
\end{assumption}
The above assumption ensures that the function $\vb$ is Lipschitz. In the proof, we also show that the Lipschitz property of $\vb$ implies the Lipschitz property of the trained neural network $\bfa f_{\bfa\theta}$.

\subsection{Main Theorems}\label{mainthms}
This section presents the main results of this work. We state three theorems addressing (1) approximation guarantees for gradient descent, (2) generalization bounds for conditional velocity field learning, and (3) sampling error bounds. These results answer the three questions posed in the \emph{Introduction}, respectively.
\paragraph{The Approximation Bound.}
We introduce the following new notations
\begin{align*}
       & \bfa u_i^{(t)}:=\bfa f_{\bfa\theta^{(t)}}\lef(\xb_{\tau_i},\tau_i,\xb_{0,i}\rig),\quad\vb_i:=\vb_{\tau_i}(\xb_{\tau_i}|\xb_{1,i}),\\
        &\bfa U^{(t)}:=[\bfa u_1^{(t),\top},\ldots, \bfa u_n^{(t),\top}]^\top,\quad\bfa V:=[\vb_1^\top,\ldots\vb_n^\top]^\top.
\end{align*}

The prediction residual is
\[\Delta^{(t)}:=\bfa U^{(t)}-\bfa V,\]
where $\bfa U^{(t)}$ stacks the network outputs on the training sample and $\bfa V$ stacks the target conditional velocities.

The optimization mechanism is the general matrix-valued NTK dynamics of multi-output squared-loss regression on the augmented input; flow matching specifies the interpolation and velocity labels. Our two-layer verification controls both parameter blocks and ReLU-gate changes, using bilinearity under fixed activation patterns. Deeper networks require additional representation-stability bounds.
Then, we establish the following theorem characterizing the convergence rate of gradient descent.

\begin{theorem}\label{theorem1}
    Suppose that assumptions \ref{assump1}, \ref{assump3}, and \ref{assump4} hold. Given $m\gtrsim \frac{n^5\Vert\bfa U^{(0)}-\bfa V\Vert_2^2}{\kappa^2\lambda_{\min}(\bfa H^{(\infty)})\delta^3}d\log nd$ , $\eta <\frac{2\lambda_{\min}(\bfa H^{(\infty)})}{n^2d^2} $, and $\kappa=o\lef( \frac{\lambda_{\min}(\bfa H^{(\infty)})\delta}{n\sqrt d}\rig)$, with probability at least $1-\delta$, the prediction error at the $t$-th iteration satisfies
    $\Vert\Delta^{(t)}\Vert_2^2
       \leq\Big(1-\frac{\eta\lambda_{\min}(\bfa H^{(\infty)})}{2}\Big)^{t-1}\Vert\Delta^{(t-1)}\Vert_2^2$
    for some  $C\in(0,1)$.
    Furthermore, we have
    \begin{align*}
        \Delta^{(t)} &= - (I-\eta\bfa H^{(\infty)})^t\bfa V + \bfa E^{(t)},\\
        \Vert\bfa E^{(t)}\Vert_2&\leq C\Big(\frac{n\sqrt d\kappa}{\lambda_{\min}(\bfa H^{(\infty)})\delta} + \frac{n^3d^3}{\sqrt m\delta}\log^3\frac{m}{\delta} +\frac{\eta n^3\log(n/\delta)}{\sqrt m\kappa\lambda_{\min}(\bfa H^{(\infty)})\delta^{3/2}}\Big)
    \end{align*}
     with probability at least $1-\delta$.
\end{theorem}
\begin{remark}
    We establish a linear convergence rate for gradient descent. Intuitively, we show that gradient descent on a multidimensional neural network exhibits a convergence behavior comparable to that of minimizing a quadratic objective of the form $\bfa x^\top \bfa H^{(\infty)} \bfa x$. The second result provides a more fine-grained characterization of the training dynamics by expressing the loss $\Delta^{(t)}$ via an explicit matrix multiplication recursion along the optimization trajectory. A proof sketch is provided in Section \ref{sect:4} in the appendix.
\end{remark}

\paragraph{The Generalization Bound.}
Our main result on generalization guarantees for the velocity field learning loss is summarized in the following theorem.

\begin{theorem}\label{thmgenbound}  Suppose that assumptions \ref{assump1}, \ref{assump3}, and \ref{assump4} hold.
    Let $\bfa\theta^{(t)}$ being the learned parameters of the neural network with $t\geq \frac{C}{\eta\lambda_{\min}(\bfa H^{(\infty)})}\log\frac{n}{\delta}$ iterations of training. Let the choice of parameter $\kappa = O(\frac{\lambda_0\delta}{n})$ and $m\geq \kappa^{-2}\poly(n,\lambda_{\min}(\bfa H^{(\infty)}), \delta^{-1},d)$. Let $\tau\sim\text{Uniform}([0,1])$. With probability at least $1-\delta$ over the random initialization of neural network and the training set, the following holds
    \begin{align*}
        \bb E\Big[\Vert\bfa f_{\bfa\theta^{(t)}}(\xb_{\tau},\tau,\xb_0)-\vb_{\tau}(\xb_{\tau}\mid\xb_0)\Vert_2\Big]&\leq C\sqrt{\bfa V^\top(\bfa H^{(\infty)})^{-1}\bfa V}\sqrt{\frac{(d+B_0^2)\log^2(n/\delta)}{n}}\\
        &+ C\sqrt{\frac{d\log(n/\lambda_{\min}(\bfa H^{(\infty)})\delta)}{n}}.
    \end{align*}
    Moreover, $\sqrt{\bfa V^\top(\bfa H^{(\infty)})^{-1}\bfa V}$ is of order $B_0$, which implies
$
        \bb E\Big[\Vert\bfa f_{\bfa\theta^{(t)}}(\xb_{\tau},\tau,\xb_0)-\vb_{\tau}(\xb_{\tau}\mid\xb_0)\Vert_2\Big]\leq C\sqrt{\frac{d}{n}\log(n/\delta)}$.
\end{theorem}
\begin{remark}
    Our generalization bound implies that, with high probability, there exists a neural network produced by gradient descent whose generalization loss scales as $\sqrt{\frac{d}{n}\log(n/\delta)}$. This rate is analogous to the minimax rates that arise in many high-dimensional regression problems \citep{wainwright2019high}. We introduce the $\poly(\cdot)$ notation to streamline the presentation of higher-order factors.

A key technical difficulty is that our generalization error is measured in an $L_2$ norm and the associated loss is unbounded, which falls outside the scope of many standard learning-theoretic tools developed for bounded losses. To address this issue, we develop a new Rademacher-complexity-based generalization bound tailored to unbounded loss functions. These techniques and the corresponding proof machinery are extensively discussed in Section \ref{sect:5}.
\end{remark}
Theorem~\ref{thmgenbound} controls expected Euclidean-norm error, rather than mean-squared error. Its simplified rate uses the stated target/kernel compatibility and does not imply an architectural advantage over deeper networks.
To facilitate the generalization bound given by Theorem \ref{thmgenbound}, we introduce a new generalization bound for multi-task learning with unbounded losses.

Let $\mathcal F$ be a class of maps from the input space $\mathcal X$ to $\mathbb R^d$, with $f_j$ its $j$-th output coordinate. Here the augmented input is $\widetilde\xb=(\xb_\tau^\top,\tau,\xb_0^\top)^\top\in\mathbb R^{2d+1}$. For Theorem~\ref{thmgenbound}, we use the localized network class
\[
\mathcal F_{R_{\ab},R_{\wb}}=\left\{f_\theta:
\sum_{r=1}^m\|\ab_r-\ab_r^{(0)}\|_2^2\leq R_{\ab}^2,\quad
\sum_{r=1}^m\|\wb_r-\wb_r^{(0)}\|_2^2\leq R_{\wb}^2\right\},
\]
whose radii are supplied by the parameter-localization bounds in Appendix~\ref{sect:5}.

The following empirical multi-task Rademacher complexity is introduced by \citet{maurer2016benefit}
\begin{align*}
    \ca R_{\xb}(\ca F) = \frac{1}{n}\bb E\bigg[\sup_{f\in\ca F}\sum_{i=1}^n\sum_{j=1}^d\epsilon_{ij}f_j(\tde\xb_i)\bigg|\{\tde\xb_i\}_{i\in[n]}\bigg],
\end{align*}
where $\epsilon_{ij}$s are i.i.d. random sign. This complexity has been widely used in multi-task learning settings, where the output of the learning algorithm is multivariate. We also note that in related work \citet{han2024neural}, the authors construct a coupling between kernel-based regression and neural networks, and derive generalization bounds by leveraging existing results for kernel-based regression. In contrast, we adopt a different approach by directly upper bounding the empirical Rademacher complexity of the neural network class.

We establish the following lemma, which provides generalization bounds for unbounded losses and complements the bounded-loss results in \citep{maurer2016vector}.
\begin{lemma}\label{multitasklearn}
    Suppose the loss function $\ell(\cdot,\cdot)$ satisfies $\bb E[\sup_{\bfa f\in\ca F}\cosh(\lambda \ell(f(\xb),y))]\leq\exp(\frac{1}{2}\sigma^2\lambda^2)$ for all $\lambda\in\bb R$. And $\ell(f(\xb),\yb)$ is $\rho$-Lipschitz in the first argument in Euclidean norm. Then with probability at least $1-\delta$ over the samples $\{\xb_{i}\}_{i\in[n]}$,
    \begin{align*}
        \sup_{f\in\ca F}\Big\{\bb E[\ell(f(\xb), \yb)] &- \frac{1}{n}\sum_{i=1}^n\ell(f(\xb_i),\yb_i)\Big\}\leq 2\rho{\ca R}_{\xb}(\ca F)
 + \sigma\sqrt{\frac{2\log(1/\delta)}{n}}.
 \end{align*}
\end{lemma}
\begin{remark}
     We note that the condition of $\bb E[\sup_{\bfa f\in\ca F}\cosh(\lambda\ell(f(\xb),y))]\leq\exp(\frac{1}{2}\sigma^2\lambda^2)$  entails a strong tail control for the supremum of the stochastic process $\ell(f(\xb),y)$. A closely related condition appears in \citet{meir2003generalization}, where the authors study the standard univariate setting. Our bound extends their analysis to the multi-task learning regime.
\end{remark}
\paragraph{The Sampling Bound.}
We next provide an upper bound on the Wasserstein-1 distance for the generated samples $\wh\xb_1$. We first recall the definition of the Wasserstein-1 distance.

\begin{definition}[Wasserstein--1 distance]

For probability measures $\mu,\nu$ on $\mathbb R^d$ with finite first moments,
\[
W_1(\mu,\nu):=\inf_{\gamma\in\Pi(\mu,\nu)}
\int_{\mathbb R^d\times\mathbb R^d}\|x-y\|_2\,d\gamma(x,y),
\]
where $\Pi(\mu,\nu)$ is the set of their couplings. This is equivalent to the Kantorovich--Rubinstein dual formulation.
\end{definition}
Our sampling proof uses $W_1(P_{X_1},P_{\widehat X_1})\leq\mathbb E\|X_1-\widehat X_1\|_2$: the true and learned flows share the initial source $X_0$, so their endpoint joint law is an admissible coupling.

And our formal results are given as follows.

\begin{theorem}\label{thmsampbound}
    Under the same condition as that in theorem \ref{thmgenbound} with an additional assumption \ref{assump2}, we can show that with probability at least $1-\delta$ over the random initialization of neural networks, the training samples and the test samples,
    \begin{align*}
        W_1(\bb P_{\xb_1},\bb P_{\wh\xb_1}) \leq C(d)\Big(\frac{\log(n/\delta)}{n}\Big)^{\frac{1}{2(d+2)}} B_1^{-\frac{d+1}{d+2}},
    \end{align*}
    where $C(d)$ is a constant depending on $d$.
\end{theorem}
\begin{remark}
    The theorem above implies that, provided $B_1=o\left(\frac{n}{\log n}\right)^{\frac{1}{2(d+1)}}$, the Wasserstein-1 error is $o(1)$. Nevertheless, consistency in the $W_1$ metric requires the dimension $d$ to remain bounded, which is also reflected in recent work such as \citet{fukumizu2024flow, su2025high}. This pronounced dependence on $d$ reflects the well-known \emph{curse of dimensionality} in nonparametric distribution estimation \citep{wasserman2006all}. Section \ref{sect:6} in the appendix provides a more detailed discussion of our proof strategy, including the geometric argument that upgrades the average Euclidean-norm error bound in Theorem \ref{thmgenbound} to an $L_{\infty}$ bound.
\end{remark}

\section{Experiments}\label{sect:7}
In this section, we provide empirical evidence supporting Theorems~\ref{theorem1}--\ref{thmsampbound} on synthetic data and real-world image benchmarks.

\subsection{Simulations}\label{sect:7:sim}

\paragraph{Setup.}
The target $\bb P_1$ is a balanced two-component Gaussian mixture in $\bb R^d$ with means $\pm2\bfa e_1$ and identity covariance, and $\bb P(\xb_0\mid\xb_1)=N(\xb_1,\bfa I_d)$ with the linear schedule $\sigma_\tau=1-\tau$, $\mu_\tau=\tau$. The two-layer ReLU $\bfa f_{\bfa\theta}$ is initialized with $\wb_r\sim N(0,\kappa^2\bfa I_{2d+1})$ and $\ab_r\sim N(0,\bfa I_d)$ at $\kappa=1$ and trained by full-batch gradient descent for $T=500$ iterations; sampling uses $T_{\rm Euler}=200$ Euler steps and truncation radius $2\sqrt{d}\,B_1$ with $B_1=10$. We jointly sweep $m\in\{16,32,64,128,256\}$, $\eta\in\{10^{-4},10^{-3},10^{-2}\}$, and $d\in\{5,10,\dots,50\}$ at $n_{\rm train}=500$, $n_{\rm test}=5000$; both the training loss $L(\bfa\theta^{(t)})$ and the empirical sliced Wasserstein-1 diagnostic $\overline{W_1}(\bb P_{\xb_1},\bb P_{\wh\xb_1^{(t)}})$ (Appendix~\ref{app:exp:slicedw1}) are evaluated every $10$ gradient steps. Each configuration is independently repeated over $10$ random initializations; all reported figures display the mean with shaded $\pm1$ standard-deviation bands.

\paragraph{Results.}
Figure~\ref{fig:global} shows the terminal-iterate ($t=T$) metrics across all grid cells. The training loss grows monotonically with $d$ and decreases with $\eta$, consistent with Theorem~\ref{theorem1}: the convergence factor $1-\eta\lambda_{\min}(\bfa H^{(\infty)})/2$ shrinks as $\eta$ increases. Increasing $m$ at fixed $(\eta,d)$ yields only marginal improvement, the empirical signature of the over-parameterized NTK regime. The sliced $\overline{W_1}$ mirrors this pattern, growing with $d$ (the curse-of-dimensionality factor $C(d)$ in Theorem~\ref{thmsampbound}) and shrinking with $\eta$. The shaded $\pm1$ std bands are narrow relative to the inter-curve gaps, confirming that the observed trends are robust to random initialization across all widths including $m=16$ and $m=32$. The complete per-cell trajectories confirming geometric decay of both metrics at every $d$ are in Appendices~\ref{app:exp:traces}--\ref{app:exp:w1traces}.

\begin{figure}[t]
  \centering
  \begin{subfigure}[t]{0.49\linewidth}
    \centering
    \includegraphics[width=\linewidth]{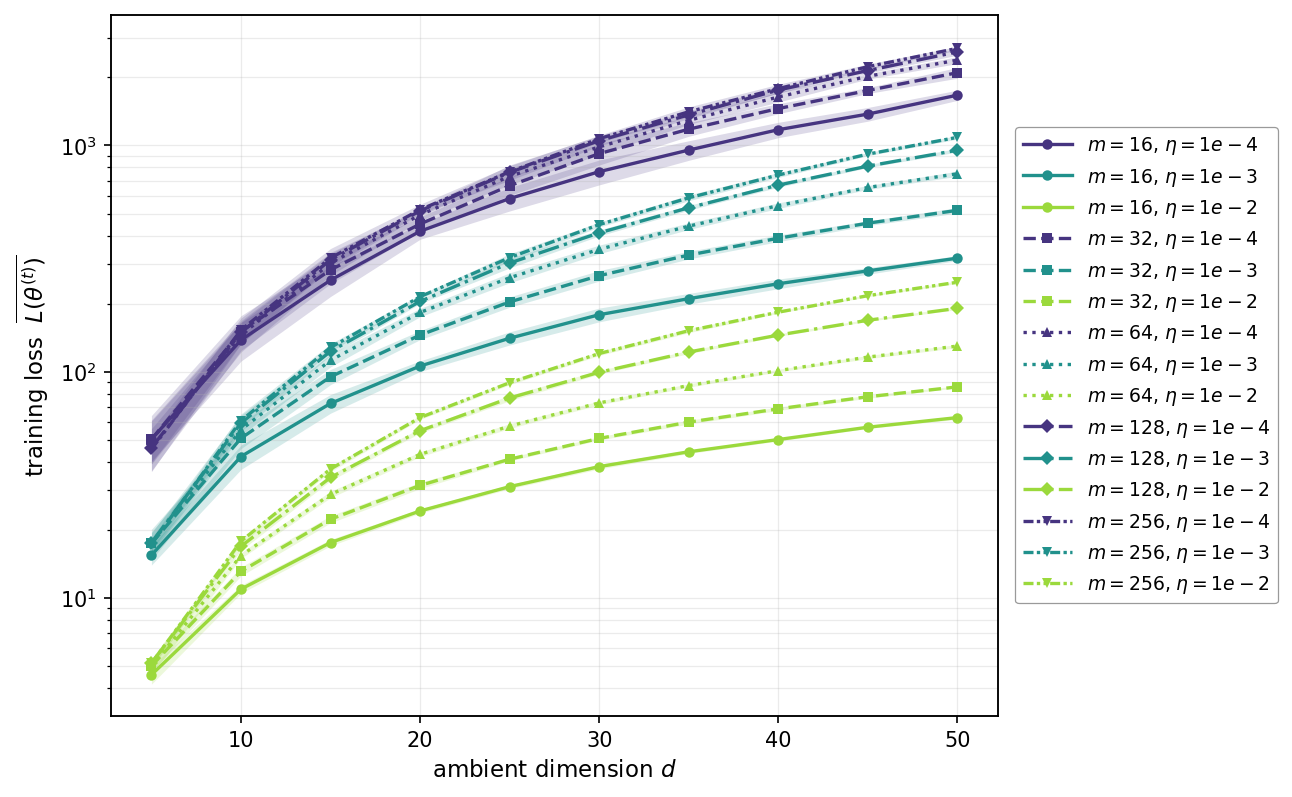}
    \caption{Training loss $L(\bfa\theta^{(T)})$ vs.\ $d$.}
    \label{fig:loss_vs_d}
  \end{subfigure}
  \hfill
  \begin{subfigure}[t]{0.49\linewidth}
    \centering
    \includegraphics[width=\linewidth]{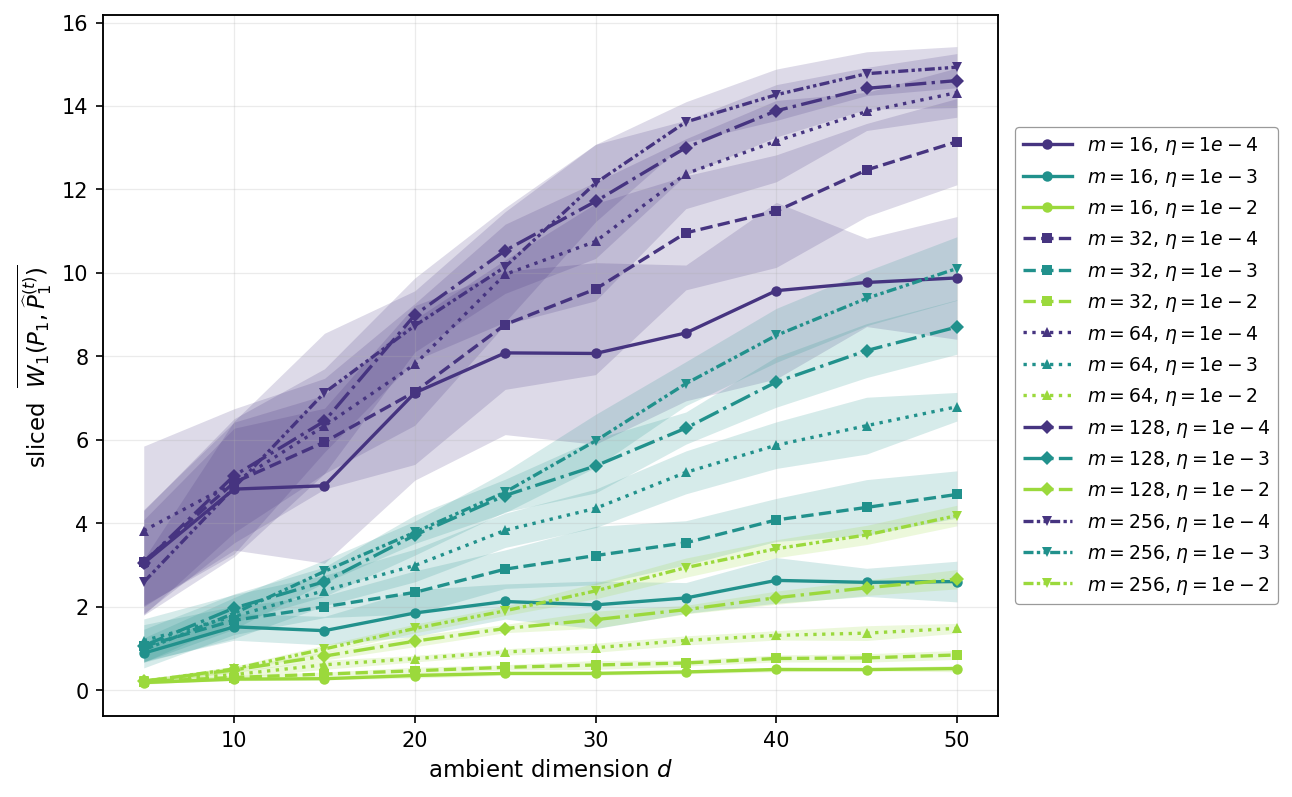}
    \caption{Sliced $\overline{W_1}(\bb P_1,\bb P_{\wh\xb_1^{(T)}})$ vs.\ $d$.}
    \label{fig:w1_vs_d}
  \end{subfigure}
  \caption{Global view of the $5\cdot3\cdot10=150$ cell sweep at $n_{\rm train}=500$, $n_{\rm test}=5000$. Each curve is one $(m,\eta)$ pair; color encodes $\eta$ ($10^{-4}$ darkest, $10^{-2}$ lightest) and marker/linestyle encodes width $m$. Values at the final iterate $t=T=500$; solid lines show the mean and shaded bands show $\pm1$ standard deviation over $10$ independent random seeds. (a) Terminal training loss grows with $d$ and decreases with $\eta$ (Theorem~\ref{theorem1}). (b) Sliced $\overline{W_1}$ grows with $d$ (Theorem~\ref{thmsampbound}).}
  \label{fig:global}
\end{figure}

\subsection{Real-World Datasets: MNIST and Fashion-MNIST}\label{sect:7:realworld}

\paragraph{Setup.}
We apply Algorithm~\ref{alg:learning} without modification to MNIST and Fashion-MNIST, two standard benchmarks of $28\times28$ grayscale images from ten classes each. Before implementing the algorithm, we first reduce the dimensionality of the images to $d=7$ PCA components by learning on the full training split. We use $n_{\rm train}=500$ and $n_{\rm test}=2000$ codes, retain the same linear schedule, conditional source $\bb P(\zb_0\mid\zb_1)=N(\zb_1,\bfa I_7)$, $T=500$ iterations, $T_{\rm Euler}=200$ Euler steps, and $B_1=10$, and sweep $m\in\{128,256,512,1024\}$ with $\eta\in\{10^{-4},10^{-3},10^{-2}\}$, giving $12$ grid cells per dataset.

\paragraph{Results.}
Figure~\ref{fig:realworld:recon} shows reconstructed images for the representative cell $(m,\eta)=(512,10^{-2})$ at the terminal iterate. To reconstruct, we draw new samples $\wh\zb_1$ from the learned distribution $\bb P_{\wh\zb_1^{(T)}}$ in $\bb R^7$ via Algorithm~\ref{alg:learning}, then apply the inverse PCA map to produce $28\times28$ pixel images. On MNIST the samples are clearly recognizable as handwritten digits spanning all ten classes; on Fashion-MNIST they display distinct garment categories with the characteristic smoothness of a $7$-component PCA basis. Quantitative training dynamics (loss and sliced $\overline{W_1}$ for all $12$ cells on each dataset) are reported in Appendix~\ref{app:exp:realworld:dynamics}; both metrics decay geometrically in $t$ and group by $\eta$ rather than $m$, consistent with Theorems~\ref{theorem1}--\ref{thmsampbound}. Full per-cell reconstruction grids are in Appendix~\ref{app:exp:realworld:recon}.

\begin{figure}[t]
  \centering
  \begin{subfigure}[t]{0.48\linewidth}
    \centering
    \includegraphics[width=\linewidth]{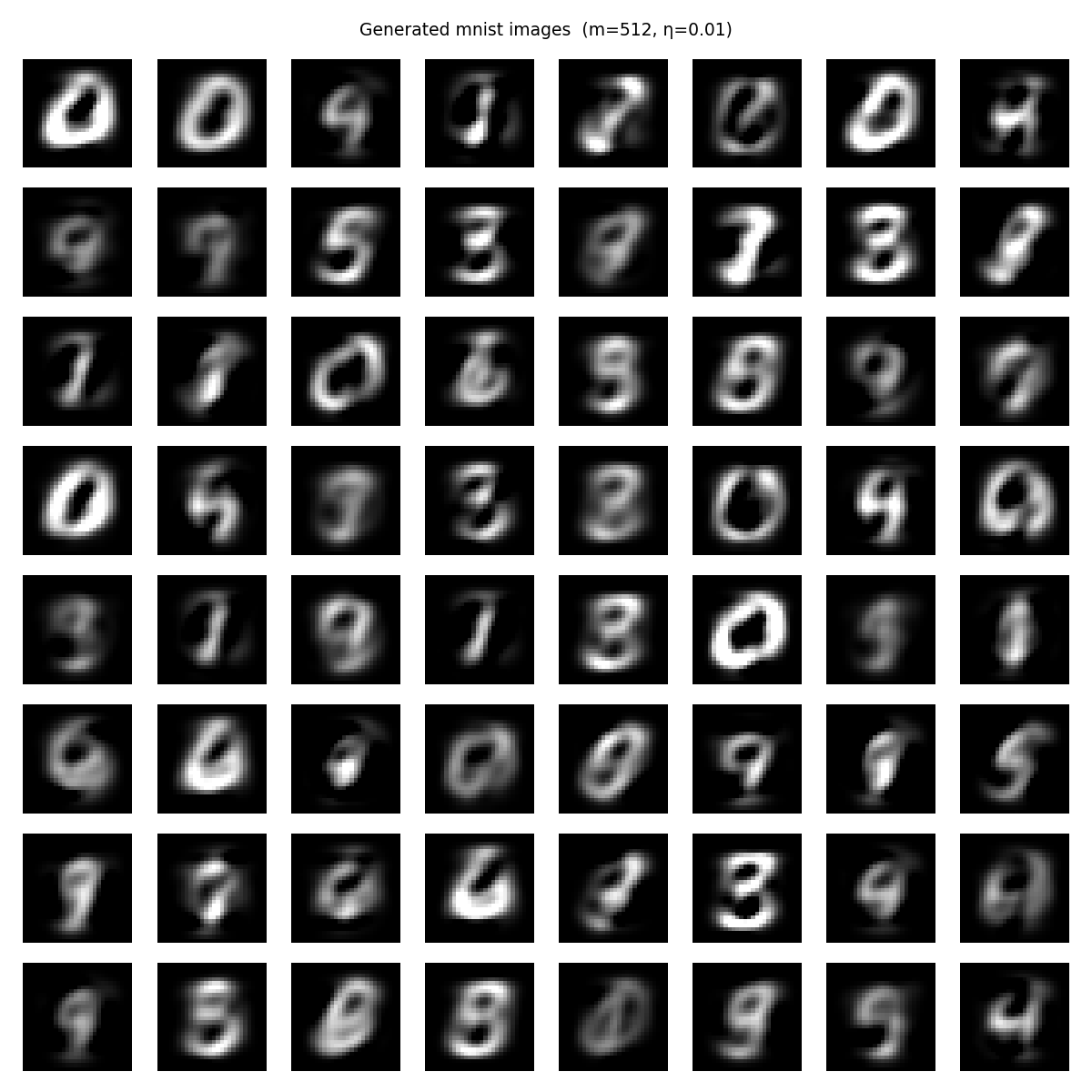}
    \caption{MNIST, $(m,\eta)=(512,10^{-2})$.}
    \label{fig:realworld:recon_mnist}
  \end{subfigure}
  \hfill
  \begin{subfigure}[t]{0.48\linewidth}
    \centering
    \includegraphics[width=\linewidth]{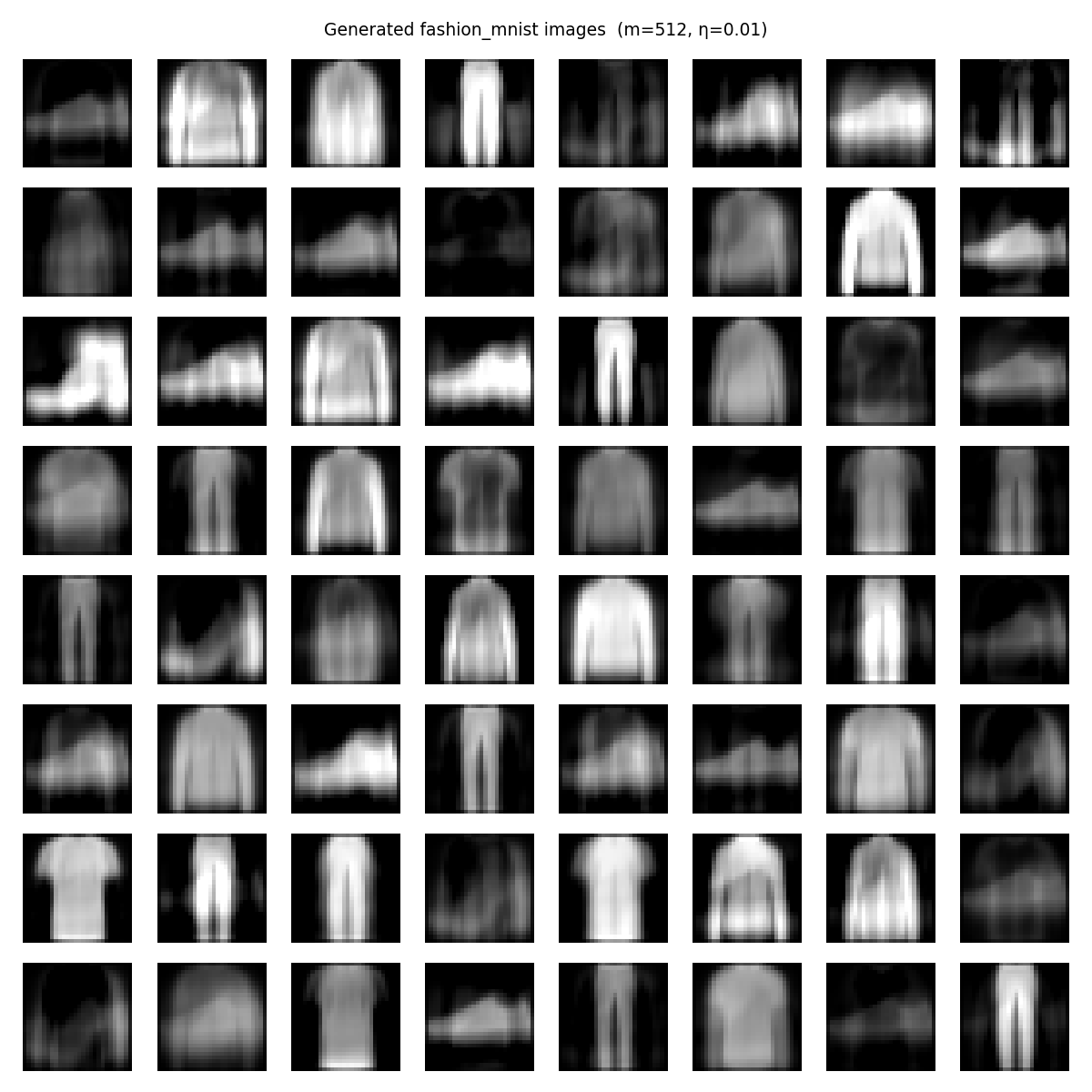}
    \caption{Fashion-MNIST, $(m,\eta)=(512,10^{-2})$.}
    \label{fig:realworld:recon_fmnist}
  \end{subfigure}
  \caption{$8\times8$ grids of $28\times28$ images generated in PCA code space ($d=7$) and inverse-transformed; terminal iterate $t=T=500$, cell $(m,\eta)=(512,10^{-2})$, $n_{\rm train}=500$. Left: MNIST samples are recognizable as handwritten digits. Right: Fashion-MNIST samples display distinct garment categories with the smoothness characteristic of a $7$-component PCA basis. Full per-cell reconstruction grids are in Appendix~\ref{app:exp:realworld:recon}.}
  \label{fig:realworld:recon}
\end{figure}

Additional NTK spectrum/contraction measurements and the PCA dimension/sample-size study are reported in Appendix~\ref{app:additional-results}.

\section{Discussion}\label{sect:8}
We study how gradient-descent training of a finite-width network affects the accuracy of its velocity field and the distribution of generated samples. Compared with statistical analyses of empirical-risk minimizers such as \citet{kumar2026flow}, our guarantees apply to the training iterates, with both layers trained. The analysis uses source-conditioned velocities and the stated coupling and kernel assumptions. One question left open is whether the sampling rate can be improved: much of the dimension dependence arises when we turn average velocity accuracy into uniform control along the flow. The rectified-flow results of \citet{bansal2024convergence} and \citet{roy2026low} suggest possible ways forward, though they require mean-squared velocity bounds and regularity conditions, as well as derivative-approximation bounds in the latter work. Proving these properties for our trained network would help connect the analyses. Our experiments illustrate the training behavior and how the sample size affects the usable PCA dimension; the images are generated in PCA space, and empirical sliced Wasserstein-1 serves as a diagnostic. Hard-instance lower bounds would help assess the sharpness of our rates. It would also be useful to extend the proofs to deeper networks, independent or minibatch-OT couplings, and optimizers such as stochastic gradient descent and Adam. For DDPM/DDIM, this would require relating velocity, noise, and score parameterizations and treating endpoint behavior and the stability of each sampler.

\begin{ack}
    We thank the anonymous reviewers and area chairs for their insightful and constructive feedback. Yuan Cao is supported in part by NSFC 12301657, Hong Kong RGC ECS 27308624, and Hong Kong RGC GRF grants 17301825 and 17313926.
\end{ack}

\bibliographystyle{plainnat}
\bibliography{references}

\clearpage

\appendix
\section*{Appendix Contents}
\addcontentsline{toc}{section}{Appendix Contents}
\startcontents[appendix]
\printcontents[appendix]{}{1}{\setcounter{tocdepth}{2}}
\section{Gradient Descent Analysis}\label{sect:4}
This section provides a comprehensive overview of the roadmap toward Theorem \ref{theorem1}. Our proof strategy follows \citet{du2017gradient} in that we employ an induction argument based on the following condition:
\begin{condition}\label{inductivecond}
    Given $m\gtrsim \frac{n^5\Vert\bfa U^{(0)}-\bfa V\Vert_2^2}{\kappa^2\lambda_{\min}(\bfa H^{(\infty)})\delta^3}d\log nd$ , $\eta <\frac{2\lambda_{\min}(\bfa H^{(\infty)})}{n^2d^2} $, and $\kappa=o\lef( \frac{\lambda_{\min}(\bfa H^{(\infty)})\delta}{n\sqrt d}\rig)$, with probability at least $1-\delta$, the prediction error at the $t-1$-th iteration satisfies
    \begin{align*}
       &\Vert\Delta^{(t-1)}\Vert_2^2
       \leq\Big(1-\frac{\eta\lambda_{\min}(\bfa H^{(\infty)})}{2}\Big)^{t-2}\Vert\Delta^{(0)}\Vert_2^2.
    \end{align*}
\end{condition}
However, since the $\ab_r$ parameters are trainable in our setup, the induction argument becomes more intricate compared with \citet{du2017gradient}. A key step, in the spirit of \citet{du2017gradient}, is to introduce the following set:
\begin{align*}
    \ca A_{i,r}(R) &=\{\exists\wb_r:\Vert\wb_r-\wb_r^{(0)}\Vert_2\leq R,\mbbm 1\{\tde\xb_i^\top\wb_r^{(0)}\geq 0\}\neq\mbbm 1\{\tde\xb_i^\top\wb_r \geq0\} \},
\end{align*}
where we use the short hand notation $\tde\xb_{i}:=[\xb_{\tau_i,i},\tau_i,\xb_{0,i}^\top]^\top$.
We further define the index sets $S_i:=\{r\in[m]:\mbbm 1\{\ca A_{i,r}(R)\}=0\}$ and its complement $S_i^\perp:=\{r\in[m]: r\notin S_i\}$. These sets facilitate the analysis of the ReLU activation, since for all $r\in S_i$ we have
\begin{align*}
    \mbbm 1\{\tde\xb_i^\top\wb_r^{(t)} > 0\} =\ldots=\mbbm 1\{\tde\xb_{i}^\top\wb_r^{(0)}>0\}.
\end{align*}
Hence, the induction can be carried out more conveniently on the index set $S_i$. We further show that the complement set $S_i^\perp$ satisfies the following properties.
\begin{lemma}\label{lm:2.5}
       Under assumptions \ref{assump1}, \ref{assump3}, and \ref{assump4}, with probability at least $1-\delta$ over the initialization, we have
    $
            \sum_{i=1}^n|S_i^\perp|\leq\frac{CmnR}{\delta}.
    $
\end{lemma}
Building on the above preparation, we carry out the induction as follows:
\begin{align*}
    \sum_{i=1}^n\Vert\bfa u_i^{(t)}-\vb_i\Vert_2^2 &= \sum_{i=1}^n\Big(\Vert\bfa u_i^{(t)}-\bfa u_{i}^{(t-1)}\Vert_2^2+\Vert\bfa u_i^{(t)} - \vb_i\Vert_2^2 \\
    &+ 2\la\bfa u_i^{(t)}-\bfa u_i^{(t-1)},\bfa u_i^{(t)}-\bfa u_i^{(t-1)},\bfa u_i^{(t-1)}-\vb_i\ra\Big).
\end{align*}
The difference between $\bfa u^{(t)}$ and $\bfa u^{(t-1)}$ can then be analyzed using the device of the set $S_i$ as follows:
\begin{align*}
    \sqrt m(\bfa u_i^{(t)}-\bfa u_i^{(t-1)} )&= \sum_{r\in S_i}\Big(\ab_r^{(t)}\sigma(\tde\xb_i^\top\wb_r^{(t)})-\ab_r^{(t-1)}\sigma(\tde\xb_i^\top\wb_r^{(t-1)})\Big)\\
    &+\sum_{r\in S_i^\perp}\Big(\ab_r^{(t)}\sigma(\tde\xb_i^\top\wb_r^{(t)})-\ab_r^{(t-1)}\sigma(\tde\xb_i^\top\wb_r^{(t-1)})\Big).
\end{align*}
And for the first term involving $\sum_{i\in S_i}$ we further analyze $\ab_r^{(t)}- \ab_r^{(t-1)}$ and $\wb_r^{(t)} - \wb_r^{(t-1)}$ by the gradient descent formula. Note that
\begin{align*}
    &\frac{\pta L(\bfa\theta^{(t)})}{\pta\wb_r} = \frac{1}{\sqrt m}\sum_{i=1}^m(\bfa u_i^{(t)}-\vb_i)^\top\ab_r^{(t)}\tde\xb_i\mbbm 1\{\tde\xb_i^\top\wb_r^{(t)} > 0\},\\
    &\frac{\pta L(\bfa\theta^{(t)})}{\pta\ab_r}=\frac{1}{\sqrt m}\sum_{i=1}^m(\bfa u_i^{(t)}-\vb_i)^\top\sigma(\tde\xb_i^\top\wb_r^{(t)}).
\end{align*}
To upper bound the two terms above, we require separate bounds on $\Vert\bfa u_i^{(t)}-\vb_i\Vert_2$, $\Vert\ab_r^{(t)}\Vert_2$, and $\Vert\wb_r^{(t)}\Vert_2$. The first quantity is controlled by Condition \ref{inductivecond}. The remaining two quantities can be bounded in terms of $\Vert\ab_r^{(t)}-\ab_r^{(0)}\Vert_2$ and $\Vert\wb_r^{(t)}-\wb_r^{(0)}\Vert_2$ via the following lemma.
\begin{lemma}\label{lm:awbound}
    Under condition \ref{inductivecond}, we can show that for all $r\in[m]$,
    \begin{align*}
    &\Vert\ab^{(t)}_r - \ab^{(0)}_r \Vert_2\leq C\frac{\eta\sqrt n}{\sqrt m}\Vert\Delta^{(0)}\Vert_2\max_{i\in[n],j\in[t]}|\tde\xb_i^\top\wb_r^{(j)}|,\\
    &\Vert\wb_r^{(t)}- \wb_r^{(0)}\Vert_2\leq C\frac{\eta\sqrt n}{\sqrt m}\Vert\Delta^{(0)}\Vert_2\max_{i\in[n], j\in[t]}\Vert\ab_r^{(t)}\tde\xb_i\Vert_2.
    \end{align*}
\end{lemma}
Using Lemma \ref{lm:awbound}, we can bound $\ab_r^{(t)}$ and $\wb_r^{(t)}$, which in turn yields bounds on $\frac{\pta L}{\pta \wb_r}$ and $\frac{\pta L}{\pta \ab_r}$. This allows us to complete the induction for Condition \ref{inductivecond}.

\section{Generalization Analysis}\label{sect:5}
This section discusses some intermediate results in the proof of the generalization bound given by Theorem \ref{thmgenbound}. Consider a class of functions $\ca F$ containing functions $\bb R^d\to\bb R^d$. The following empirical multi-task Rademacher complexity is introduced by \citet{maurer2016benefit}
\begin{align*}
    \ca R_{\xb}(\ca F) = \frac{1}{n}\bb E\bigg[\sup_{f\in\ca F}\sum_{i=1}^n\sum_{j=1}^d\epsilon_{ij}f_j(\tde\xb_i)\bigg|\{\tde\xb_i\}_{i\in[n]}\bigg],
\end{align*}
where $\epsilon_{ij}$s are i.i.d. random sign. This complexity has been widely used in multi-task learning settings, where the output of the learning algorithm is multivariate. We also note that in related work \citet{han2024neural}, the authors construct a coupling between kernel-based regression and neural networks, and derive generalization bounds by leveraging existing results for kernel-based regression. In contrast, we adopt a different approach by directly upper bounding the empirical Rademacher complexity of the neural network class.

We establish the following lemma, which provides generalization bounds for unbounded losses and complements the bounded-loss results in \citep{maurer2016vector}.
\begin{lemma}
    Suppose the loss function $\ell(\cdot,\cdot)$ satisfies $\bb E[\sup_{\bfa f\in\ca F}\cosh(\lambda \ell(f(\xb),y))]\leq\exp(\frac{1}{2}\sigma^2\lambda^2)$ for all $\lambda\in\bb R$. And $\ell(f(\xb),\yb)$ is $\rho$-Lipschitz in the first argument in Euclidean norm. Then with probability at least $1-\delta$ over the samples $\{\xb_{i}\}_{i\in[n]}$,
    \begin{align*}
        \sup_{f\in\ca F}\Big\{\bb E[\ell(f(\xb), \yb)] &- \frac{1}{n}\sum_{i=1}^n\ell(f(\xb_i),\yb_i)\Big\}\leq 2\rho{\ca R}_{\xb}(\ca F)
 + \sigma\sqrt{\frac{2\log(1/\delta)}{n}}.
 \end{align*}
 \begin{remark}
     We note that the condition of $\bb E[\sup_{\bfa f\in\ca F}\cosh(\lambda\ell(f(\xb),y))]\leq\exp(\frac{1}{2}\sigma^2\lambda^2)$  entails a strong tail control for the supremum of the stochastic process $\ell(f(\xb),y)$. A closely related condition appears in \citet{meir2003generalization}, where the authors study the standard univariate setting. Our bound extends their analysis to the multi-task learning regime.
 \end{remark}
\end{lemma}
We next elaborate on the choice of the function class $\ca F$. Since the class of neural networks is inherently unbounded, it is essential to exploit additional structure induced by the gradient descent dynamics. In particular, Lemma \ref{lm:awbound} suggests that, with high probability, the parameter deviations $\Vert\ab_r^{(t)}-\ab_r^{(0)}\Vert_2$ and $\Vert\wb_r^{(t)}-\wb_r^{(0)}\Vert_2$ remain small throughout optimization. Motivated by this observation, we consider the following restricted class:
\begin{align*}
    \ca F_{R_{\ab}, R_{\wb}}:&=\Big\{\{\wb_r\}_{r\in[m]}, \{\ab_r\}_{r\in[m]}:\Big(\sum_{r=1}^m\Vert\wb_r - \wb_r^{(0)}\Vert_2^2\Big)^{{1}/{2}} \leq R_{\wb},\\
    &\Big(\sum_{r=1}^m \Vert\ab_r -\ab_r^{(0)}\Vert_2^2\Big)^{{1}/{2}}\leq R_{\ab} ,\forall r\in[m]\Big\}.
\end{align*}
We emphasize that $\ca F_{R_{\ab},R_{\wb}}$ imposes simultaneous control over all network weights, rather than bounding an individual parameter vector $\ab_r$ as in Lemma \ref{lm:awbound}.

With this restriction, the analysis decomposes into three components: (1) upper bounding the empirical Rademacher complexity $\ca R_{\xb}(\ca F)$ for the class $\ca F_{R_{\wb},R_{\ab}}$; (2) establishing that the defining constraints of $\ca F_{R_{\ab},R_{\wb}}$ hold with high probability along the gradient descent trajectory; and (3) upper bounding the sub-Gaussian parameter $\sigma$ and verifying that the moment condition in Lemma \ref{multitasklearn} is satisfied for the neural networks under consideration.

To bound the Rademacher complexity, we first establish the following lemma.
\begin{lemma}\label{lmrademacher}
  Under assumptions \ref{assump1}, \ref{assump3}, and \ref{assump4}, given $\kappa = O(\frac{\lambda_0\delta}{n})$ and $m=\kappa^{-2}\poly(n, d,\delta^{-1},\lambda_{\min}^{-1}(\bfa H^{(\infty)}))$ the empirical Rademacher complexity of class $\ca F_{R_{\ab},R_{\wb}}$ is bounded as follows with probability at least $1-\delta$ over the training samples and the initialization of the network parameters,
    \begin{align*}
        \ca R(\ca F_{R_{\wb},R_{\ab}}) &\leq 2\sqrt{\bfa V\bfa H^{(\infty),-1}\bfa V^\top}\sqrt{\frac{d\log^2(n/\delta)}{n}}+\frac{C\cdot\poly(n, \lambda_{\min}^{-1}(\bfa H^{(\infty)}) ,d,1/\delta)}{m^{1/4}\kappa^{1/2}}.
    \end{align*}
\end{lemma}
\begin{remark}
    In Lemma \ref{lmrademacher}, we show that the empirical Rademacher complexity admits an upper bound in terms of $\sqrt{\bfa V\bfa H^{(\infty),-1}\bfa V^\top}$, which is itself bounded by $B_0$ with high probability. A related bound is obtained in \citet{arora2019fine} for the univariate setting where only the $\wb$ parameters are trainable. Our analysis extends their result by accommodating high-dimensional outputs and allowing the $\ab$ parameters to be trainable as well.
\end{remark}
For the class $\ca F_{R_{\ab},R_{\wb}}$, we next establish the following lemma, which shows that its defining constraints hold with high probability.
\begin{lemma}\label{wdiffadiff}
  Under assumptions \ref{assump1}, \ref{assump3}, and \ref{assump4},  for the two sets of weights $\{\wb_r\}_{r\in[m]}$ and $\{\ab_r\}_{r\in[m]}$ we show that with probability at least $1-\delta$,
    \begin{align*}
    &\Big(\sum_{r=1}^m\Vert\wb_r^{(t)} - \wb_r^{(0)}\Vert_2^2\Big)^{{1}/{2}}\leq \sqrt{\bfa V\bfa T\bfa B^{(\infty)}\bfa T\bfa V^\top} + \frac{C\cdot\poly(n,\lambda_{\min}^{-1}(\bfa H^{(\infty)}), d, 1/\delta)}{m^{1/4}\kappa^{1/2}}+\frac{Cnd\kappa }{\lambda_{\min}(\bfa H^{(\infty)})\delta},\\
    &\Big(\sum_{r=1}^m \Vert\ab_r^{(t)} -\ab_r^{(0)}\Vert_2^2\Big)^{{1}/{2}}
    \leq \sqrt{\bfa V\bfa T\bfa C^{(\infty)}\bfa T\bfa V^\top}+ \frac{C\cdot\poly(n, \lambda_{\min}^{-1}(\bfa H^{(\infty)}) ,d,1/\delta)}{m^{1/4}\kappa^{1/2}}+\frac{Cnd\kappa }{\lambda_{\min}(\bfa H^{(\infty)})\delta}.
    \end{align*}
    where $\bfa T=\eta\sum_{i=1}^t(I-\eta\bfa H^{(\infty)})^i$ is a polynomial on $\bfa H^{(\infty)}$.
\end{lemma}
\begin{remark}
    Lemma \ref{wdiffadiff} implies that, with high probability, we can simultaneously control the deviations $\Vert\ab_r-\ab_r^{(0)}\Vert_2$ and $\Vert\wb_r-\wb_r^{(0)}\Vert_2$ for all $r\in[m]$. A closely related statement for the univariate setting appears as Lemma 5.3 in \citet{arora2019fine}. Our result substantially extends theirs by accommodating the multivariate case and allowing the $\ab_r$ parameters to be trainable; in the special case $d=1$ and with $\ab_r$ fixed, our bound recovers their result.
\end{remark}
Building on the preceding two lemmas, we next verify that neural networks in $\ca F_{R_{\ab},R_{\wb}}$ satisfy an appropriate Lipschitz property. Moreover, we show that the moment condition required in Lemma \ref{multitasklearn} holds for this class.
\begin{lemma}\label{logcoshbound}
   Under assumptions \ref{assump1}, \ref{assump3}, and \ref{assump4}, for the class of functions $\ca F_{R_{\ab},R_{\wb}}$, with probability $1-\delta$ we have
\begin{align*}
     \log\bb E&\Big[\sup_{\bfa f\in\ca F_{R_{\wb},R_{\ab}}}\cosh\lef(\lambda\lef\Vert\bfa f(\xb_{\tau_i}^\top, \tau_i,\xb_{0,i}) - \vb(\xb_{\tau_i}|\xb_{0,i}))\rig\Vert_2\rig)\Big]\\
                &\leq \lef(CB_0\Big(\kappa R_{\wb} +\kappa R_{\ab}+ 2B_0\Big)\sqrt{\log\frac{m}{\delta}}\lambda\rig)^2.
\end{align*}
\end{lemma}
Combining Lemmas \ref{multitasklearn}, \ref{lmrademacher}, \ref{wdiffadiff}, and \ref{logcoshbound}, we complete the proof of Theorem \ref{thmgenbound}.

\section{Sampling Analysis}\label{sect:6}
This section provides a proof sketch of Theorem \ref{thmsampbound} and discusses some important intermediate results.
The proof for the Wasserstein-1 distance relies on an upper bound on $\bb E[\Vert\xb_1-\wh\xb_1\Vert_2]$. However, such a bound is hard to obtain because we only have $L_1$ convergence for the error in the velocity field estimation, $\Vert \bfa f_{\bfa\theta^{(t)}} - \vb_{\tau}(\xb_{\tau}|\xb_{0})\Vert_2$. To upper bound the difference between two solutions to the ODE \eqref{odee}, we need a stronger result of the form
\begin{align*}
\sup_{\xb_0,\xb_{\tau},\tau}\Vert \bfa f_{\bfa\theta^{(t)}}(\xb_{\tau},\tau,\xb_0) - \vb_{\tau}(\xb_{\tau}|\xb_{0})\Vert_2=o(1).
\end{align*}
This average-to-uniform conversion and the balance between compact-region and tail errors explain the dimension-dependent sampling rate, without establishing a matching lower bound. The results of \citet{bansal2024convergence,roy2026low} concern marginal state-time drifts and require mean-squared velocity control, with additional derivative-approximation assumptions in the latter work. These conditions do not follow directly from our source-conditioned expected-norm guarantee. Squared loss is not globally Lipschitz, so extending our norm-loss contraction argument requires additional range or tail control.
Such a result does not hold in general, as one can easily find counterexamples. To resolve this challenge, we prove the following lemma.
\begin{lemma}\label{l1tolinfty}
Assume that a compact domain $K\subset\bb R^d$ is compact with nonempty interior. Assume that $X\in K$ has density bounded below, with $p(X) \geq p_{\min} >0$. Assume that $f:\bb R^d\to\bb R^d$ is Lipschitz in Euclidean norm with $|f(x)-f(y)|\leq L\Vert x -y\Vert_2$. Assume that there exists $\beta>0$ such that for every $x\in K$ and every $r\in(0,diam(K)]$ we have $Vol(B(x,r)\cap K)\geq\beta r^d$. Then we have
\begin{align*}
\sup_{X\in K} |f(X)|\leq\lef(\frac{2^{d+1}L^d}{\beta p_{\min}}\bb E[|f(X)|]\rig)^{\frac{1}{d+1}}.
\end{align*}
\end{lemma}
\begin{remark}
The idea behind Lemma \ref{l1tolinfty} is that for a bounded, sufficiently well-shaped region with probability density bounded away from $0$, we can show that $L_{\infty}$ convergence holds within this restricted region given $L_1$ convergence.
\end{remark}
The above lemma further gives rise to the following result.
\begin{lemma}\label{uniformbound}
Under the same conditions as in Theorem \ref{thmgenbound}, we define the compact set $K:={(\xb_{0},\tau,\xb_{\tau}):\Vert\xb_0\Vert_2\leq \zeta, \tau \in[0,1]}$. Then we have
\begin{align*}
\sup_{\tde\xb\in K}\Vert\bfa f_{\bfa\theta}(\tde\xb)& - \vb(\xb_{\tau}|\xb_0)\Vert_2\leq C\Big(\frac{(4B_0)^d}{\exp(-\zeta^2/2d)}\sqrt{\frac{d}{n}\log(n/\delta)}\Big)^{\frac{1}{d+1}}.
\end{align*}
Specifically, if we let $\zeta = C\sqrt{d\log n}$ for some constant $C$, and assume that $d\vee B_0= o(\log n)$, we can show that with probability at least $1-\delta$,
\begin{align*}
\sup_{(\xb_0,\xb_{\tau},\tau)\in K}\Vert \bfa f_{\bfa\theta^{(t)}}(\xb_{\tau},\tau,\xb_0) - \vb_{\tau}(\xb_{\tau}|\xb_{0})\Vert_2=o(\log({1}/{\delta})).
\end{align*}
\end{lemma}
\begin{remark}
The above lemma implies that we can divide the whole space $\bb R^{2d}\times[0,1]$ into two sets, $K^c$ and $K$, and we can show that uniform convergence holds on $K$. However, the constraint on the dimension is significantly stricter here, as we can only allow $d=o(\log n)$ rather than $d$ being of order $o(n/\log n)$ in Theorem \ref{thmgenbound}.
\end{remark}
Then, we utilize a Gronwall-type inequality to bound the difference between the solutions $\wh{\xb}1$ and $\xb_1$ on the set $K$. For the complement set $K^c$, we use Assumption \ref{assump4} and the truncation step in Algorithm \ref{alg:learning} to obtain the following upper bound
\begin{align*}
\bb E[\Vert\xb_1 - \wh\xb_1\Vert_2]
&= \bb E[\Vert\xb_1-\wh\xb_{1}\Vert_2 \mid \xb_0 \in K] ,\bb P(\xb_0\in K) + 3 B_0,\bb P(\xb_0\in K^c).
\end{align*}
We further utilize the fact that $W_1(\bb P_{\wh\xb_1},\bb P_{\xb_1})\leq \bb E\big[\Vert\xb_1-\wh\xb_1\Vert_2\big]$ to obtain Theorem \ref{thmsampbound}.

\section{Proof of Theorem \ref{theorem1}}\label{app:proof-theorem1}
Our proof of theorem \ref{theorem1} and related lemmas are provided in this section. For convenience we restate all the lemmas in section \ref{sect:4} and provide proof for them.

The proof goes by induction, where we consider the  condition \ref{inductivecond}.
Recall that we denote $\bfa u_i^{(t)}= f_{\bfa\theta^{(t)}}(\xb_{\tau_i},\tau_i)$.
    The derivative is given by
    \begin{align}\label{deripart}
        \frac{\partial L(\bfa\theta^{(t)})}{\partial\wb_r}&=\frac{1}{\sqrt m}\sum_{i=1}^n \lef(f_{\bfa\theta^{(t)}}(\xb_{\tau_i},\tau_i) -\vb_{\tau_i}(\xb_{\tau_i}|\xb_{1,i})\rig)^\top\frac{\pta f_{\bfa\theta^{(t)}}\lef(\xb_{\tau_i},\tau_i\rig)}{\pta\wb_r}\nnb\\
        &=\frac{1}{\sqrt m}\sum_{i=1}^n(f_{\bfa\theta^{(t)}}(\xb_{\tau_i},\tau_i) - \vb_{\tau_i}(\xb_{\tau_i}|\xb_{1,i}))^\top\ab_{r}^{(t)}[\xb^\top_{\tau_i},\tau_i,\xb_{0,i}^\top]^\top\mbbm 1\lef\{[\xb_{\tau_i}^\top,\tau]\wb_r\geq 0\rig\}.\\
        \Big\Vert \frac{\partial L(\bfa\theta^{(t)})}{\partial\wb_r}\Big\Vert_2&=\frac{1}{\sqrt m}\lef\Vert(\bfa U^{(t)} - \bfa V)[\bfa A_{r,1}^{(t),\top},\bfa A_{r,2}^{(t),\top},\ldots,\bfa A_{r,2}^{(t),\top}]^\top\rig\Vert\nnb\\
        &\leq \frac{\sqrt n}{\sqrt m}\lef\Vert\bfa U^{(t)} - \bfa V\rig\Vert_2\max_{i\in[n]}\Vert\bfa A_{r,i}^{(t)}\Vert_2,\nnb\\
        \frac{\pta L(\bfa\theta^{(t)})}{\pta \ab_r} &= \frac{1}{\sqrt m}\sum_{i=1}^n\Big(\bfa u_i^{(t)} - \vb_{\tau_i}(\xb_{\tau_i}|\xb_{1,i})\Big)\sigma([\xb_{\tau_i,i}^\top,\tau_i,\xb_{0,i}^\top]\wb_r^{(t)}).
    \end{align}
    where $\bfa A_{r,i}^{(t)}:=\ab_r^{(t)}[\xb_{\tau_i,i}^\top,\tau_i,\xb_{0,i}^\top]$.
    Then the following lemma provides an upper bound for $\Vert\wb_r^{(t)} - \wb_r^{(0)}\Vert_2$.
    \begin{lemma}\label{lm:2.3}
    Under condition \ref{inductivecond}, we can show that
        $$\Vert\wb_i^{(t)}- \wb_i^{(0)}\Vert_2\leq \frac{\eta\sqrt n}{(1-C^{\frac{1}{2}})\sqrt m}\Vert \bfa U^{(0)}-\bfa V\Vert_2\max_{i\in[n], j\in[t]}\Vert\bfa A_{r,i}^{(j)}\Vert_2.$$
\end{lemma}
\begin{proof}
    The proof goes by showing that for all $r\in[m]$ we have
    \begin{align*}
        \Vert \wb_r^{(t)}- \wb_r^{(0)}\Vert_2&\leq \sum_{j=0}^{t-1}\Vert\wb_r^{(j)} -\wb_r^{(j+1)}\Vert_2\leq \eta\sum_{j=0}^{t-1}\Big\Vert\frac{\pta L({\bfa\theta}^{(j)})}{\pta \wb_r}\Big\Vert_2\\
        &\leq \frac{\eta}{\sqrt m}\sum_{j=0}^{t-1}\Big\Vert\sum_{i=1}^n(\bfa u_i(j) - \vb_{\tau_i}(\xb_{\tau_i}|\xb_{1,i}))^\top\ab_{r}^{(j)}[\xb_{\tau_i,i}^\top,\tau_i,\xb_{0,i}^\top]^\top\Big\Vert_2\\
        &\leq\frac{\eta}{\sqrt m}\sum_{j=0}^{t-1}\Vert \bfa U(j) -\bfa V\Vert_2\Vert [\bfa A_{r,1}^{(j),\top},\ldots, \bfa A_{r, n}^{(j),\top}]\Vert_2\\
        &\leq\frac{\eta\sqrt n}{\sqrt m}\sum_{j=0}^{t-1}C^{\frac{1}{2}j}\Vert\bfa U^{(0)} - \bfa V\Vert_2\max_{i\in[n]}\Vert\bfa A_{r,i}^{(j)}\Vert_2\\
        &\leq \frac{\eta\sqrt n}{(1-C^{\frac{1}{2}})\sqrt m}\Vert \bfa U^{(0)}-\bfa V\Vert_2\max_{i\in[n], j\in[t]}\Vert\bfa A_{r,i}^{(j)}\Vert_2.
    \end{align*}
\end{proof}
And similarly, we can show the following lemma upperbounding $\ab^{(t)}-\ab^{(0)}$.
\begin{lemma}\label{lm:2.4}
    Under condition \ref{inductivecond}, we can show that for all $r\in[m]$,
    \begin{align*}
    \Vert\ab^{(t)}_r - \ab^{(0)}_r \Vert_2\leq \frac{\eta\sqrt n}{(1-C^{\frac{1}{2}})\sqrt m}\Vert\bfa U^{(0)}-\bfa V\Vert_2\max_{i\in[n],j\in[t]}|[\xb_{\tau_i,i}^\top,\tau_i,\xb_{0,i}^\top]\wb_r^{(j)}|.
\end{align*}
\end{lemma}
\begin{proof}
    We utilize \eqref{deripart} to show that
    \begin{align*}
        \Vert\ab_r^{(t)} - \ab_r^{(0)}\Vert_2&= \Big\Vert\sum_{i=1}^{t-1}\ab_r^{(i)} - \ab_r^{(i+1)}\Big\Vert_2\leq\eta\sum_{j=0}^{t-1}\Big\Vert\frac{\pta L(\bfa \theta^{(j)})}{\pta \ab_r}\Big\Vert_2\\
        &\leq \frac{\eta}{\sqrt m}\sum_{j=0}^{t-1}\Big\Vert\sum_{i=1}^n(\bfa u_i^{(j)}-\vb_{\tau_i}(\xb_{\tau_i}|\xb_{1,i}))\sigma\lef([\xb_{\tau_i,i}^\top,\tau_i,\xb_{0,i}^\top]\wb_r^{(j)}\rig)\Big\Vert_2\\
        &\leq\frac{\eta\sqrt n}{\sqrt m}\sum_{j=0}^{t-1}\Vert\bfa U(j)-\bfa V\Vert_2\max_{i\in[n],j\in[t]}|[\xb_{\tau_i,i}^\top,\tau_i,\xb_{0,i}^\top]\wb_r^{(j)}|\\
        &\leq\frac{\eta\sqrt n}{\sqrt m}\sum_{j=0}^{t-1}C^{\frac{1}{2}j}\Vert\bfa U^{(0)}-\bfa V\Vert_2\max_{i\in[n],j\in[t]}|[\xb_{\tau_i,i}^\top,\tau_i,\xb_{0,i}^\top]\wb_{r}^{(j)}|\\
        &\leq\frac{\eta\sqrt n}{(1-C^{\frac{1}{2}})\sqrt m}\Vert\bfa U^{(0)}-\bfa V\Vert_2\max_{i\in[n],j\in[t]}|[\xb_{\tau_i,i}^\top,\tau_i,\xb_{0,i}^\top]\wb_r^{(j)}|.
    \end{align*}
\end{proof}
We consider the L2 error evolution. Note that
    \begin{align}\label{Formulae}
        &\sum_{i=1}^n\Vert\bfa u_i^{(t)} -  \vb(\xb_{\tau_i}|\xb_{0,i})\Vert_2^2 = \sum_{i=1}^n\Vert\bfa u_i^{(t)}-\bfa u_i^{(t-1)} + \bfa u_i^{(t-1)} - \vb(\xb_{\tau_i}|\xb_{0,i})\Vert_2^2 \nnb\\
        &= \ub{\sum_{i=1}^n \Vert\bfa u_i^{(t)} - \bfa u_i^{(t-1)} \Vert_2^2}_{T_1} + \ub{\sum_{i=1}^n2\la\bfa u_i^{(t)} -\bfa u_i^{(t-1)}, \bfa u_i^{(t-1)} - \vb(\xb_{\tau_i}| \xb_{1,i}) \ra}_{T_2} + \sum_{i=1}^n\Vert\bfa u_i^{(t-1)} -  \vb(\xb_{\tau_i}|\xb_{0,i})\Vert_2^2.
    \end{align}
    Then, it is further noted that using \eqref{deripart} and defining sets $\ca A_{i,r}(R)= \{\exists \wb:\Vert\wb -\wb_r^{(0)}\Vert_2\leq R, \mbbm 1\{[\xb_{\tau_i,i}^\top,\tau_i,\xb_{0,i}^\top]\wb_r^{(0)}\geq 0\}\neq \mbbm 1\{[\xb_{\tau_i,i}^\top,\tau_i,\xb_{0,i}^\top]\wb_r^{(0)}\geq 0\} \}$. One can further define the two sets $S_i:=\{r\in[m], \mbbm 1\{\ca A_{i,r}\} = 0\}$ and $S_i^\perp:=\{ r\in[m],\mbbm 1\{\ca A_{i,r}\} =1\}$. Hence, we achieve that
    \begin{align}\label{diffu}
        \bfa u_i^{(t)} - \bfa u_i^{(t-1)} &= \frac{1}{\sqrt{m}}\sum_{r=1}^m\ab_r^{(t)}\lef(\sigma([\xb^\top_{\tau_i,i}, \tau_i,\xb_{0,i}^\top]\wb_r^{(t)}) -\sigma([\xb_{\tau_i,i}^\top,\tau_i,\xb_{0,i}^\top]\wb_r^{(t-1)}) \rig)\nnb\\
         &+ \frac{1}{\sqrt m}\sum_{r=1}^m(\ab_r^{(t)} - \ab_r^{(t-1)})\sigma\lef([\xb_{\tau_i,i}^\top,\tau_i,\xb_{0,i}^\top]\wb_r^{(t-1)}\rig)\nnb\\
         &=\frac{1}{\sqrt m}\sum_{r\in S_i}\ab_r^{(t)}\Big(\sigma([\xb_{\tau_i,i}^\top,\tau_i,\xb_{0,i}^\top]\wb_r^{(t)}) -\sigma([\xb_{\tau_i,i}^\top,\tau_i,\xb_{0,i}^\top]\wb_{r}^{(t-1)})\Big)\nnb\\
         &+\frac{1}{\sqrt m}\sum_{r\in S_i^\perp}\ab_r^{(t)}\Big(\sigma([\xb_{\tau_i,i}^\top,\tau_i,\xb_{0,i}^\top]\wb_r^{(t)})- \sigma([\xb_{\tau_i,i}^\top,\tau_i,\xb_{0,i}^\top]\wb_{r}^{(t-1)})\Big)\nnb\\
         &+\frac{1}{\sqrt m}\sum_{r=1}^m\lef(\ab_r^{(t)} - \ab_r^{(t-1)}\rig)\sigma\lef([\xb_{\tau_i,i}^\top,\tau_i,\xb_{0,i}^\top]\wb_r^{(t-1)} \rig).
    \end{align}
    Combining the two results above, we summarize the upper bounds on $\Vert\wb_r^{(t)} - \wb_r^{(0)}\Vert_2$ and $\Vert\ab_r^{(t)}-\ab_r^{(0)}\Vert_2$ through the following bound.
    \begin{lemma}
        For all $r\in[m]$, given that $m\geq\frac{n^5\Vert\bfa U^{(0)}- \bfa V\Vert_2^2 d\log(nd)}{\kappa^2\lambda_{\min}(\bfa H^{(\infty)})\delta^3}$ and $\eta\leq\frac{2\lambda_{\min}(\bfa H^{(\infty)})}{n^2d^2}$ we have with probability at least $1-\delta$,
        \begin{align*}
            \Vert\ab_r^{(t)}- \ab_r^{(0)} \Vert_2&\leq \frac{C\eta n}{\sqrt m}\Vert[\xb_{\tau_i,i}^\top,\tau_i,\xb_{0,i}^\top]\Vert_2\leq \frac{C\eta n\sqrt{p}}{\sqrt m}\log\frac{n}{\delta},\\
            \Vert\wb_r^{(t)} - \wb_r^{(0)} \Vert_2&\leq \frac{C\eta n \sqrt{pd}}{\sqrt m}\log\frac{n}{\delta}\log\frac{m}{\delta}.
        \end{align*}
    \end{lemma}
    \begin{proof}
        We note that by random initialization, we have
        \begin{align*}
            \bfa u_i= \frac{1}{\sqrt m}\sum_{r=1}^m\ab_r^{(0)}\sigma([\xb_{\tau_i}]^\top\wb_r^{(0)})\qquad\sim\qquad N\Big(\bfa 0, I_d\Big).
        \end{align*}
        Hence, using lemma \ref{lm:2.3} and lemma \ref{lm:2.4} we can show that
        \begin{align*}
            \Vert\bfa U^{(0)} - \bfa V\Vert_F\leq \Vert\bfa U^{(0)}\Vert_F + \Vert\bfa V\Vert_F\leq C\sqrt n.
        \end{align*}
        And subsequently we can show that with probability at least $1-\delta$, given $m\gtrsim \frac{n^5\Vert\bfa U^{(0)}-\bfa V\Vert_2^2}{\lambda_{\min}(\bfa H^{(\infty)})\delta^3}d\log nd$ and $\eta <\frac{2\lambda_{\min}(\bfa H^{(\infty)})}{n^2d^2}$, for all $t\in\bb Z^+$,
        \begin{align*}
            \Vert\wb_r^{(t)}-\wb_r^{(0)}\Vert_2\leq\frac{C\eta n}{\sqrt m}\max_{i\in[n],j\in[t]}\Vert\bfa A_{r,i}\Vert_2.
        \end{align*}
        And we further notice that with probability at least $1-\delta$,
        \begin{align*}
            \max_{i\in[n],j\in[t]}\lef|[\xb_{\tau_i,i}^\top,\tau_i,\xb_{0,i}^\top]\wb_r^{(j)}\rig|&\leq\max_{i\in[n],j\in[t]}\lef|[\xb_{\tau_i,i}^\top,\tau_i,\xb_{0,i}^\top]\wb_r^{(0)} + \Vert\wb_r^{(0)}-\wb_r^{(j)}\Vert_2\Vert[\xb_{\tau_i,i}^\top,\tau_i,\xb_{0,i}^\top]\Vert_2\rig| \\
            &\leq \max_{i\in[n], j\in[t]}\Big|[\xb_{\tau_i,i}^\top,\tau_i,\xb_{0,i}^\top]\wb_r^{(0)}\Big| + \frac{C\eta n}{\sqrt m}\max_{i\in[n], j\in[t]}\lef\Vert\bfa A_{r,i}\rig\Vert_2\lef\Vert[\xb_{\tau_i,i}^\top,\tau_i,\xb_{0,i}^\top]\rig\Vert_2.
        \end{align*}
        Using the union bound, we can conclude by the sub-Guassian property of $[\xb_{\tau_i},\tau_i]$,
        \begin{align*}
            \bb P\lef(\max_{i\in[n]}\lef\Vert[\xb_{\tau_i,i}^\top,\tau_i,\xb_{0,i}^\top]\rig\Vert_2\geq t\rig)\leq \sum_{i\in[n]}\bb P\lef(\lef\Vert[\xb_{\tau_i,i}^\top,\tau_i,\xb_{0,i}^\top]\rig\Vert_2\geq t\rig)\leq n\exp\Big(-\frac{Ct^2}{p}\Big).
        \end{align*}
        Hence, we can show that with probability at least $1-\delta$, one has $\Vert[\xb_{\tau_i,i}^\top,\tau_i,\xb_{0,i}^\top]\Vert < C\sqrt p\log\frac{n}{\delta}$.
        Hence, we can show that with probability at least $1-\delta$,
        \begin{align*}
            \max_{i\in[n],j\in[t]}\lef\Vert\bfa A_{r,i}^{(j)}\rig\Vert_2\leq Cd\log\frac{n}{\delta}\log\frac{m}{\delta}.
        \end{align*}
        Then, we can show that with probability at least $1-\delta$,
        \begin{align*}
            \Vert\wb_r^{(t)} - \wb_r^{(0)} \Vert_2\leq \frac{C\eta n d}{\sqrt m}\log\frac{n}{\delta}\log\frac{m}{\delta}.
        \end{align*}
        Similarly, we can show that with probability at least $1-\delta$,
        \begin{align*}
            \lef\Vert\ab_r^{(t)}- \ab_r^{(0)} \rig\Vert_2\leq \frac{C\eta n}{\sqrt m}\Vert[\xb_{\tau_i,i}^\top,\tau_i,\xb_{0,i}^\top]\Vert_2\leq \frac{C\eta n\sqrt{d}}{\sqrt m}\log\frac{n}{\delta}.
        \end{align*}
    \end{proof}
    For the first term in \eqref{diffu}, we note that by the Lipschitzness of ReLU function and the following lemma :
    \begin{lemma}\label{apx:lm:2.5}
        With probability at least $1-\delta$ over the initialization, we have
    $
            \sum_{i=1}^n|S_i^\perp|\leq\frac{CmnR}{\delta}.
    $\end{lemma}
    We achieve that with probability $1-\delta$ over the initialization,
    \begin{align}\label{diffterm3}
        \bigg\Vert\frac{1}{\sqrt m}\sum_{r\in S_i^\perp}\ab_r^{(t)}\Big(\sigma([\xb_{\tau_i,i}^\top,\tau_i,\xb_{0,i}^\top]\wb_r^{(t)})&- \sigma([\xb_{\tau_i,i}^\top,\tau_i,\xb_{0,i}^\top]\wb_{r}^{(t-1)})\Big)\bigg\Vert_2\nnb\\
        &\leq\frac{\eta}{\sqrt m} \max_{r\in[m]}\bigg\Vert\frac{\pta L(\bfa\theta^{(t-1)})}{\pta\wb_r^{(t-1)}}\bigg\Vert_2\sum_{r\in S_i^\perp}\Big\Vert\ab_r^{(t)}[\xb_{\tau_i}^\top, \tau_i]\Big\Vert_2\nnb\\
        &\leq\frac{\eta |S_i^\perp|}{\sqrt m}\max_{r\in[m]}\bigg\Vert\frac{\pta L(\bfa\theta^{(t-1)})}{\pta \wb_r^{(t-1)}}\bigg\Vert_2\cdot \max_{r\in[m], i\in[n]}\Vert\bfa A_{i,r}(t)\Vert_2\nnb\\
        &\leq\frac{\eta |S_i^\perp|\sqrt n}{m}\lef\Vert\bfa U^{(t-1)} -\bfa V \rig\Vert_2\cdot \max_{r\in[m], i\in[n]}\Vert\bfa A_{i,r}^{(t)}\Vert_2^2\nnb\\
        &\leq\frac{\eta n^{3/2} R}{\delta}\Vert\bfa U( t-1) -\bfa V\Vert_2\max_{r\in[m],i\in[n]}\Vert\bfa A_{i,r}^{(t)}\Vert_2^2.
    \end{align}
    For the third term in \eqref{diffu} we can show that using \eqref{deripart},
    \begin{align*}
        \frac{1}{\sqrt m}&\sum_{r=1}^m(\ab_r^{(t)}-\ab_r^{(t-1)})\sigma\lef([\xb_{\tau_i,i}^\top,\tau_i,\xb_{0,i}^\top]\wb_r^{(t-1)}\rig) = \frac{\eta}{\sqrt m}\sum_{r=1}^m\frac{\pta L(\bfa\theta^{(t-1)})}{\pta \ab_r}\sigma\lef([\xb_{\tau_i}^\top, \tau_i]\wb_r^{(t-1)}\rig)\\
        &=\frac{\eta}{m}\sum_{r=1}^m\sum_{j=1}^n\lef(\bfa u_j^{(t-1)} - \vb_{\tau_j}(\xb_{\tau_j}|\xb_{1,j})\rig)\sigma([\xb_{\tau_j}^\top,\tau_j,\xb_{0,j}^\top]\wb_r^{(t-1)})\sigma([\xb_{\tau_i,i}^\top,\tau_i,\xb_{0,i}^\top]\wb_r^{(t-1)})\\
        &=\eta\sum_{j=1}^n\bfa B_{ij}^{(t-1)}\lef(\bfa u^{(t-1)}_j - \vb_{\tau_j}(\xb_{\tau_j}|\xb_{1,j})\rig),
    \end{align*}
    where we define the matrix $\bfa B^{(t)}$ by $\bfa B_{ij}^{(t)}=\frac{1}{m}\sum_{r=1}^m\sigma([\xb_{\tau_j}^\top,\tau_j,\xb_{0,j}^\top]\wb_{r}^{(t-1)})\sigma([\xb_{\tau_i,i}^\top,\tau_i,\xb_{0,i}^\top]\wb_r^{(t-1)})$.
   Then we consider the second term in \eqref{diffu}. It is noted that under $A_{r,i}(R)$ with $R\geq \frac{\eta\sqrt n}{(1-C^{\frac{1}{2}})\sqrt m}\Vert \bfa U^{(0)}-\bfa V\Vert_2\max_{i\in[n], t}\Vert\bfa A^{(t)}_{r,i}\Vert_2$ we can show that $\mbbm 1\{[\wb_{\tau_i}^\top,\tau_i]\wb_{r}^{(t)} > 0\} = \mbbm 1\{[\xb_{\tau_i,i}^\top,\tau_i,\xb_{0,i}^\top]\wb_r^{(t-1)} > 0\} =\ldots =\mbbm 1\{[\xb_{\tau_i}^\top, \tau_i]\wb^{(0)}_r > 0\}$, which implies that
\begin{align*}
   \frac{1}{\sqrt m}&\sum_{r\in[m]}\ab_r^{(t)}\lef(\sigma([\xb_{\tau_i,i}^\top,\tau_i,\xb_{0,i}^\top]\wb_r^{(t)}) - \sigma([\xb_{\tau_i,i}^\top,\tau_i,\xb_{0,i}^\top]\wb_{r}^{(t-1)})\rig)\\
&= \frac{1}{\sqrt m}\sum_{r\in [m]}\ab_r^{(t)}[\xb_{\tau_i,i}^\top,\tau_i,\xb_{0,i}^\top]\lef(\wb_{r}^{(t)}-\wb_r^{(t-1)}\rig)\mbbm 1\{[\xb_{\tau_i},\tau_i]\wb_r^{(0)} >  0\}\\
&= \frac{\eta}{\sqrt m}\sum_{r\in [m]}\ab_{r}^{(t)}[\xb_{\tau_i,i}^\top,\tau_i,\xb_{0,i}^\top]\frac{\pta L(\bfa\theta^{(t-1)})}{\pta \wb_{r}}\mbbm 1\{[\xb_{\tau_i},\tau_i]\wb_r^{(0)} > 0\}\\
&=\frac{\eta}{m}\sum_{r\in [m]}\ab_r^{(t)}\sum_{j=1}^n[\xb_{\tau_i,i}^\top,\tau_i,\xb_{0,i}^\top][\xb_{\tau_j}^\top,\tau_j,\xb_{0,j}^\top]^\top(\bfa u_j^{(t-1)} -\vb_{\tau_j}(\xb_{\tau_j}|\xb_{1,j}))^\top\ab_r^{(t)}\\
&\cdot\mbbm 1\{[\xb_{\tau_i}, \tau_i]\wb_r^{(0)} > 0, [\xb_{\tau_j}^\top,\tau_j,\xb_{0,j}^\top]\wb_r^{(0)} > 0\}\\
&=\sum_{r\in S_i}\ab_r^{(t)}\ab_r^{(t),\top}\sum_{j=1}^n\bfa G_{ij, r}(\bfa u_j^{(t-1)}- \vb_{\tau_j}(\xb_{\tau_j}|\xb_{1,j}))\\
&=\sum_{j=1}^n\lef(\bfa C_{ij}^{(t)} - \bfa C^{(t),\perp}_{ij}\rig)(\bfa u_j^{(t-1)} - \vb_{\tau_j}(\xb_{\tau_j}|\xb_{1,j})),
\end{align*}
where we define $\bfa G_{ij,r}= [\xb_{\tau_i,i}^\top,\tau_i,\xb_{0,i}^\top][\xb_{\tau_j}^\top,\tau_j,\xb_{0,j}^\top]^\top\mbbm 1\Big\{[\xb_{\tau_i},\tau_i]\wb_r^{(0)}>0, [\xb_{\tau_j}^\top,\tau_j,\xb_{0,j}^\top]\wb_{r}^{(0)} > 0\Big\}$ and
\begin{align*}
    \bfa C_{ij}^{(t)} &= \frac{1}{m}\sum_{r\in[m]}\ab_r^{(t)}\ab_r^{(t),\top}[\xb_{\tau_i,i}^\top,\tau_i,\xb_{0,i}^\top][\xb_{\tau_j}^\top,\tau_j,\xb_{0,j}^\top]^\top\mbbm 1\Big\{[\xb_{\tau_i,i}^\top,\tau_i,\xb_{0,i}^\top]\wb_r^{(0)}>0, [\xb_{\tau_j}^\top,\tau_j,\xb_{0,j}^\top]\wb_{r}^{(0)} > 0\Big\},\\
    \bfa C_{ij}^{(t),\perp} &= \frac{1}{m}\sum_{r\in S_i^\perp}\ab_r^{(t)}\ab_r^{(t),\top}[\xb_{\tau_i,i}^\top,\tau_i,\xb_{0,i}^\top][\xb_{\tau_j}^\top,\tau_j,\xb_{0,j}^\top]^\top\mbbm 1\Big\{[\xb_{\tau_i,i}^\top,\tau_i,\xb_{0,i}^\top]\wb_r^{(0)}>0, [\xb_{\tau_j}^\top,\tau_j,\xb_{0,j}^\top]\wb_{r}^{(0)} > 0\Big\}.
\end{align*}

To analyze the matrices $\{\bfa H_{ij}\}_{i,j\in[n]}$ defined by $\bfa H_{ij}^{(t)}=\bfa C_{ij}^{(t)}-\bfa C_{ij}^{(t),\perp} + \bfa B_{ij}^{(t-1)}$. We define the corresponding limit
\begin{align*}
    &\bfa H_{ij}^{(\infty)} = \bfa C_{ij}^{(\infty)}+\bfa B_{ij}^{(\infty)}=\bb E_{\wb_r\sim N(0, I_d)}\lef[ I_d\sigma\lef([\xb_{\tau_j}^\top,\tau_j,\xb_{0,j}^\top]\wb_r\rig)\sigma\lef([\xb_{\tau_i,i}^\top,\tau_i,\xb_{0,i}^\top]\wb_r\rig)\rig] \\
    &+ \bb E_{\wb_r\sim N(0, I_d), \ab_r\sim N(0, I_d)}\lef[\ab_r\ab_r^\top[\xb_{\tau_i,i}^\top,\tau_i,\xb_{0,i}^\top][\xb_{\tau_j}^\top,\tau_j,\xb_{0,j}^\top]^\top\mbbm 1\lef\{[\xb_{\tau_i,i}^\top,\tau_i,\xb_{0,i}^\top]\wb_r >0, [\xb_{\tau_j}^\top,\tau_j,\xb_{0,j}^\top]\wb_r > 0\rig\}\rig]\\
    &=I_d\bb E_{\wb_r\sim N(0, I_d)}\Big[[\xb_{\tau_j}^\top, \tau_j]\lef(\wb_r\wb_r^\top + I_d\rig)[\xb_{\tau_i,i}^\top,\tau_i,\xb_{0,i}^\top]^\top\mbbm 1\{[\xb_{\tau_i,i}^\top,\tau_i,\xb_{0,i}^\top]\wb_r > 0,[\xb_{\tau_j}^\top,\tau_j,\xb_{0,j}^\top]\wb_r > 0\}\Big].
\end{align*}
Hence, the matrix defined by $\bfa H^{(\infty)} :=\begin{bmatrix}
    \bfa H_{1,1}^{(\infty)}&\bfa H_{1,2}^{(\infty)}&\ldots&\bfa H_{1,n}^{(\infty)}\\
    \bfa H_{2,1}^{(\infty)}&\bfa H_{2,2}^{(\infty)}&\ldots&\bfa H_{2,n}^{(\infty)}\\
    \vdots&\vdots &\vdots &\vdots\\
    \bfa H_{n,1}^{(\infty)}&\bfa H_{n,2}^{(\infty)}&\ldots&\bfa H_{n,n}^{(\infty)}
    \end{bmatrix}$. Using the property of the infinity norm, we can show that
    \begin{align*}
        \lef\Vert\bfa H_{ij}^{(\infty)} -\bfa H^{(t)}_{ij}\rig\Vert_{2}&\leq \Big\Vert\bfa H_{ij}^{(\infty)} - \bfa C_{ij}^{(0)}-\bfa B_{ij}^{(0)}\Big\Vert_{2} + \Big\Vert\bfa C_{ij}^{(t),\perp}\Big\Vert_{2}+ \Big\Vert\bfa C_{ij}^{(t)} +\bfa B^{(t)}_{ij} - \bfa C_{ij}^{(0)} -\bfa B_{ij}^{(0)}\Big\Vert_2.
    \end{align*}
    By concentration of measure and the property that $\ab_{r,i}$ and $\xb_{\tau_j},\tau_j$ are sub-Gaussian random variables, we can show that with probability at least $1-\delta$ and the union bound, for $k,l\in[d]$,
    \begin{align*}
    \Big\Vert(\bfa H_{ij}^{(\infty)} -\bfa C_{ij}^{(0)} -\bfa B_{ij}^{(0)})_{kl}\Big\Vert&=\Big\Vert\frac{1}{m}\sum_{r=1}^m\sigma([\xb_{\tau_j}^\top,\tau_j,\xb_{0,j}^\top]\wb_r^{(0)})\sigma([\xb_{\tau_i},\tau_i]\wb_r^{(0)}) I_d\\
    &+\ab_r^{(0)}\ab_r^{(0),\top}[\xb_{\tau_j}^\top,\tau_j,\xb_{0,j}^\top][\xb_{\tau_i},\tau_i]^\top\mbbm 1\{[\xb_{\tau_i,i}^\top,\tau_i,\xb_{0,i}^\top]\wb_r^{(0)} > 0, [\xb_{\tau_j}^\top,\tau_j,\xb_{0,j}^\top]\wb_r^{(0)} >0\}\\
    &-\bb E\lef[I_d\sigma([\xb_{\tau_j}^\top,\tau_j,\xb_{0,j}^\top]\wb_r^{(0)})\sigma([\xb_{\tau_i},\tau_i]\wb_r^{(0)})\rig] \\
    &-\bb E\Big[\ab_r^{(0)}\ab_r^{(0),\top}[\xb_{\tau_j}^\top,\tau_j,\xb_{0,j}^\top][\xb_{\tau_i},\tau_i]^\top\mbbm 1\{[\xb_{\tau_i,i}^\top,\tau_i,\xb_{0,i}^\top]\wb_r^{(0)}>0, [\xb_{\tau_j}^\top,\tau_j,\xb_{0,j}^\top]\wb_r^{(0)}>0\}\Big]\Big\Vert_{2} \\
    &\leq \frac{\max_{i\in[d]}\bb E[\xb_{\tau_{j}, i}^2]}{\sqrt m}\log\frac{d}{\delta}.
    \end{align*}
And similarly, one can show that with probability at least $1-\delta$, for all $i,j\in[n]$, simultaneously,
\begin{align*}
    \lef\Vert\bfa H^{(\infty)}_{ij} -\bfa C^{(0)}_{ij} - \bfa B_{ij}^{(0)}\rig\Vert_{2}\leq \frac{\max_{i\in[d]}\bb E[\xb_{\tau_{j}, i}^2]}{\sqrt m}\log\frac{dn}{\delta}.
\end{align*}
Furthermore, using the union bound and the results given by Lemma~\ref{apx:lm:2.5}, with probability at least $1-\delta$, for all $i,j\in[n]$, we have simultaneously,
\begin{align*}
    \Big\Vert\bfa C_{ij}^{(t),\perp}\Big\Vert_2\leq \frac{|S_i^\perp|}{m}\max_{i\in[n], r\in[m]}\Vert \bfa A_{r,i}^{(t)}\Vert_2\leq\frac{Cn^2R}{\delta}\max_{i\in[n],r\in[m]}\Vert\bfa A_{r,i}^{(t)}\Vert_2.
\end{align*}

We note that for $\bfa u,\bfa v$ with $\Vert\bfa u-\bfa v\Vert_2\leq\epsilon$ and $\Vert\bfa u\Vert_2\leq B$, then we can show that
\begin{align*}
    \Vert\bfa u\bfa u^\top -\bfa v\bfa v^\top\Vert_2&= \max_{\bfa z:\Vert\bfa z\Vert_2 = 1}\lef|(\bfa u^\top\bfa z)^2 - (\bfa v^\top\bfa z)^2\rig| = \lef|(\bfa u-\bfa v)^\top\bfa z(\bfa u+\bfa v)^\top\bfa z\rig|\leq \Vert\bfa u-\bfa v\Vert_2\Vert\bfa u+\bfa v\Vert_2\\
    &\leq (2B+\epsilon)\epsilon.
\end{align*}
And for the last term we use lemmas \ref{lm:2.4} and \ref{apx:lm:2.5} to show that with probability at least $1-\delta$,
\begin{align*}
    \Big\Vert&\bfa C_{ij}^{(t)} + \bfa B_{ij}^{(t)} -\bfa C_{ij}^{(0)}- \bfa B_{ij}^{(0)}\Big\Vert_{2}\\
    &\leq \Big\Vert\frac{1}{m}\sum_{r\in[m]}(\ab_r^{(t)}\ab_r^{(t),\top}-\ab_r^{(0)}\ab_r^{(0),\top})[\xb_{\tau_i,i}^\top,\tau_i,\xb_{0,i}^\top][\xb_{\tau_j}^\top,\tau_j,\xb_{0,j}^\top]^\top \mbbm 1\lef\{[\xb_{\tau_i},\tau_i]\wb_r^{(0)} >0,[\xb_{\tau_j}^\top,\tau_j,\xb_{0,j}^\top]\wb_r^{(0)} >0\rig\}\Big\Vert_2\\
    &+\Big\Vert\frac{1}{m}\sum_{r\in S_i}[\xb_{\tau_i,i}^\top,\tau_i,\xb_{0,i}^\top]^\top(\wb_r^{(t)}\wb_r^{(t)} - \wb_r^{(0)}\wb_r^{(0)})[\xb_{\tau_i,i}^\top,\tau_i,\xb_{0,i}^\top]\mbbm 1\lef\{[\xb_{\tau_i}, \tau_i]\wb_r^{(0)}>0, [\xb_{\tau_j}^\top,\tau_j,\xb_{0,j}^\top]\wb_r^{(0)}>0\rig\}\Big\Vert_2\\
    &+\Big\Vert\bfa C_{ij}^{(t),\perp}\Big\Vert_2+ \Big\Vert \bfa C_{ij}^{(0),\perp}\Big\Vert_2\\
    &\leq C\max_{i\in[n],r\in[m]}\Vert\bfa A_{r,i}^{(t)}\Vert_2^3\frac{\eta\sqrt n}{\sqrt m}\Vert\bfa U^{(0)}-\bfa V\Vert_2 + \frac{Cn^2R}{\delta}\max_{i\in[n],r\in[m]}\Vert\bfa A_{r,i}^{(t)}\Vert_2.
\end{align*}
And collecting pieces, we can show that with probability at least $1-\delta$, for all $i,j\in[n]$,
\begin{align*}
    \Big\Vert\bfa H_{ij}^{(\infty)} -\bfa H_{ij}^{(t)}\Big\Vert_2&\leq \frac{\max_{i\in[d]}\bb E[\xb_{\tau_j, i}^2]}{\sqrt m}\log\frac{n^2d}{\delta}\max_{i\in[n],r\in[m]}\Vert\bfa A_{r,i}^{(t)}\Vert_2 + \frac{CnR}{\delta} \\
    &+ C\max_{i\in[n],r\in[m]}\Vert\bfa A_{r,i}^{(t)}\Vert_2^3\frac{\eta\sqrt n}{\sqrt m}\Vert\bfa U^{(0)}-\bfa V\Vert_2.
\end{align*}
Hence, for the larger matrix $\bfa H$, we can show that
\begin{align*}
    \Big\Vert\bfa H^{(\infty)} - \bfa H^{(t)}\Big\Vert_2\leq \sqrt n\sup_{i,j\in[n]}\Big\Vert \bfa H^{(\infty)}_{ij} -\bfa H^{(t)}_{ij}\Big\Vert_2.
\end{align*}
And, we further conclude that
\begin{align}\label{T1formula}
    T_1 &= \sum_{i=1}^n\Vert\bfa u_i^{(t)}- \bfa u_i^{(t-1)}\Vert_2^2\nnb\\
    &=\eta^2 \sum_{i=1}^n\Big \Vert \sum_{j=1}^n\lef(\bfa H_{i,j}^{(t)} -\bfa H_{i,j}^{(t),\perp}\rig)(\bfa u_j^{(t-1)} -\vb_{\tau_j}(\xb_{\tau_j}|\xb_{1,j}))\Big\Vert^2\nnb\\
    &\leq \eta^2\Vert\bfa H-\bfa H^\perp\Vert_2^2\lef\Vert\bfa U(t - 1)-\bfa V\rig\Vert_2^2\leq\eta^2 n^2d^2\max_{i,r,t}\Vert\bfa A_{i,r}^{(t)}\Vert_2^4\Vert\bfa U(t - 1) -\bfa V\Vert_2^2.
\end{align}
where we define the matrix
$    \bfa H^{(t)}=\begin{bmatrix}
    \bfa H_{1,1}^{(t)}&\bfa H_{1,2}^{(t)}&\ldots&\bfa H_{1,n}^{(t)}\\
    \bfa H_{2,1}^{(t)}&\bfa H_{2,2}^{(t)}&\ldots&\bfa H_{2,n}^{(t)}\\
    \vdots&\vdots &\vdots &\vdots\\
    \bfa H_{n,1}^{(t)}&\bfa H_{n,2}^{(t)}&\ldots&\bfa H_{n,n}^{(t)}
    \end{bmatrix}
$. For the $T_2$ term, we note that with probability at least $1-\delta$,
\begin{align}\label{T2formula}
    T_2&=\sum_{i=1}^n 2(\bfa u_i^{(t-1)}-\vb(\xb_{\tau_i}|\xb_{0,i}))^\top(\bfa u_i^{(t)}-\bfa u_i^{(t-1)})\nnb\\
    &= 2\eta\sum_{i=1}^n\sum_{j=1}^n(\bfa u_j^{(t-1)}-\vb_{\tau_j}(\xb_{\tau_j}|\xb_{1,j}))^\top(\bfa H_{i,j}^{(t)}-\bfa H_{i,j}^{\perp,(t)})(\bfa u_j^{(t-1)}-\vb_{\tau_j}(\xb_{\tau_j}|\xb_{1,j}))\nnb\\
    &=-2\eta\lef(\bfa U^{(t-1)} -\bfa V\rig)\bfa H^{(t)}\lef(\bfa U^{(t-1)} - \bfa V\rig)^\top\nnb\\
    &\leq -2\eta(\bfa U^{(t-1)} - \bfa V)\bfa H^{(\infty)}(\bfa U^{(t-1)}-\bfa V)^\top +\eta\Vert\bfa H^{(\infty)} -\bfa H^{(t)}\Vert_2\Vert\bfa U^{(t-1)}-\bfa V\Vert_2^2\nnb\\
    &\leq -2\eta\lambda_{\min}(\bfa H^{(\infty)})\Vert\bfa U^{(t-1)} - \bfa V\Vert_2^2 + \eta\Vert\bfa H^{(\infty)}-\bfa H^{(t)}\Vert_2\Vert\bfa U^{(t-1)}-\bfa V\Vert_2^2.
\end{align}
Then we consider the event $$\ca E:=\lef\{\max_{i\in[n],r\in[m],j\in[t]}\Vert\bfa A_{r,i}^{(j)} \Vert_2\leq C(d,m,n)\rig\}.$$ where the $C(d,m,n)$ are to be decided later. Note that by definition of $\bfa A$, we have
\begin{align*}
    \max_{i\in[n],r\in[m],j\in[t]}\Vert\ab_r^{(t)}[\xb_{\tau_i,i}^\top,\tau_i,\xb_{0,i}^\top]\Vert_2\leq \max_{i\in[n],r\in[m],j\in[t]}\Vert\ab_r^{(j)}\Vert_2\Vert[\xb_{\tau_i,i}^\top,\tau_i,\xb_{0,i}^\top]\Vert_2.
\end{align*}
For the first term, we can utilize the fact that $\ab_r^{(0)}\sim N(0, I_d)$ and lemma \ref{lm:2.4} to show that with probability at least $1-\delta$
\begin{align*}
\max_{r\in[m],j\in[t]}\Vert\ab_r^{(j)}\Vert_2&\leq \max_{r\in[m]}\Vert\ab_r^{(0)}\Vert_2 + \max_{r\in[m],j\in[t]}\Vert\ab_r^{(j)}-\ab_r^{(0)}\Vert_2\\
&\leq  \max_{r\in[m], j\in[t]}\frac{\eta\sqrt n}{(1-C^{\frac{1}{2}})\sqrt m}\Vert\bfa U^{(0)}-\bfa V\Vert_2\Vert[\xb_{\tau_i,i}^\top,\tau_i,\xb_{0,i}^\top]\Vert_2\Vert\wb_r^{(j)}\Vert_2\\
&+\max_{r\in[m]}\Vert\ab_r^{(0)}\Vert_2.
\end{align*}
We further note that with probability at least $1-\delta$,
\begin{align*}
    \max_{r\in[m]}\Vert\ab_r^{(0)}\Vert_2\leq C\sqrt d\log\frac{m}{\delta}, \qquad \max_{r\in[m]}\Vert\wb_r^{(j)}\Vert_2\leq C\kappa\sqrt d\log\frac{m}{\delta}.
\end{align*}
Therefore, we can finally show that with $m\gtrsim n\eta^2d\Vert\bfa U^{(0)}-\bfa V\Vert_2^2\log\frac{m}{\delta}$, we have with probability at least $1-\delta$,
\begin{align*}
    \max_{r\in[m]}\Vert\ab_r^{(j)}\Vert_2\leq C\sqrt d\log\frac{m}{\delta}.
\end{align*}
Similarly, we note that
\begin{align*}
    \Vert\xb_{\tau_i}\Vert_2 = \Vert\sigma_{\tau_i}\cdot\xb_{0,i} + \mu_{\tau_i}\cdot\xb_{1,i}\Vert_2\leq \sigma_{\tau_i}\Vert\xb_{0,i}\Vert_2 + \mu_{\tau_i}\Vert\xb_{1,i}\Vert_2.
\end{align*}
And with probability at least $1-\delta$, we can show that by union bound,
\begin{align*}    \sup_{i\in[n]}\lef(\Vert\xb_{0,i}\Vert_2+\Vert\xb_{1,i}\Vert_2\rig)\leq C\sqrt{d}\log\frac{n}{\delta}.
\end{align*}
Hence, given $\sigma_{\tau_i} = O(1)$ and $\mu_{\tau_i} = O(1)$, we can show that with probability at least $1-\delta$,
\begin{align*}
    \sup_{i\in[n]}\Vert[\xb_{\tau_i},\tau_i]\Vert_2\leq \sup_{i\in[n]}\Vert\xb_{\tau_i}\Vert_2 + 1\leq C\sqrt{d}\log{\frac{n}{\delta}}.
\end{align*}
Hence, collecting the two pieces for $\max_{i\in[n]}\Vert[\xb_{\tau_i},\tau_i]\Vert_2$ and $\max_{r\in[m],j\in[t]}\Vert\ab_r^{(j)}\Vert_2$, we can show that with probability at least $1-\delta$, we have
\begin{align}\label{abound}
    \max_{i\in[n], r\in[m], j\in[t]}\lef\Vert\bfa A_{r,i}^{(j)}\rig\Vert_2\leq Cd\log\frac{n}{\delta}\log\frac{m}{\delta}.
\end{align}
And we just let $C(d, m, n) = Cd\log\frac{m}{\delta}\log\frac{n}{\delta}$.
And collecting \eqref{T1formula}, \eqref{T2formula}, and \eqref{Formulae} we conclude that with probability at least $1-\delta$, given the fact that $R=\frac{C\eta\sqrt n}{(1-C^{1/2})\sqrt m} \Vert\bfa U^{(0)}-\bfa V\Vert_2\max_{i\in[n],j\in[t]}\Vert\bfa A_{r,i}^{(j)}\Vert_2$, we can show that
\begin{align*}
    &\Vert\bfa U^{(t)}-\bfa V\Vert_2^2\leq\Big(1+n^2d^2\eta^2-2\eta\lambda_{\min}(\bfa H^{(\infty)}) \\
    &+\eta(C\max_{i\in[n],r\in[m],j\in[t]}\Vert\bfa A_{r,i}^{(j)}\Vert_2^3\frac{\eta n\sqrt n}{\sqrt m}\Vert\bfa U^{(0)}-\bfa V\Vert_2 \\
    &+ \frac{Cn^2R}{\delta}\max_{i\in[n],r\in[m], j\in[t]}\Vert\bfa A_{r,i}^{(j)}\Vert_2)\Big)\Vert\bfa U^{(t-1)}-\bfa V\Vert_2^2\\
    &\leq \Big(1+n^2d^2\eta^2-2\eta\lambda_{\min}(\bfa H^{(\infty)}) \\
    &+\eta(C\max_{i\in[n],r\in[m],j\in[t]}\Vert\bfa A_{r,i}^{(j)}\Vert_2^3\frac{\eta n\sqrt n}{\sqrt m}\Vert\bfa U^{(0)}-\bfa V\Vert_2 \\
    &+ \frac{Cn^2R}{\delta}\max_{i\in[n],r\in[m], j\in[t]}\Vert\bfa A_{r,i}^{(j)}\Vert_2)\Big)\Vert\bfa U^{(t-1)}-\bfa V\Vert_2^2.
\end{align*}
Using the upper bound for $\max_{i\in[n],r\in[m],j\in[t]}\Vert\bfa A_{i,r}^{(t)}\Vert_2$, we can show that given $m\gtrsim \frac{n^5\Vert\bfa U^{(0)}-\bfa V\Vert_2^2}{\lambda_{\min}(\bfa H^{(\infty)})\delta^3}d\log nd$ and $\eta <\frac{2\lambda_{\min}(\bfa H^{(\infty)})}{n^2d^2} $, with probability at least $1-\delta$,
\begin{align*}
    \Vert\bfa U^{(t)}-\bfa V\Vert_2^2\leq \Big(1-\frac{\eta\lambda_{\min}(\bfa H^{(\infty)})}{2}\Big)^{t-1}\Vert\bfa U^{(0)}-\bfa V\Vert_2^2.
\end{align*}
And going back to \eqref{Formulae} we can re-write its vectorized form as
\begin{align*}
    \bfa U^{(t)} - \bfa V  = \bfa U^{(t)} - \bfa U^{(t-1)} + \bfa U^{(t-1)}-\bfa V.
\end{align*}
And we note that
\begin{align*}
    \bfa U^{(t)} - \bfa U^{(t-1)}&=
        \ub{\frac{1}{\sqrt m}\sum_{r\in S_i}\Big(\ab_r^{(t)}\sigma(\wb_r^{(t),\top}\tde\xb_{\tau_i}) - \ab_r^{(t-1)}\sigma(\wb_r^{(t-1),\top}\tde\xb_{\tau_i})\Big)}_{\zeta_1} \\
        &+ \ub{\frac{1}{\sqrt m}\sum_{r\in S_i^c}\Big(\ab_r^{(t)}\sigma(\wb_r^{(t),\top}\tde\xb_{\tau_i}) - \ab_r^{(t-1)}\sigma(\wb_r^{(t-1),\top}\tde\xb_{\tau_i})\Big)}_{\zeta_2}.
\end{align*}
And we can show that for $\xi_1$, we have with probability at least $1-\delta$ for all $i\in[n]$,
\begin{align*}
    \zeta_{1,i} &= \frac{1}{\sqrt m}\sum_{r\in S_i}\lef(\ab_r^{(t)}\wb_r^{(t),\top}-\ab_r^{(t-1)}\wb_r^{(t-1),\top}\rig)\mbbm 1_{\wb_r^{(t-1),\top}\tde\xb_{\tau_i} > 0}\tde\xb_{\tau_i}\\
    &= \frac{1}{\sqrt m}\sum_{r\in S_i}\Big(\Big(\ab_r^{(t-1)} + \eta\frac{\pta L(\bfa\theta^{(t-1)})}{\pta\ab_r}\Big)\Big(\wb_r^{(t-1),\top} +\eta\frac{\pta L(\bfa\theta^{(t-1)})^\top}{\pta\wb_r}\Big) - \ab_r^{(t-1)}\wb_r^{(t-1)}\Big)\mbbm 1_{\wb_r^{(0),\top}\tde\xb_{\tau_i}>0}\tde\xb_{\tau_i}\\
    &=\frac{1}{\sqrt m}\sum_{r\in S_i}\eta \Big(\frac{\pta L(\bfa\theta^{(t-1)})}{\pta\ab_r}\wb_r^{(t-1),\top} + \ab_r^{(t-1)}\frac{\pta L(\bfa\theta^{(t-1)})^\top}{\pta\wb_r}\Big)\mbbm 1_{\wb_r^{(t-1),\top}\tde\xb_{\tau_i} > 0}\tde\xb_{\tau_i} \\
    &+ \eta^2\frac{\pta L(\bfa\theta^{(t-1)})}{\pta \ab_r}\Big(\frac{\pta L(\bfa\theta^{(t-1)})}{\pta \wb_r}\Big)^\top \mbbm 1_{\wb_r^{(t-1),\top}\tde\xb_{\tau_i} > 0} \tde\xb_{\tau_i}.
\end{align*}
Hence, with probability at least $1-\delta$, we can show that
\begin{align*}
       \Vert\zeta_1\Vert_2&= -\eta\bfa H^{(t-1)}\lef(\bfa U^{(t-1)} -\bfa V\rig) + O\Big(\frac{n\eta^2}{m}\Vert\bfa U^{(t)} - \bfa V\Vert_2^2 \max_{i\in[n]}\Vert\bfa A_{r,i}^{(t-1)}\Vert_2\Big)\\
    & = -\eta\bfa H^{(\infty)}\Big(\bfa U^{(t-1)}-\bfa V\Big) +\eta\Vert\bfa H^{(t-1)}- \bfa H^{(\infty)}\Vert_2\Vert\bfa U^{(t-1)} - \bfa V\Vert_2 \\
    &+O\Big(\frac{n\eta^2}{m}\Vert\bfa U^{(t)} - \bfa V\Vert_2^2 \max_{i\in[n]}\Vert\bfa A_{r,i}^{(t-1)}\Vert_2\Big)\\
    &=-\eta\bfa H^{(\infty)}(\bfa U^{(t-1)}-\bfa V) + \eta\Big(C\frac{n^3d^3}{\sqrt m\delta}\log^3\frac{m}{\delta} +\frac{Cnd\kappa }{\lambda_{\min}(\bfa H^{(\infty)})\delta} \Big).
\end{align*}
And for $\zeta_2$, we can show that with probability at least $1-\delta$, we have for all $i\in[n]$
\begin{align*}
    \Vert\zeta_{2,i}\Vert_2&= \Big\Vert\frac{1}{\sqrt m}\sum_{r\in S_i^c}\Big(\ab_r^{(t)}\sigma(\wb_r^{(t),\top}\tde\xb_{\tau_i})-\ab_r^{(t-1)}\sigma(\wb_r^{(t-1),\top}\tde\xb_{\tau_i})\Big)\Big\Vert_2\leq\frac{|S_i^c|}{\sqrt m}d\kappa\log\frac{n}{\delta}\\
    &\leq \frac{\eta n^{5/2}\log(n/\delta)}{\sqrt m\kappa \lambda_{\min}(\bfa H^{(\infty)})\delta^{3/2}}.
\end{align*}
Hence we conclude that, with $\Vert\bfa\epsilon^{(k)}\Vert_2\leq C\lef(\frac{nd\kappa }{\lambda_{\min}(\bfa H^{(\infty)})\delta} + \frac{n^3d^3}{\sqrt m\delta}\log^3\frac{m}{\delta} +\frac{\eta n^{3}\log(n/\delta)}{\sqrt m\kappa\lambda_{\min}(\bfa H^{(\infty)})\delta^{3/2}}\rig)$ for all $k\in[t-1]$,
\begin{align*}
    \bfa U^{(t)} -\bfa V&= (I-\eta\bfa H^{(\infty)})(\bfa U^{(t-1)} - \bfa V) + \bfa\epsilon^{(t-1)}\\
    &=(I - \eta\bfa H^{(\infty)})^{t-1}(\bfa U^{(0)}-\bfa V) +\sum_{k=0}^{t-1}\lef(I-\eta\bfa H^{(\infty)}\rig)^k\bfa\epsilon^{(k)}\\
    &=-(I -\eta\bfa H^{(\infty)})^{t-1}\bfa V +(I-\eta\bfa H^{(\infty)})\bfa U^{(0)}+\sum_{k=0}^{t-1}\Big(I-\eta\bfa H^{(\infty)}\Big)^k\bfa\epsilon^{(k)}\\
    &=-(I-\eta\bfa H^{(\infty)})^{t-1}\bfa V +\bfa E^{(t-1)},
\end{align*}
We note that $\Vert\bfa U^{(0)}\Vert_2\leq \frac{Cnd\kappa }{\lambda_{\min}(\bfa H^{(\infty)})\delta}$ with probability at least $1-\delta$. Then our final result wrap up at
\begin{align*}
    \bfa U^{(t)} -\bfa V = (I - \eta\bfa H^{(\infty)})^{t-1}\bfa V +\bfa E^{(t-1)}.
\end{align*}
where $\bfa E^{(t-1)}$ satisfies $\lef\Vert\bfa E^{(t-1)}\rig\Vert_2\leq C\lef(\frac{Cnd\kappa }{\lambda_{\min}(\bfa H^{(\infty)})\delta} + \frac{n^3d^3}{\sqrt m\delta}\log^3\frac{m}{\delta} +\frac{\eta n^3\log(n/\delta)}{\sqrt m\kappa\lambda_{\min}(\bfa H^{(\infty)})\delta^{3/2}}\rig)$.

\section{Proof of Theorem \ref{thmgenbound}}\label{app:proof-thmgenbound}
Similar to the proof of theorem \ref{theorem1}, we will restate all the intermediate results in this section for the completeness of the proof.

\begin{lemma}
    Suppose the loss function $\ell(\cdot,\cdot)$ satisfies $\bb E[\sup_{\bfa f\in\ca F}\cosh(\lambda \ell(f(\xb),y))]\leq\exp(\frac{1}{2}\sigma^2\lambda^2)$. And, $\ell(f(\xb),y)$ is $\rho$-Lipschitz in the first argument in Euclidean norm. Then with probability at least $1-\delta$ over the samples $\{\xb_{i}\}_{i\in[n]}$,
    \begin{align*}
        \sup_{f\in\ca F}\Big\{\bb E[\ell(f(\xb), y)] - \frac{1}{n}\sum_{i=1}^n\ell(f(\xb_i),y_i)\Big\}\leq 2\rho{\ca R}_{\xb}(\ca F)
 + \sigma\sqrt{\frac{2\log(1/\delta)}{n}}.\end{align*}
\end{lemma}
\begin{proof}
    Our proof simply follows from \citep{meir2003generalization} theorem 3 with the contraction bound of multi-task Rademacher complexity given by \citep{maurer2016vector}.
    \end{proof}
To simplify the notations, we introduce the following terms
    \begin{align*}
    \ca Z^{(t)}&=\frac{1}{\sqrt m}\begin{bmatrix}
        [\xb_{\tau_1}^\top,\tau_1,\xb_{0,1}]^\top\mbbm 1_{[\xb_{\tau_1}^\top,\tau_1]\wb_1^{(t)} > 0} & \ldots &[\xb_{\tau_1}^\top,\tau_1,\xb_{0,1}]^\top\mbbm 1_{[\xb_{\tau_1}^\top,\tau_1]\wb_m^{(t)} > 0}\\
        \vdots & \vdots &\vdots\\
        [\xb_{\tau_n}^\top,\tau_n,\xb_{0,n}]^\top\mbbm 1_{[\xb_{\tau_n}^\top,\tau_n]\wb_1^{(t)} >0}&\ldots &[\xb_{\tau_n}^\top,\tau_n,\xb_{0,n}]^\top\mbbm 1_{[\xb_{\tau_n}^\top,\tau_n]\wb_m^{(t)} > 0}
    \end{bmatrix},\quad\ca A^{(t)} =\begin{bmatrix}
        \ab_1^{(t)}&\bfa 0&\ldots&\bfa 0\\
        \bfa 0&\ab_2^{(t)}&\ldots &\bfa 0\\
        \bfa 0&\bfa 0&\ldots&\ab_m^{(t)}
    \end{bmatrix},\\
    \ca W^{(t)} &= \begin{bmatrix}
        \wb_1^{(t)}&\bfa 0&\ldots&\bfa 0\\
        \bfa 0 &\wb_2^{(t)}&\ldots&\bfa 0\\
        \bfa 0&\bfa 0&\ldots &\wb_m^{(t)}
    \end{bmatrix},\quad \Delta^{(t)}=  \begin{bmatrix}
        f_{\bfa\theta^{(t)}}(\xb_{\tau_1},\tau_1) - \vb\lef(\xb_{\tau_1}|\xb_{1,1}\rig)\\ f_{\bfa\theta^{(t)}}(\xb_{\tau_2},\tau_2) -\vb(\xb_{\tau_2}|\xb_{1,2})\\\vdots \\ f_{\bfa\theta^{(t)}}(\xb_{\tau_n},\tau_n) -\vb(\xb_{\tau_n}|\xb_{1,n})
    \end{bmatrix},\quad
    \bfa V^\top =\begin{bmatrix}
        \vb(\xb_{\tau_1}|\xb_{1,1})\\
        \vb(\xb_{\tau_2}|\xb_{1,2})\\
        \vdots\\
        \vb(\xb_{\tau_n}|\xb_{1,n})
    \end{bmatrix},\\
    \Sigma^{(t)} &= \frac{1}{\sqrt m}\begin{bmatrix}
\sigma([\xb_{\tau_1}^\top,\tau_1,\xb_{1}^\top]\wb_1^{(t)}) I_d&\sigma([\xb_{\tau_2}^\top,\tau_2,\xb_{0,2}^\top]\wb_1^{(t)}) I_d&\ldots &\sigma([\xb_{\tau_n}^\top,\tau_n,\xb_{0,n}^\top]\wb_1^{(t)})I_d\\
\vdots &\vdots&\ddots&\vdots\\
\sigma([\xb_{\tau_1}^\top,\tau_1,\xb_{0,1}^\top]\wb_m^{(t)})I_d&\sigma([\xb_{\tau_2}^\top,\tau_2,\xb_{0,2}^\top]\wb_{m}^{(t)}) I_d&\ldots &\sigma([\xb_{\tau_n}^\top,\tau_n,\xb_{0,n}^\top]\wb_m^{(t)})I_d
    \end{bmatrix},\\
    \tda A^{(t)}&= \begin{bmatrix}
        \ab_1^{(t)}&\ab_2^{(t)}&\ldots&\ab_m^{(t)}
    \end{bmatrix}.
\end{align*}
The following two lemmas provide upper bounds to $\Vert\ca Z^{(t)}- \ca Z^{(0)}\Vert_2$ and $\Vert\Sigma^{(t)} - \Sigma^{(0)}\Vert_2$.

We then provide the following two sharper bounds that simultaneously controls the difference of the deviation of all neurons.
\begin{lemma}\label{zdiff}
    With probability at least $1-\delta$ over the random initialization, we have
    \begin{align*}
        \Vert \ca Z^{(t)} - \ca Z^{(0)}\Vert_F^2&\leq\frac{1}{\delta}\sum_{i=1}^n\frac{C\eta n\Vert\tde\xb_{\tau_i}\Vert_2^3}{\sqrt m \kappa^2}\log^2m,\\ \Vert\Sigma^{(t)}-\Sigma^{(0)}\Vert_F^2&\leq\frac{1}{\delta}\sum_{i=1}^n\Big(\frac{C\eta^2n^2d^3\Vert\tde\xb_{\tau_i}\Vert_2^2\log^2 n\log^2m}{m} + \frac{C\eta nd^2\Vert\tde\xb_{\tau_i}\Vert_2^3\log^2m}{\sqrt m\kappa^2}\Big).
    \end{align*}
\end{lemma}
\begin{proof}
Recall the following definition
\begin{align*}
    A_{r,i}(R) = \{|\wb_r^{(0),\top}\xb_i|\leq R\},\qquad i\in[n],r\in[m].
\end{align*}
Then, it is not hard to notice that
\begin{align*}
    \mbbm 1_{\tde\xb_{\tau_i}^\top\wb_i^{(t)}\tde\xb_{\tau_i}^\top\wb_r^{(t)} < 0} \leq \mbbm 1_{A_{r,i}(R)} + \mbbm 1_{\Vert\wb_{r}^{(t)}-\wb_r^{(0)}\Vert_2\Vert\tde\xb_{\tau_i}\Vert_2\geq R}.
\end{align*}
Recall that $\wb_r^{(0)}$ is Isotropic Guassian $N(0, \kappa ^2I_p)$, we conclude that
\begin{align}\label{aineq}
    \bb E\lef[\mbbm 1_{A_{r,i}(R)}\Big|\Vert\tde\xb_{\tau_i}\Vert_2\rig] =\int_{-\frac{R}{\lef\Vert\tde\xb_{\tau_i}\rig\Vert_2\kappa}}^{\frac{R}{\Vert\tde\xb_{\tau_i}\Vert_2\kappa}}\frac{C}{\sqrt{2\pi\Vert\tde\xb_{\tau_i}\Vert_2}d}\exp\lef(-\frac{x^2}{2d\Vert\tde\xb_{\tau_i}\Vert_2^2}\rig)dx\leq\frac{CR}{d\Vert\tde\xb_{\tau_i}\Vert_2\kappa^2}.
\end{align}

Then we utilize lemma \ref{lm:2.4} and condition \ref{inductivecond} to show that
\begin{align*}
    \bb E\Big[\mbbm 1_{\Vert\wb_r^{(t)} - \wb_r^{(0)}\Vert_2\Vert\tde\xb_{\tau_i}\Vert_2\geq R}\Big|\Vert\tde\xb_{\tau_i}\Vert_2 = a\Big]&=\bb P\lef(\Vert\wb_r^{(t)}-\wb_r^{(0)}\Vert_2\geq\frac{R}{a}\rig)\\
    &\leq\sqrt{mn}\exp\Big(-\frac{1}{2}\sqrt{\log^2\lef(\frac{n}{m}\rig)+\frac{CR\sqrt m}{\eta n ad}}\Big).
\end{align*}
    We note that for $\ca Z^{(t)}-\ca Z^{(0)}$ we have
    \begin{align*}
        \bb E\Big[\Vert\ca Z^{(t)} &- \ca Z^{(0)}\Vert_F^2\Big|\{\tde\xb_{\tau_i}\}_{i\in[n]}\Big]= \frac{1}{m}\sum_{i=1}^n\bb E\Big[\sum_{j=1}^m\Vert\tde\xb_{\tau_i}\Vert_2^2\Big(\mbbm 1_{\tde\xb_{\tau_i}^\top\wb_i^{(t)} > 0} - \mbbm 1_{\tde\xb_{\tau_i}^\top\wb_i^{(0)} > 0}\Big)^2\Big|\{\tde\xb_{\tau_i}\}_{i\in[n]}\Big]\\
        &=\frac{1}{m}\sum_{i=1}^n\sum_{j=1}^{m}\Vert\tde\xb_{\tau_i}\Vert_2^2\bb E\Big[\mbbm 1_{\tde\xb_{\tau_i}^\top\wb_i^{(t)}\tde\xb_{\tau_i}^\top\wb_i^{(0)}<0}\Big|\{\tde\xb_{\tau_i}\}_{i\in[n]}\Big]\\
&\leq\frac{1}{m}\sum_{i=1}^n\sum_{r=1}^m\Vert\tde\xb_{\tau_i}\Vert_2^2\bb E\lef[\mbbm 1_{\ca A_{r,i}(R_{\wb}\Vert\tde\xb_{\tau_i}\Vert_2)} + \mbbm 1_{\Vert\wb_r^{(t)}-\wb_{r}^{(0)}\Vert\geq R_{\wb}}\Big|\{\xb_{\tau_i}\}_{i\in[n]}\rig]\\
&\leq \frac{1}{m}\sum_{i=1}^n\sum_{r=1}^m\Vert\tde \xb_{\tau_i}\Vert_2^2\Big(\frac{CR_{\wb}}{d\kappa^2} + C\sqrt{mn}\exp\Big(-\sqrt{\log^2\lef(\frac{n}{m}\rig)+ \frac{CR_{\wb}\sqrt m}{\eta n \Vert\tde\xb_{\tau_i}\Vert_2d}}\Big)\Big).
    \end{align*}
Hence, by Markov's inequality, with probabilty at least $1-\delta$, we have for all $R_{\wb} > 0$,
\begin{align*}
    \lef\Vert\ca Z^{(t)} - \ca Z^{(0)}\rig\Vert_F^2\leq \frac{1}{\delta}\sum_{i=1}^n\Vert\tde\xb_{\tau_i}\Vert_2^2\Big(\frac{C R_{\wb}}{d\kappa^2} + C\sqrt{mn}\exp\Big(-\sqrt{\log^2\Big(\frac{n}{m}\Big) + \frac{CR_{\wb}\sqrt m}{\eta n\Vert\tde\xb_{\tau_i}\Vert_2d}}\Big)\Big).
\end{align*}
Optimizing over $R_{\wb}$, we can show that with probability at least $1-\delta$,
\begin{align*}
    \Vert\ca Z^{(t)} - \ca Z^{(0)}\Vert_F^2\leq \frac{1}{\delta}\sum_{i=1}^n\frac{C\eta n\Vert\tde\xb_{\tau_i}\Vert_2^3}{\kappa^2\sqrt m}\log^2m.
\end{align*}
And similarly we can show that
\begin{align*}
    &\bb E\Big[\Big\Vert \Sigma^{(t)} - \Sigma^{(0)}\Big\Vert_F^2\Big| \{\tde\xb_{\tau_i}\}_{i\in[n]}\Big] =\frac{d}{m}\sum_{i=1}^n\sum_{j=1}^{m}\bb E\Big[\Big(\sigma([\xb_{\tau_i,i}^\top,\tau_i,\xb_{0,i}^\top]\wb_j^{(t)}) - \sigma([\xb_{\tau_i,i}^\top,\tau_i,\xb_{0,i}^\top]\wb_j^{(0)})\Big)^2\Big]\\
    &\leq\frac{d}{m}\sum_{i=1}^n\sum_{j=1}^m\bb E\Big[\Big([\xb_{\tau_i,i}^\top,\tau_i,\xb_{0,i}^\top]\Big(\wb_j^{(t)}- \wb_j^{(0)}\Big)\Big)^2\Big|\tde\xb\Big] + \frac{d}{m}\sum_{i=1}^n\sum_{j=1}^m\bb E\Big[(|\tde\xb_{\tau_i}^\top\wb_j^{(t)}|^2 + |\tde\xb_{\tau_i}^\top\wb_j^{(0)}|^2)\mbbm 1_{\tde\xb_{\tau_i}^\top\wb_i^{(t)}\tde\xb_{\tau_i}^\top\wb_i^{(0)}<0}\Big|\tde\xb\Big]\\
    &\leq d\sum_{i=1}^n\Vert\tde\xb_{\tau_i}\Vert_2^2\bb E\lef[\Vert\wb_j^{(t)}-\wb_j^{(0)}\Vert_2^2\Big|\tde\xb\rig] + d\sum_{i=1}^n\bb E\Big[\lef(\Vert\wb_j^{(t)}\Vert_2^2 + \Vert\wb_j^{(0)}\Vert_2^2\rig)\mbbm 1_{\tde\xb_{\tau_i}^\top\wb_i^{(t)}\tde\xb_{\tau_i}^\top\wb_i^{(0)}<0}\Big|\tde\xb\Big]\Vert\tde\xb_{\tau_i}\Vert_2^2\\
    &\leq \sum_{i=1}^n\Vert\tde\xb_{\tau_i}\Vert_2^2 \frac{C\eta^2n^2d^3 \log^2 n\log^2 m}{m} + \sum_{i=1}^n\frac{C\eta nd\Vert\tde\xb_{\tau_i}\Vert_2^3}{\sqrt m}\log^2m\\
    &\leq\sum_{i=1}^n\lef(\Vert\tde\xb_{\tau_i}\Vert_2^2\frac{C\eta^2n^2d^3\log^2 n\log^2m}{m} + \frac{C\eta nd^2\Vert\tde\xb_{\tau_i}\Vert_2^3\log^2m}{\sqrt m}\rig).
\end{align*}
And we finish the proof by showing that with probability at least $1-\delta$,
\begin{align*}
    \Vert\Sigma^{(t)} - \Sigma^{(0)}\Vert_F^2 &\leq\frac{1}{\delta}\sum_{i=1}^n\Big(\Vert\tde\xb_{\tau_i}\Vert_2^2\frac{C\eta^2n^2d^3\log^2 n\log^2m}{m} + \frac{C\eta nd^2\Vert\tde\xb_{\tau_i}\Vert_2^3\log^2m}{\kappa^2\sqrt m}\Big).
\end{align*}
\end{proof}

And then we can analyze the difference between $\vect(\bfa W^{(t)})$ and $\vect(\bfa W^{(0)})$ as follows
\begin{lemma}
    For the two vectors $\vect(\bfa W^{(t)})$ and $\vect(\tda A^{(t)})$ we can show that
    \begin{align*}
    \Vert\ca W^{(t)} - \ca W^{(0)}\Vert_F&= \Vert\vect(\bfa W^{(t)} - \bfa W^{(0)})\Vert_2 \\
    &\leq \sqrt{\bfa V\bfa T\bfa B^{(\infty)}\bfa T\bfa V^\top} + \frac{C\cdot\poly\lef(n,\lambda_{\min}^{-1}(\bfa H^{(\infty)}), d, 1/\delta\rig)}{m^{1/4}\kappa^{1/2}}+\frac{Cnd\kappa }{\lambda_{\min}(\bfa H^{(\infty)})\delta}.\\
    \Vert\ca A^{(t)} - \ca A^{(0)}\Vert_F&= \Vert\vect(\bfa A^{(t)} - \bfa A^{(0)})\Vert_2 \\
    &\leq  \sqrt{\bfa V\bfa T\bfa C^{(\infty)}\bfa T\bfa V^\top}+ \frac{C\cdot\poly(n, \lambda_{\min}^{-1}(\bfa H^{(\infty)}) ,d,1/\delta)}{m^{1/4}\kappa^{1/2}}+\frac{Cnd\kappa }{\lambda_{\min}(\bfa H^{(\infty)})\delta}.
    \end{align*}
\end{lemma}
\begin{proof}
    We first consider the update rule for $\{\wb_r\}_{r\in[m]}$ and $\{\ab_r\}_{r\in[m]}$. We notice that for $\bfa W^{(t)}=\begin{bmatrix}
        \wb_1^{(t)},\ldots,\wb_m^{(t)}
    \end{bmatrix}$, using the gradient descent update rule we note that
    \begin{align*}
    \vect(\bfa W^{(t+1)}) = \vect(\bfa W^{(t)}) -\eta\ca Z^{(t),\top}\ca A^{(t)}\Delta^{(t)}.
    \end{align*}
    Then, we use the fact that
    \begin{align*}
       \Delta^{(t)} = -\lef(I-\eta\bfa H^{(\infty)}\rig)^{t}\bfa V^\top + \bfa E^{(t)}.
    \end{align*}
    And we can further show that
    \begin{align*}
         &\vect\lef(\bfa W^{(t)}\rig) -\vect\lef(\bfa W^{(0)}\rig) = \sum_{k=0}^{t-1}\vect(\bfa W^{(k+1)})-\vect(\bfa W^{(k)})=-\sum_{k=0}^{t-1}\eta\ca Z^{(t),\top}\ca A^{(k)}\Delta^{(k), \top}\\
         &=\sum_{k=0}^{t-1}\eta\ca Z^{(k),\top}\ca A^{(k)}\lef(I-\eta\bfa H^{(\infty)}\rig)^k\bfa V^\top -\sum_{k=0}^{t-1}\eta\ca Z^{(k),\top}\ca A^{(k)}\lef(I-\eta\bfa H^{(\infty)}\rig)^t\bfa E^{(k)}\\
         &=\ub{\sum_{k=0}^{t-1}\eta\ca Z^{(0),\top}\ca A^{(0)}\lef( I -\eta\bfa H^{(\infty)}\rig)^k\bfa V^\top}_{\xi_1} +\ub{\sum_{k=0}^{t-1}\eta\ca Z^{(k),\top}\Big(\ca A^{(k)}-\ca A^{(0)}\Big)\lef(I-\eta\bfa H^{(\infty)}\rig)^t\bfa V^\top}_{\xi_2}\\
         &-\ub{\sum_{k=0}^{t-1}\eta\ca Z^{(k),\top}\ca A^{(k)}\lef(I-\eta\bfa H^{(\infty)}\rig)^k\bfa E^{(k)}}_{\xi_3} + \ub{\sum_{k=0}^{t-1}\eta\Big(\ca Z^{(k),\top}- \ca Z^{(0),\top}\Big)\Big(\ca A^{(k)}- \ca A^{(0)}\Big)\Big( I -\eta\bfa H^{(\infty)}\Big)^t\bfa V^\top}_{\xi_4}.
    \end{align*}
    And similarly, we note that for $\{\ab_r\}_{r\in[m]}$ we can show that
    \begin{align*}
        &\vect\lef(\tda A^{(t)}\rig) - \vect(\tda A^{(0)}) = \sum_{k=0}^{t-1}\vect(\tda A^{(k+1)}) -\vect(\tda A^{(k)}) = -\sum_{k=0}^{t-1}\eta \frac{\pta L(\bfa\theta^{(t-1)})}{\pta \vect(\tda A^{(t-1)})} \\
        &=-\sum_{k=0}^{t-1}\eta\Sigma^{(k)}\Delta^{(t)} = \sum_{k=0}^{t-1}\eta\Sigma^{(k)}\lef(I - \eta\bfa H^{(\infty)}\rig)^k\bfa V^\top-\sum_{k=0}^{t-1}\eta\Sigma^{(k)}\lef(I-\eta\bfa H^{(\infty)}\rig)^t\bfa E^{(k)}\\
        &=\ub{\sum_{k=0}^{t-1}\eta\Sigma^{(0)}\lef(I-\eta\bfa H^{(\infty)}\rig)^k\bfa V^\top}_{\xi_5}+\ub{\sum_{k=0}^{t-1}\eta(\Sigma^{(k)}-\Sigma^{(0)})\lef(I -\eta\bfa H^{(\infty)}\rig)^k\bfa V^\top}_{\xi_6}\\
        &-\ub{\sum_{k=0}^{t-1}\eta\Sigma^{(k)}\lef(I-\eta\bfa H^{(\infty)}\rig)^k\bfa E^{(k)}}_{\xi_7}.
    \end{align*}
    Define $\bfa T =\eta\sum_{k=0}^{t-1}(I - \eta\bfa H^{(\infty)})$. And we consider the terms $\xi_1,\ldots,\xi_6$ separately, note that
    \begin{align*}
        \Vert\xi_1\Vert_2^2&=\Big\Vert\ca Z^{(0),\top}\ca A^{(0)}\sum_{k=0}^{t-1}\eta\lef(I - \eta\bfa H^{(\infty)}\rig)^k\bfa V^\top\Big\Vert_2^2 = \bfa V\bfa T\ca A^{(0),\top}\ca Z^{(0)}\ca Z^{(0),\top}\ca A^{(0)}\bfa T\bfa V^\top\\
        &\leq \bfa V\bfa T\bfa B^{(\infty)}\bfa T\bfa V^\top + \lef\Vert\bb E\lef[\ca A^{(0),\top}\bfa B^{(\infty)}\ca A^{(0)}\rig]-\ca A^{(0),\top}\ca Z^{(0)}\ca Z^{(0),\top}\ca A^{(0)}\rig\Vert_2\Vert\bfa T\Vert_2^2\Vert\bfa V\Vert_2^2.
    \end{align*}
    We note that with probability at least $1-\delta$,
$        \Vert\bfa H^{(\infty)} -\bfa H^{(0)}\Vert_2\leq \frac{\sqrt dn{\log\frac{n}{\delta}}\log\frac{m}{\delta}}{\sqrt m}$. Hence, we have
\begin{align*}
    \Vert\xi_1\Vert_2^2\leq \bfa V\bfa T\bfa B^{(\infty)}\bfa T\bfa V^\top +  \frac{C\sqrt d n^2\log\frac{n}{\delta}\log\frac{m}{\delta}}{\sqrt m}.
\end{align*}
Further, we can show that with probability at least $1-\delta$,
\begin{align*}
    \Vert\xi_2\Vert_2&\leq \sum_{k=0}^{t-1}\eta\Vert\ca Z^{(t),\top}|\Vert\Vert\ca A^{(t)}-\ca A^{(0)}\Vert_2\Big\Vert \lef(I-\eta\bfa H^{(\infty)}\rig)^t\Big\Vert\Vert\bfa V^\top\Vert_2\\
    &\leq \frac{C\eta n\sqrt{d}}{\sqrt m}\log\frac{n}{\delta}\sum_{k=0}^{t-1}\eta (1-\eta\lambda(\bfa H^{(\infty)}))^tn^2\sqrt d\leq \frac{C\eta^2 n^3 d}{\sqrt m}\log\frac{n}{\delta}.
\end{align*}
For $\xi_3$,  we can show that
\begin{align*}
    \Vert\xi_3\Vert_2 = \sum_{k=0}^{t-1}\eta\Vert \ca Z^{(k),\top}\Vert_2\Vert\ca A^{(k)}\Vert_2\Vert\bfa E^{(k)}\Vert_2\leq \frac{\text{poly}\lef(n,\lambda_{\min}^{-1}(\bfa H^{(\infty)}),d,1/\delta\rig)}{m^{\frac{1}{4}}\kappa^{\frac{1}{2}}}+\frac{Cnd\kappa }{\lambda_{\min}(\bfa H^{(\infty)})\delta}.
\end{align*}
And for $\xi_4$, we can show that with probability at least $1-\delta$,
\begin{align*}
   \Vert\xi_4\Vert_2 &= \sum_{k=0}^{t-1}\eta (\ca Z^{(k),\top} - \ca Z^{(0),\top})(\ca A^{(k)} - \ca A^{(0)})\lef(I - \eta\bfa H^{(\infty)}\rig)^t\bfa V^\top\\
    &\leq\sum_{k=0}^{t-1}\eta\Vert\ca Z^{(t),\top}- \ca Z^{(0),\top}\Vert_2\Vert\ca A^{(k)} -\ca A^{(0)}\Vert_2(1-\eta\lambda_{\min}(\bfa H^{(\infty)}))^t n^2\sqrt d\\
    &\leq \sum_{k=0}^{t-1}(1 - \eta\lambda_{\min}(\bfa H^{(\infty)}))^t\cdot \frac{C\eta^2 n\sqrt{d}}{\sqrt m}\log\frac{n}{\delta}\cdot \sqrt{\frac{1}{\delta}\sum_{i=1}^n\frac{ n\Vert\tde\xb_{\tau_i}\Vert_2^3}{\sqrt m}\log^2m}\\
    &\leq\frac{n^2}{1-\eta\lambda_{\min}(\bfa H^{(\infty)})}\frac{C\eta^2n^2\sqrt{d}\log\frac{n}{\delta}\log m}{m^{\frac{3}{4}}\sqrt{\delta}}\sqrt{\sum_{i=1}^n\Vert\tde\xb_{\tau_i}\Vert_2^3}.
\end{align*}
Then, we can show that
\begin{align*}
    \Vert\xi_5\Vert_2^2 &= \Big\Vert\Sigma^{(\infty)}\bfa T\bfa V^\top\Big\Vert_2^2  = \bfa V\bfa T\Sigma^{(\infty),\top}\Sigma^{(\infty)}\bfa T\bfa V^\top \\
    &\leq  \bfa V\bfa T\bfa C^{(\infty)}\bfa T\bfa V^\top + \lef\Vert\bfa V\bfa T\lef(\bfa C^{(\infty)}-\Sigma^{(\infty),\top}\Sigma^{(\infty)}\rig)\bfa T\bfa V^\top\rig\Vert_2\\
    &\leq \bfa V\bfa T\bfa C^{(\infty)}\bfa T\bfa V^\top+ \Big\Vert\bfa C^{(\infty)} - \Sigma^{(\infty),\top}\Sigma^{(\infty)}\Big\Vert_2\Vert\bfa T\Vert_2^2\Vert\bfa V\Vert_2^2.
\end{align*}
And for $\xi_6$, we can show that
\begin{align*}
    \Vert\xi_6\Vert_2&= \Big\Vert\sum_{k=0}^{t-1}\eta\lef(\Sigma^{(k)}- \Sigma^{(0)}\rig)\lef(I-\eta\bfa H^{(\infty)}\rig)^k\bfa V^\top\Big\Vert_2\\
    &\leq\sum_{k=0}^{t-1}\eta\Vert\Sigma^{(k)} -\Sigma^{(0)}\Vert_2\Big(1-\eta\lambda_{\min}(\bfa H^{(\infty)})\Big)^k\Vert\bfa V^\top\Vert_2\\
    &\leq \frac{\eta nd^{3/2}\log^2n\log^2m}{\sqrt{m\delta}(1-\eta\lambda(\bfa H^{(\infty)}))}\sqrt{\sum_{i=1}^n\Vert\tde\xb_{\tau_i}\Vert_2^2}.
\end{align*}
For $\xi_7$, we can show that
\begin{align*}
   \Vert\xi_7\Vert_2 &= \sum_{k=0}^{t-1}\eta\lef\Vert\Sigma^{(k)}\Big(I-\eta\bfa H^{(\infty)}\Big)^k\bfa E^{(k)}\rig\Vert_2\leq\eta\sum_{k=0}^{t-1}\Vert\Sigma^{(k)}\Vert_2\lef(1-\eta\lambda_{\min}(\bfa H^{(\infty)})\rig)^k\Vert\bfa E^{(k)}\Vert_2\\
   &\leq \frac{\text{poly}\lef(n,\lambda_{\min}(\bfa H^{(\infty)})^{-1},d\rig)}{m^{\frac{1}{2}}\kappa}+\frac{Cnd\kappa }{\lambda_{\min}(\bfa H^{(\infty)})\delta}.
\end{align*}
And collecting pieces, we show that with probability at least $1-\delta$,
\begin{align*}
    \Vert\bfa W^{(t)} - \bfa W^{(0)}\Vert_F&= \Vert\vect(\bfa W^{(t)} - \bfa W^{(0)})\Vert_2 \\
    &\leq \sqrt{\bfa V\bfa T\bfa B^{(\infty)}\bfa T\bfa V^\top} + \frac{\poly\lef(n,\lambda_{\min}^{-1}(\bfa H^{(\infty)}), d, 1/\delta\rig)}{m^{1/4}\kappa^{1/2}} + \frac{Cnd\kappa }{\lambda_{\min}(\bfa H^{(\infty)})\delta}.\\
    \Vert\tda A^{(t)} - \tda A^{(0)}\Vert_F&= \Vert\vect(\bfa A^{(t)} - \bfa A^{(0)})\Vert_2 \\
    &\leq  \sqrt{\bfa V\bfa T\bfa C^{(\infty)}\bfa T\bfa V^\top}+ \frac{\poly(n, \lambda_{\min}^{-1}(\bfa H^{(\infty)}) ,d,1/\delta)}{m^{1/4}\kappa^{1/2}}+\frac{Cnd\kappa }{\lambda_{\min}(\bfa H^{(\infty)})\delta}.
\end{align*}
\end{proof}
Hence, we let $\ell(\xb, \yb) := \Vert\xb - \yb\Vert_2$.
The following theorem provides the generalization bound for the learning problem.
\begin{lemma}
Recall the following class
    \begin{align*}
        \ca F_{R_{\ab}, R_{\wb}}:=\Big\{\wb, \ab:\forall r\in[m],\Big(\sum_{r=1}^m\Vert\wb_r - \wb_r^{(0)}\Vert_2^2\Big)^{{1}/{2}} \leq R_{\wb},\Big(\sum_{r=1}^m \Vert\ab_r -\ab_r^{(0)}\Vert_2^2\Big)^{{1}/{2}}\leq R_{\ab} \Big\}.
    \end{align*}
    Then, given $\kappa = O(\frac{\lambda_0\delta}{n})$ and $m=\kappa^{-2}\poly\lef(n, d,\delta^{-1},\lambda_{\min}^{-1}(\bfa H^{(\infty)})\rig)$ the empirical Rademacher complexity of class $\ca F_{R_{\ab},R_{\wb}}$ is bounded as follows with probability at least $1-\delta$
    \begin{align*}
        \ca R(\ca F_{R_{\wb},R_{\ab}}) &\leq\lef(\sqrt{\bfa V\bfa T\bfa C^{(\infty)}\bfa T\bfa V^\top}+ \sqrt{\bfa V\bfa T\bfa B^{(\infty)}\bfa T\bfa V^\top}\rig)\sqrt{\frac{d\log^2(n/\delta)}{n}\Big(1+\sqrt{\frac{\log(2/\delta)}{2m}}\Big)}\\
        &+\frac{C\cdot\poly(n, \lambda_{\min}^{-1}(\bfa H^{(\infty)}) ,d,1/\delta)}{m^{1/4}\kappa^{1/2}}\\
        &\leq 2\sqrt{\bfa V\bfa H^{(\infty),-1}\bfa V^\top}\sqrt{\frac{d\log^2(n/\delta)}{n}\Big(1+\sqrt{\frac{\log(2/\delta)}{2m}}\Big)}+\frac{C\cdot\poly(n, \lambda_{\min}^{-1}(\bfa H^{(\infty)}) ,d,1/\delta)}{m^{1/4}\kappa^{1/2}}.
    \end{align*}
\end{lemma}
\begin{proof}
    Due to our previous analysis, the proof goes by providing an upper bound for the Rademacher complexity. Recall the following class
    \begin{align*}
        \ca F_{R_{\ab}, R_{\wb}}:=\Big\{\wb, \ab:\forall r\in[m],\Big(\sum_{r=1}^m\Vert\wb_r - \wb_r^{(0)}\Vert_2^2\Big)^{{1}/{2}} \leq R_{\wb},\Big(\sum_{r=1}^m \Vert\ab_r -\ab_r^{(0)}\Vert_2^2\Big)^{{1}/{2}}\leq R_{\ab} \Big\}.
    \end{align*}
    And we recall the following event
    \begin{align*}
        A_{r,i}(R) := \lef\{\Big|[\xb_{\tau_i,i}^\top,\tau_i,\xb_{0,i}^\top]\wb_{r}^{(0)}\Big|\leq R\rig\}.
    \end{align*}
    Therefore, we can show that under the event $A_{r,i}$ and the parameter space given by $\ca F_{R_{\ab}, R_{\wb}}$, one can show that
    \begin{align*}
        \lef|[\xb_{\tau_i,i}^\top,\tau_i,\xb_{0,i}^\top]\wb_r^{(0)} - [\xb_{\tau_i,i}^\top,\tau_i,\xb_{0,i}^\top]\wb_r^{(t)}\rig |\leq \Vert[\xb_{\tau_i,i}^\top,\tau_i,\xb_{0,i}^\top]\Vert_2\Vert\wb_r^{(0)} - \wb_r^{(t)}\Vert_2\leq R_{\wb}\Vert[\xb_i^\top,\tau_i]\Vert_2.
    \end{align*}
    Note that by union bound and the fact that $\lef\Vert[\xb_{\tau_i,i}^\top,\tau_i,\xb_{0,i}^\top]\rig\Vert$ is Sub-Gaussian, we achieve that
    \begin{align*}
        \bb P\Big(\max_{i\in[n]}\Vert[\xb_{\tau_i,i}^\top,\tau_i,\xb_{0,i}^\top]\Vert_2\geq t\Big)\leq n\bb P\Big(\Vert[\xb_{\tau_i},\tau_i]\Vert_2\geq t\Big)\leq Cn\exp\lef(-\frac{Cnt^2}{2}\rig).
    \end{align*}
    Hence, defining $B_{\xb}^{\delta} = C\sqrt{n\log\frac{n}{\delta}}$, we can show that, with probability at least $1-\delta$, $\max_{i\in[n]}\Vert[\xb_{\tau_i,i}^\top,\tau_i,\xb_{0,i}^\top]\Vert_2\leq B_{X}^{\delta}$.
    Hence, we have under of $A_{r,i}^c(R_{\wb}B_{\xb}^{\delta})$, with probability at least $1-\delta$ over $\xb$, we have
    \begin{align*}
        \wb_r^{(t),\top}\tde\xb_i\geq \wb_r^{(0),\top}\tde\xb_i - |\wb_r^{(0),\top}\tde\xb_i - \wb_{r}^{(t),\top}\tde\xb_i|> 0,
    \end{align*}
    which implies that
    \begin{align*}
        \mbbm 1_{A_{r,i}^c(R_{\wb}B_{\xb}^{\delta})}\sigma([\xb_{\tau_i,i}^\top,\tau_i,\xb_{0,i}^\top]\wb_r^{(t)}) = \mbbm 1_{A_{r,i}^c (R_{\wb}B_{\xb}^{\delta})} \mbbm 1_{[\xb_{\tau_i,i}^\top,\tau_i,\xb_{0,i}^\top]\wb_r^{(0)}>0}[\xb_{\tau_i,i}^\top,\tau_i,\xb_{0,i}^\top]\wb_r^{(t)}.
    \end{align*}
    Moreover, under $A_{r,i}(R_{\wb}B_{\xb}^{\delta})$, one has with probability at least
    \begin{align*}
    \sigma&\lef([\xb_{\tau_i,i}^\top,\tau_i,\xb_{0,i}^\top]\wb_r^{(t)}\rig)\mbbm 1_{A_{r,i}(R_{\wb}B_{\xb}^{\delta})}\leq \Big|[\xb_{\tau_i,i}^\top,\tau_i,\xb_{0,i}^\top]\wb_r^{(t)}\mbbm 1_{A_{r,i}(R_{\wb}B_{\xb}^{\delta})}\Big|\\
    &\leq [\xb_{\tau_i,i}^\top,\tau_i,\xb_{0,i}^\top]\wb_r^{(0)}\mbbm 1_{A_{r,i}(R_{\wb}B_{\xb}^{\delta})} + \Big|\Big([\xb_{\tau_i,i}^\top,\tau_i,\xb_{0,i}^\top]\wb_r^{(t)}-[\xb_{\tau_i,i}^\top,\tau_i,\xb_{0,i}^\top]\wb_r^{(0)}\Big)\mbbm 1_{A_{r,i}(R_{\wb}B_X^{\delta})}\Big|\\
    &\leq \Big(R_{\wb}B_{\xb}^{\delta} + \Vert\wb_{r}^{(t)}- \wb_r^{(0)}\Vert_2 B_{\xb}^{\delta}\Big)\mbbm 1_{A_{r,i}}\\
    &\leq R_{\wb}B_{\xb}^{\delta} + C\Big(1 - \frac{\eta\lambda_{\min}(\bfa H^{(\infty)})}{2}\Big)^{t-1}.
    \end{align*}
    Then we notice that
    \begin{align*}
        &\sum_{i=1}^n\sum_{j=1}^d\epsilon_{ij}f_{\bfa\theta^{(t)}}(\xb_{\tau_i},\tau_i) = \sum_{i=1}^n\sum_{j=1}^d\sum_{r=1}^m\epsilon_{ij}\frac{1}{\sqrt m}\ab_{rj}^{(t)}\sigma([\xb^\top_{\tau_i},\tau_i,\xb_{0,i}^\top]\wb_r^{(t)}) \\
        &=\sum_{i=1}^n\sum_{j=1}^d\sum_{r=1}^m\frac{1}{\sqrt m}\epsilon_{ij}\ab_{rj}^{(t)}\sigma([\xb^\top_{\tau_i},\tau]\wb_r^{(t)})\lef(\mbbm 1_{A_{r,i}^c(R_{\wb}B_X^{\delta})} + \mbbm 1_{A_{r,i}(R_{\wb}B_X^{\delta})}\rig)\\
        &=\ub{\sum_{i=1}^n\sum_{j=1}^d\sum_{r=1}^m\frac{1}{\sqrt m}\epsilon_{ij}\ab_{rj}\sigma\lef([\xb_{\tau_i,i}^\top,\tau_i,\xb_{0,i}^\top]\wb_r^{(t)}\rig)\mbbm 1_{A_{r,i}^c\lef(R_{\wb}B_X^{\delta}\rig)}}_{T_1(\bfa\theta)} \\
        &+ \ub{\sum_{i=1}^n\sum_{j=1}^d\sum_{r=1}^m\frac{1}{\sqrt m}\epsilon_{ij}\ab_{rj}\sigma\lef([\xb_{\tau_i,i}^\top,\tau_i,\xb_{0,i}^\top]\wb_r^{(t)}\rig)\mbbm 1_{A_{r,i}\lef(R_{\wb}B_X^{\delta}\rig)}}_{T_2(\bfa\theta)}.
    \end{align*}

    And for the first term we can show that
    \begin{align*}
        \bb E\Big[\sup_{\bfa\theta\in\ca F_{R_{\ab}, R_{\wb}}}T_1(\bfa\theta)\Big|\tde\xb\Big] &= \bb E\Big[\sup_{\bfa\theta\in\ca F_{R_{\ab}, R_{\wb}}}\sum_{i=1}^n\sum_{j=1}^d\sum_{r=1}^m\frac{1}{\sqrt m}\epsilon_{ij}\ab_{rj}^{(t)}\sigma\lef([\xb_{\tau_i,i}^\top,\tau_i,\xb_{0,i}^\top]\wb_r^{(t)}\rig)\mbbm 1_{A_{r,i}^c\lef(R_{\wb}B_{X}^{\delta}\rig)}\Big|\tde\xb,\wb^{(0)}\Big]\\
        &\leq \bb E\Big[\sup_{\bfa\theta\in\ca F_{R_{\ab}, R_{\wb}}}\sum_{i=1}^n\sum_{j=1}^d\sum_{r=1}^m\frac{1}{\sqrt m}\epsilon_{ij}\ab_{rj}^{(t)}[\xb_{\tau_i,i}^\top,\tau_i,\xb_{0,i}^\top]\wb_r^{(t)}\mbbm 1_{[\xb_{\tau_i,i}^\top,\tau_i,\xb_{0,i}^\top]\wb_r^{(0)}>0}\Big|\tde\xb,\wb^{(0)}\Big].
    \end{align*}

    Then we can show that by Cauchy-Schwartz inequality, with probability at least $1-\delta$,
    \begin{align*}
        \bb E\Big[&\sup_{\bfa\theta\in\ca F_{R_{\ab}, R_{\wb}}} \sum_{i=1}^n\sum_{j=1}^d\sum_{r=1}^m\frac{1}{\sqrt m}\epsilon_{ij}\ab_{rj}^{(t)}[\xb_{\tau_i,i}^\top,\tau_i,\xb_{0,i}^\top]\wb_r^{(t)}\mbbm 1_{[\xb_{\tau_i,i}^\top,\tau_i,\xb_{0,i}^\top]\wb_r^{(0)}>0}\Big|\tde\xb\Big]\\
        &=\bb E\Big[\sup_{\bfa\theta\in\ca F_{R_{\ab},R_{\wb}}}\sum_{i=1}^n\sum_{j=1}^d\ca E_{ij}(\ca Z^{(0)}\ca W^{(t)}\ca A^{(t)})_{ij}\Big|\tde\xb\Big]\\
        &\leq\bb E\Big[\sup_{\bfa\theta\in\ca F_{R_{\ab}, R_{\wb}}}\tr\lef(\ca E^\top\ca Z^{(0)}\ca W^{(0)}(\ca A^{(t)}- \ca A^{(0)})\rig)\Big|\tde\xb\Big]\\
        &+\bb E\Big[\sup_{\bfa\theta\in\ca F_{R_{\ab},R_{\wb}}}\tr\lef(\ca E^\top\ca Z^{(0)}(\ca W^{(t)}-\ca W^{(0)})\ca A^{(0)}\rig)\Big|\tde\xb\Big]\\
        &\leq \bb E\Big[\sup_{\bfa\theta\in\ca F_{R_{\ab}, R_{\wb}}}\Vert\ca E^\top\ca Z^{(0)}\Vert_F\lef\Vert\ca W^{(0)}(\ca A^{(t)} - \ca A^{(0)})\rig\Vert_F\Big|\tde\xb\Big]\\
        &+\bb E\Big[\sup_{\bfa\theta\in\ca F_{R_{\ab}, R_{\wb}}}\Vert\ca E^\top\ca Z^{(0)}\Vert_F\lef\Vert(\ca W^{(t)}-\ca W^{(0)}) \ca A^{(0)}\rig\Vert_F\Big|\tde\xb\Big]\\
        &\leq \bb E\Big[\Vert\ca E^\top\ca Z^{(0)}\Vert_F\Big|\tde\xb,\wb^{(0)}\Big]\lef(\Vert\ca W^{(0)}\Vert_2R_{\ab}+\Vert\ca A^{(0)}\Vert_2R_{\wb}\rig) \\
        &\leq \Vert\ca Z^{(0)}\Vert_F \lef(R_{\wb}+R_{\ab}\rig)\log\frac{n}{\delta}.
    \end{align*}
    Notice that for matrix $A,B,C,D$, we use Cauchy-Schwartz inequality to obtain that
    \begin{align*}
        \tr(ABCD) = \la(AB)^\top,CD \ra_{F}\leq \Vert AB\Vert_F\Vert CD \Vert_F\leq \Vert A\Vert_F\Vert B\Vert_2\Vert C\Vert_F\Vert D\Vert_2.
    \end{align*}
    We further notice that with probability at least $1-\delta$, by Hoeffding's inequality,
    \begin{align*}
        \lef\Vert\ca Z^{(0)}\rig\Vert_F^2&= \sum_{i=1}^n\sum_{j=1}^m\frac{1}{m}\Vert\tde\xb_{\tau_i}\Vert_2^2\mbbm 1_{\tde\xb_{\tau_i}^\top\wb_j^{(0)} > 0}\leq \frac{d}{m}\log\frac{n}{\delta}\sum_{i=1}^n\sum_{j=1}^m\mbbm 1_{\tde\xb_{\tau_i}^\top\wb_j^{(0)} >0}\\
        &\leq dn\log\frac{n}{\delta}\Big(1+\sqrt{\frac{\log(2/\delta)}{2m}}\Big).
    \end{align*}
    Hence, collecting pieces, we can show that
    \begin{align*}
        \bb E\Big[\sup_{\bfa\theta\in\ca F_{R_{\ab},R_{\wb}}}T_1(\bfa\theta)\Big]\leq( R_{\wb}+R_{\ab})\sqrt{dn\log^2(n/\delta)\Big(1+\sqrt{\frac{\log(2/\delta)}{2m}}\Big)}.
    \end{align*}
    And, for the second term $T_2$, using the notation $\tde R_{\wb} := \max_{r}\Vert\wb_r^{(t)}-\wb_r^{(0)}\Vert_2\leq $ and $\tde R_{\ab}:=\max_{r}\Vert\ab_r^{(t)}-\ab_r^{(0)}\Vert_2$ and $\tde R_{\wb}\vee\tde R_{\ab}\leq C\frac{\eta n\sqrt d}{\sqrt m}\log\frac{n}{\delta}$ we can estimate it by
    \begin{align*}
        &\bb E\Big[\sup_{\bfa\theta\in\ca F_{R_{\ab},R_{\wb}}}T_2(\bfa\theta)\Big|\tde\xb,\wb^{(0)}\Big] \\
        &\leq\bb E\bigg[\sup_{\bfa\theta\in\ca F_{R_{\ab}, R_{\wb}}}\sum_{i=1}^n\sum_{j=1}^d\sum_{r=1}^m\frac{1}{\sqrt m}\epsilon_{ij}\ab_{rj}^{(t)}\Big(\tde R_{\wb}B_{\xb}^{\delta} + C\Big(1-\frac{\eta\lambda_{\min}(\bfa H^{(\infty)})}{2}\Big)^{t-1}\Big)\mbbm 1_{A_{r,i}(R_{\wb}B_{\xb}^{\delta})}\bigg|\tde\xb,\wb^{(0)}\bigg]\\
        &\leq \sum_{i=1}^n\sum_{j=1}^d\sum_{r=1}^m\frac{1}{\sqrt m}\Vert\tda A^{(t)}\Vert_{2,\infty}\Big(\tde R_{\wb}B_{\xb}^{\delta}+ C\Big( 1-\frac{\eta\lambda_{\min}(\bfa H^{(\infty)})}{2}\Big)^{t-1}\Big)\mbbm 1_{A_{r,i}(R_{\wb}B_{\xb}^{\delta})}\\
        &\leq \frac{Cd\Vert\tda A^{(t)}\Vert_{\infty}}{\sqrt m}\tde R_{\wb}B_{\xb}^{\delta}\sum_{i=1}^n\sum_{r=1}^m\mbbm 1_{A_{r,i}(R_{\wb}B_{\xb}^{\delta})}.
    \end{align*}
    Recall, by union bound we can show that with probability at least $1-\delta$, we have $$\Vert\tda A^{(t)}\Vert_{2,\infty}\leq C\sqrt{d}\log\frac{n}{\delta} + \tde R_{\ab}.$$
    And we note that by Hoeffding's inequality and \eqref{aineq}, one can show that with probability at least $1-\delta$, we have
    \begin{align*}
        \sum_{i=1}^n\sum_{r=1}^m\mbbm 1_{A_{r,i}(R_{\wb}B_{\xb}^{\delta})}\leq\sum_{i=1}^n\frac{Cm\tde R_{\wb}\log(n/\delta)}{\sqrt{d}\Vert\tde\xb_{\tau_i}\Vert_2\kappa}.
    \end{align*}
    And by anti-concentration property of $\Vert\tde\xb_{\tau_i}\Vert_2$ we can show with the Bernstein's probability at least $1-\delta$,
    \begin{align*}
        \sum_{i=1}^n\frac{1}{\Vert\tde\xb_{\tau_i}\Vert_2}\leq\frac{Cn}{\sqrt d}+ C\sqrt{\frac{n}{d}\log\frac{1}{\delta}}.
    \end{align*}
    And one obtains with probability at least $1-\delta$ when $\kappa = O(\frac{d\lambda_0\delta}{n})$,
    \begin{align*}
        \bb E\Big[\sup_{\bfa\theta\in\ca F_{R_{\ab},R_{\wb}}}T_2(\bfa\theta)\Big|\tde\xb\Big] &= \frac{Cm^{1/2}d\lef(d\log(n/\delta)+\tde R_{\ab}\rig)}{\kappa} \tde R_{\wb}^2\log^2({n}/{\delta})\Big(\frac{n}{\sqrt d}+\sqrt{\frac{n}{d}\log\frac{1}{\delta}}\Big)\\
        &\leq \frac{1}{\sqrt m\kappa}\poly\lef(n, d,\delta^{-1},\lambda_{\min}^{-1}(\bfa H^{(\infty)})\rig).
    \end{align*}
    Hence, by collecting pieces including lemma \ref{wdiffadiff}, we can show that
    \begin{align*}
        \ca R_{\xb}(\ca F_{R_{\ab}, R_{\wb}})&\leq\frac{1}{n}\bb E\Big[\sup_{\bfa\theta\in\ca F_{R_{\ab},R_{\wb}}}\Big(T_1(\bfa\theta)+T_2(\bfa\theta)\Big)\Big|\tde\xb\Big]\\
        &\leq \frac{\poly(n, d,\delta^{-1},\lambda_{\min}^{-1}(\bfa H^{(\infty)}))}{\sqrt m}+R_{\wb}R_{\ab}\sqrt{\frac{d\log^2(n/\delta)}{n}\Big(1+\sqrt{\frac{\log(2/\delta)}{2m}}\Big)}\\
        &\leq \lef(\sqrt{\bfa V\bfa T\bfa C^{(\infty)}\bfa T\bfa V^\top}+ \sqrt{\bfa V\bfa T\bfa B^{(\infty)}\bfa T\bfa V^\top}\rig)\sqrt{\frac{d\log^2(n/\delta)}{n}\Big(1+\sqrt{\frac{\log(2/\delta)}{2m}}\Big)}\\
        &+\frac{C\cdot\poly(n, \lambda_{\min}^{-1}(\bfa H^{(\infty)}) ,d,1/\delta)}{m^{1/4}\kappa^{1/2}}.
    \end{align*}
    We denote $\{\lambda_i\}_{i\in[nd]}$ and $\{\bfa v_i\}_{i\in[nd]}$ as the eigenvalues and eigenvectors of $\bfa H^{(\infty)}$. Then we can show that
    \begin{align*}
        \bfa T = \sum_{i=1}^{nd}\eta\sum_{k=1}^{t-1}(I-\eta\bfa H^{(\infty)}) = \sum_{i=1}^{nd}\frac{1-(1-\eta\lambda_i)^{t}}{\lambda_i}\bfa v_i\bfa v_i^\top.
    \end{align*}
    Hence we can show that
    \begin{align*}
        \bfa T\bfa H^{(\infty)}\bfa T =\sum_{i=1}^{nd}\frac{(1-(1-\eta\lambda_i)^t)^2}{\lambda_i}\bfa v_i\bfa v_i^\top\prec \sum_{i=1}^{nd}\frac{1}{\lambda_i}\bfa v_i\bfa v_i^\top.
    \end{align*}
    And we further show that
    \begin{align*}
        \sqrt{\bfa V\bfa T\bfa C^{(\infty)}\bfa T\bfa V^\top}&+\sqrt{\bfa V\bfa T\bfa B^{(\infty)}\bfa T\bfa V^\top}\leq 2\sqrt{\bfa V\bfa T(\bfa C^{(\infty)} + \bfa B^{(\infty)})\bfa T\bfa V^\top}\\
        &=2\sqrt{\bfa V\bfa H^{(\infty),-1}\bfa V^\top}
    \end{align*}
    and finish the proof.
\end{proof}
Then we are ready to formally prove theorem \ref{thmgenbound}.
    In total of three terms are to be analyzed, our gradient-descent optimized solution $\bfa\theta^{(t)}$ will with probability at least $1-\delta$ satisfies
    \begin{align*}
\frac{1}{n}\sum_{i=1}^n\Vert\bfa f_{\bfa\theta}(\xb_{\tau_i},\tau_i,\xb_{0,i}) - \vb_{\tau_i}(\xb_{\tau_i}|\xb_{1,i})\Vert_2
        &\leq\frac{1}{\sqrt n}\Vert\bfa U^{(t)}-\bfa V\Vert_2 \leq \frac{1}{\sqrt n}\Big(1-\frac{1}{2}\eta\lambda_{\min}(\bfa H^{(\infty)})\Big)^\frac{t}{2}\Vert\bfa U^{(0)} - \bfa V\Vert_2\\
        &\leq C\sqrt d\Big(1-\frac{1}{2}\eta\lambda_{\min}(\bfa H^{(\infty)})\Big)^{\frac{t}{2}}.
    \end{align*}
    And for the Rademacher complexity bound given by lemma \ref{lmrademacher}, we can show that with probability at least $1-\delta$,
    \begin{align*}
        \ca R_{\xb}(\ca F_{R_{\ab}, R_{\wb}})&\leq 2\sqrt{\bfa V\bfa H^{(\infty),-1}\bfa V^\top}\sqrt{\frac{d\log^2(n/\delta)}{n}\Big(1+\sqrt{\frac{\log(2/\delta)}{2m}}\Big)}\\
        &+\frac{C\cdot\poly(n, \lambda_{\min}^{-1}(\bfa H^{(\infty)}) ,d,1/\delta)}{m^{1/4}\kappa^{1/2}}.
    \end{align*}
    We first study the norm property of $\tda W^{(t_1)}$, note that by union bound with probability at least $1-\delta$,
    \begin{align*}
        \lef\Vert\tda W^{(t_1)}\rig\Vert_2 \leq \sqrt m\max_{r\in[m]}\lef\Vert\wb_{r}^{(t_1)}\mbbm 1_{\wb_r^{(0)\top}\xb>0}\rig\Vert_2\leq \sqrt m \kappa\sqrt{d\log\frac{m}{\delta}}.
    \end{align*}
    Then we verify the uniform sub-Gaussian condition. Define $\tda W^{(t_1)} =\begin{bmatrix}
        \wb_1^{(t_1)}\mbbm 1_{\wb_1^\top\xb>0}&\wb_2\mbbm 1_{\wb_2^\top\xb>0}&\ldots&\wb_n\mbbm 1_{\wb_n^\top\xb > 0}
    \end{bmatrix}$. Then, with probability at least $1-\delta$ over the training samples,
    \begin{align}\label{lipschitz}
        \sup_{\bfa f\in\ca F_{R_{\wb},R_{\ab}}}&\Vert\bfa f(\xb_{\tau_i}^\top, \tau_i,\xb_{0,i}) - \vb(\xb_{\tau_i}|\xb_{0,i})\Vert_2=\sup_{\bfa f\in\ca F_{R_{\wb},R_{\ab}}}\Big\Vert\frac{1}{\sqrt m}\sum_{r=1}^m\ab_{r}\wb_{r}^\top\tde\xb\mbbm 1_{\wb_r^{\top}\tde\xb > 0} - \vb(\xb_{\tau_i}|\xb_{0,i})\Big\Vert_2\nnb\\
        &\leq \sup_{\bfa f\in\ca F_{R_{\wb},R_{\ab}}}\Big\Vert\frac{1}{\sqrt m}\sum_{r=1}^m\ab_r\wb_r^\top\tde\xb\mbbm 1_{\wb_r^\top\tde\xb>0}\Big\Vert+ \Vert\vb(\xb_{\tau_i}|\xb_{0,i})\Vert_2\nnb\\
       &\leq \Big(\frac{1}{\sqrt m}\Big\Vert\tda A^{(t)}\tda W^{(t),\top}\Big\Vert_2 +B_0\Big)\Vert\tde\xb\Vert_2\nnb\\
        &\leq \Big(\frac{1}{\sqrt m}\Vert\tda A^{(t)}- \tda A^{(0)}\Vert_F\Vert\tda W^{(0)}\Vert_2 + \frac{1}{\sqrt m}\Vert\tda W^{(t)}-\tda W^{(0)}\Vert_F\Vert\tda A^{(0)}\Vert_2 +\frac{1}{\sqrt m}\sum_{i=1}^m\Vert\ab_i\wb_i^\top\Vert_2+B_0\Big)\Vert\tde\xb\Vert_2\nnb\\
        &\leq\Big(\kappa R_{\wb}+R_{\ab}+2B_0\Big)\sqrt{\log\frac{m}{\delta}}\Vert\tde\xb\Vert_2.
    \end{align}
    We further utilize the fact that $\Vert\tde\xb\Vert$ is $B_0$ sub-Guassian to achieve that
    \begin{align*}
        \log\bb E\lef[\sup_{\bfa f\in\ca F_{R_{\wb},R_{\ab}}}\cosh(\lambda\lef\Vert\bfa f(\xb_{\tau_i}^\top, \tau_i,\xb_{0,i}) - \vb(\xb_{\tau_i}|\xb_{0,i}))\rig\Vert_2)\rig]\leq \lef(CB_0\Big(\kappa R_{\wb} +\kappa R_{\ab}+ 2B_0\Big)\sqrt{\log\frac{m}{\delta}}\lambda\rig)^2.
    \end{align*}
    And finally we show that
    \begin{align*}
        \bb E&[\Vert\bfa f_{\bfa\theta}(\xb_{\tau},\tau,\xb_{0}) - \vb_{\tau}(\xb_{\tau}|\xb_{0})\Vert_2]\leq \frac{1}{n}\sum_{i=1}^n\Vert\bfa f_{\bfa\theta}(\xb_{\tau_i},\tau_i,\xb_{0,i}) - \vb_{\tau_i}(\xb_{\tau_i}|\xb_{0,i})\Vert_2 + \ca R_{\xb}(\ca F_{R_{\wb},R_{\ab}}) \\
        &+B_0\kappa\lef(\kappa R_{\wb} + R_{\ab} + 2B_0\rig)\sqrt{\log\frac{m}{\delta}}\\
        &=\frac{1}{\sqrt m\kappa}\poly\lef(n, d,\delta^{-1},\lambda_{\min}^{-1}(\bfa H^{(\infty)})\rig) + C\sqrt{d}\Big(1-\frac{1}{2}\eta\lambda_{\min}(\bfa H^{(\infty)})\Big)^{\frac{t}{2}}\\
        &+ CB_0\lef(\kappa R_{\wb} + R_{\ab} + 2B_0\rig)\sqrt{\frac{1}{n}\log^2\frac{m}{\delta}} + 2\sqrt{{\bfa V\bfa H^{(\infty),-1}\bfa V}}\sqrt{\frac{d\log^2(n/\delta)}{n}\Big(1+\sqrt{\frac{\log(2/\delta)}{2m}}\Big)}.
    \end{align*}
    Hence, picking $t\geq \frac{C}{\eta\lambda_{\min}(\bfa H^{(\infty)})}\log\frac{n}{\delta}$ we can show that
    \begin{align*}
        \bb E\Big[\Vert\bfa f_{\bfa\theta}(\xb_{\tau},\tau,\xb_0)-\vb_{\tau}(\xb_{\tau}|\xb_{0})\Vert_2\Big]&\leq C\sqrt{\bfa V\bfa H^{(\infty),-1}\bfa V}\sqrt{\frac{(d+B_0^2)\log^2(n/\delta)}{n}} \\
        &+ CB_0\sqrt{\frac{d\log(n/\lambda_{\min}(\bfa H^{(\infty)})\delta)}{n}}.
    \end{align*}

\section{Proof of Theorem \ref{thmsampbound}}\label{app:proof-thmsampbound}
This section provides the complete proof of theorem \ref{thmsampbound}. For completeness we first prove all the lemmas in Section~\ref{sect:6} and then move on to the final proof of theorem \ref{thmsampbound}.
\begin{lemma}\label{apx:l1tolinfty}
    Assume that a compact domain $K\subset\bb R^d$ is compact with nonempty interior. Assume that $X\in K$ has  density bounded below with $p(X) \geq p_{\min} >0$. Assume that $f:\bb R^d\to\bb R^d$ is Lipschitz in Euclidean norm with $|f(x)-f(y)|\leq L\Vert x -y\Vert_2$. Assume that there exists $\beta>0$ such that for every $x\in K$ and every $r\in(0,diam(K))]$ we have $Vol(B(x,r)\cap K)\geq\beta r^d$. Then we have
    \begin{align*}
        \sup_{X\in K} |f(X)|\leq\lef(\frac{2^{d+1}L^d}{\beta p_{\min}}\bb E[|f(X)|]\rig)^{\frac{1}{d+1}}.
    \end{align*}
\end{lemma}
\begin{proof}
Let
\[
M :=  \sup_{x\in K} |f(x)|.
\]
Since $K$ is compact and $f$ is Lipschitz (hence continuous), the supremum is attained: there exists $x_0\in K$ such that
\[
|f(x_0)| = M.
\]

By the $L$-Lipschitz property, for any $x\in K$,
\[
|f(x)|
\ge |f(x_0)| - |f(x)-f(x_0)|
\ge M - L\|x-x_0\|_2.
\]
Define
\[
r := \frac{M}{2L}.
\]
Then for all $x \in K \cap B(x_0,r)$,
\[
|f(x)| \ge M - Lr = \frac{M}{2}.
\]

Using the lower bound on the density $p(x)\ge p_{\min}$ and restricting the integral to this ball, we obtain
\[
\mathbb E|f(X)|
= \int_K |f(x)| p(x)\,dx
\ge \int_{K\cap B(x_0,r)} |f(x)| p(x)\,dx
\ge \frac{M}{2}\, p_{\min}\, \operatorname{Vol}(K\cap B(x_0,r)).
\]

By the geometric condition on $K$, there exists $\beta>0$ such that
\[
\operatorname{Vol}(K\cap B(x_0,r)) \ge \beta r^d.
\]
Therefore,
\[
\mathbb E|f(X)|
\ge \frac{M}{2}\, p_{\min}\, \beta r^d
= \frac{M}{2}\, p_{\min}\, \beta \left(\frac{M}{2L}\right)^d
= \frac{\beta p_{\min}}{2^{d+1} L^d} M^{d+1}.
\]

Rearranging yields
\[
M^{d+1}
\le
\frac{2^{d+1} L^d}{\beta p_{\min}}\, \mathbb E|f(X)|.
\]
Taking $(d+1)$-th roots gives
\[
\|f\|_{\infty}
\le
\left(
\frac{2^{d+1} L^d}{\beta p_{\min}}\, \mathbb E|f(X)|
\right)^{\frac{1}{d+1}},
\]
which completes the proof.
\end{proof}
\begin{lemma}\label{apx:uniformbound}
    Under the same condition as in theorem \ref{thmgenbound} we define compact set $K:=\{(\xb_{0},\tau,\xb_{\tau}):\Vert\xb_0\Vert_2\leq \zeta, \tau \in[0,1]\}$, we have
    \begin{align*}
        \sup_{\tde\xb\in K}\Vert\bfa f_{\bfa\theta}(\tde\xb) - \vb(\xb_{\tau}|\xb_0)\Vert_2\leq C\Big(\frac{(4B_0)^d}{\exp(-\zeta^2/2d)}\sqrt{\frac{d}{n}\log(n/\delta)}\Big)^{\frac{1}{d+1}}.
    \end{align*}
    Specifically, if we let $\zeta = C\sqrt{d\log n}, \tau \in[0,1]\}$ for some constant $C$. Assume that $d\vee B_0= o(\log n)$, we can show that with probability at least $1-\delta$,
    \begin{align*}
        \sup_{\tde\xb\in K}\Vert\bfa f_{\bfa\theta}(\tde\xb) - \vb(\xb_{\tau}|\xb_0)\Vert_2=o\lef(\log\lef(\frac{1}{\delta}\rig)\rig).
    \end{align*}
\end{lemma}
\begin{proof}
    We first consider the following region for $\xb_{0}$, $\tau$, $\xb_{\tau}$:
    \begin{align*}
        K=\{\xb_{0},\tau,\xb_{\tau}:\Vert\xb_0\Vert_2\leq \zeta, \tau \in[0,1]\}.
    \end{align*}
    Then we can show that the density of the joint distribution is given by
    \begin{align*}
        \bb P(\xb_{0},\tau)\geq \frac{1}{(2\pi)^{d/2}}\exp\lef(-\frac{\zeta^2}{2d}\rig)  = p_{\min}
    \end{align*}
    where we note that $\xb_{\tau}$ can be obtained deterministically through solving ODE with $\xb_0$ and $\tau$.
    And we also consider the Lipschitzness of the function $\bfa f_{\bfa\theta}$, note that similar to the analysis given by \eqref{lipschitz} we can show that for all $\tde\xb_1 = (\xb_{\tau_1},\tau,\xb_{0}^{(1)})$ and $\tde\xb_2 =(\xb_{\tau_2},\tau,\xb_{0}^{(2)})$ we have
    \begin{align*}
        \Vert\vb_{\tau_2}(\xb_{\tau_2}|\xb_{0}^{(1)}) -\vb_{\tau_1}(\xb_{\tau_1}|\xb_{0}^{(2)}) \Vert_2\leq B_0\lef\Vert[\xb_{\tau_1}-\xb_{\tau_2},\tau_1-\tau_2,\xb_{0}^{(1)}-\xb_{0}^{(2)}]\rig\Vert_2.
    \end{align*}
    And for the Lipschitzness of neural networks,  defining  $$\wha W = \begin{bmatrix}
        \wb_1(\mbbm 1_{\wb_1^\top(\tde\xb_1-\tde\xb_2)>0} -\mbbm 1_{\wb_1^\top(\tde\xb_1-\tde\xb_2)<0}),\ldots, \wb_r(\mbbm 1_{\wb_r^\top(\tde\xb_1-\tde\xb_2)>0} -\mbbm 1_{\wb_r^\top(\tde\xb_1-\tde\xb_2)<0})
    \end{bmatrix},$$ then we can use similar arguments as in \eqref{lipschitz} to show that for $\tde\xb_1,\tde\xb_2$ we can show that
    \begin{align}\label{lipschitzf}
        \Vert\bfa f_{\bfa\theta}(\tde\xb_1) - \bfa f_{\bfa\theta}(\tde\xb_2)\Vert &=\Big\Vert\frac{1}{\sqrt m}\sum_{r=1}^m\ab_r\lef(\sigma\lef(\wb_r^\top\tde\xb_1\rig)-\sigma(\wb_r^\top\tde\xb_2)\rig)\Big\Vert_2\nnb\\
        &\leq \frac{1}{\sqrt{m}}\Big\Vert\sum_{r=1}^m\ab_r\lef|\wb_r^\top(\tde\xb_1-\tde\xb_2)\rig|\Big\Vert_2\nnb\\
        &\leq\frac{1}{\sqrt m}\Big\Vert\sum_{r=1}^m\tda A\wha W(\tde\xb_1-\tde\xb_2)\Big\Vert_2\nnb\\
        &\leq C\sqrt{\bfa V\bfa H^{(\infty),-1}\bfa V^\top}\Vert\tde\xb_{1}-\tde\xb_2\Vert_2\sqrt{\log(m/\delta)}\nnb\\
        &\leq CB_0\Vert\tde\xb_1-\tde\xb_2\Vert_2.
    \end{align}
    Hence, the Lipschitz constant can be estimated by
    \begin{align*}
        \lef\Vert\bfa f_{\bfa\theta}(\tde\xb_1) -\vb_{\tau_1}(\xb_{\tau_1}|\xb_{0})- \bfa f_{\bfa\theta}(\tde\xb_2)+\vb_{\tau_2}(\xb_{\tau_2}|\xb_0)\rig\Vert_2\leq CB_0\Vert\tde\xb_1-\tde\xb_2\Vert_2.
    \end{align*}
    We also note that by the fact that $\xb_0$ is a standard Isotropic Gaussian,
    \begin{align*}
        \bb P(K^c)\leq C\exp\lef(-\frac{\zeta^2}{2d}\rig).
    \end{align*}
    Hence we can show by the property of conditional expectation
    \begin{align*}
        \bb E\lef[\Vert\bfa f_{\bfa\theta}(\xb_{\tau},\tau,\xb_0) - \vb_{\tau}(\xb_{\tau}|\xb_0)\Vert_2\rig]\leq \frac{1}{\bb P(K)}\bb E\lef[\Vert\bfa f_{\bfa\theta}(\xb_{\tau},\tau,\xb_0) - \vb_{\tau}(\xb_{\tau}|\xb_0)\Vert_2|K\rig].
    \end{align*}
    Therefore, using lemma \ref{apx:l1tolinfty} we can show that
    \begin{align*}
        \sup_{(\xb_0,\tau)\in K}\Vert\bfa f_{\bfa\theta}(\xb_{\tau},\tau,\xb_0)-\vb_{\tau}(\xb_{\tau}|\xb_0)\Vert_2&\leq \Big(\frac{C2^{d}B_0^d}{p_{\min}}\bb E\lef[\Vert\bfa f_{\bfa\theta}(\xb_{\tau},\tau,\xb_0)-\vb_{\tau}(\xb_{\tau}|\xb_0)\Vert_2|\xb\in K\rig]\Big)^{\frac{1}{d+1}}\\
        &\leq\Big(\frac{C(2B_0)^d}{p_{\min}}\bb E\lef[\Vert\bfa f_{\bfa\theta} - \vb_{\tau}(\xb)\Vert_2\rig]\Big)^{\frac{1}{d+1}}\\
        &\leq \Big(\frac{C(2B_0)^d}{\exp(-\zeta^2/2d)}\sqrt{\frac{d}{n}\log(n/\delta)}\Big)^{\frac{1}{d+1}}.
    \end{align*}
    Hence, given that $\zeta= C\sqrt{d\log n}$ for some constant $C$ we have $p_{\min}=o(\frac{1}{\sqrt n})$. Assuming that $B_0\vee d=o(\log n)$ one can show that
    \begin{align*}
        \sup_{\tde\xb\in K}\Vert\bfa f_{\bfa\theta}(\tde\xb) - \vb(\xb_{\tau}|\xb_{0})\Vert_2= o\lef(\log(1/\delta)\rig).
    \end{align*}
    \end{proof}
Now we are ready to prove theorem \ref{thmsampbound}.
    Note that for the approximate ODE, we have
    \begin{align*}
        \frac{d\wh\xb_{\tau}}{d\tau} = \bfa f_{\bfa\theta}(\wh\xb_{\tau},\tau,\xb_0).
    \end{align*}
    And it is noted that
    \begin{align*}
        \frac{d(\xb_{\tau} -\wh\xb_{\tau})}{d\tau}&= \vb_{\tau}(\xb_{\tau}|\xb_0) - \bfa f_{\bfa\theta}(\wh\xb_{\tau},\tau,\xb_0) \\
        &= \vb_{\tau}(\xb_{\tau}|\xb_0) - \vb_{\tau}(\wh\xb_{\tau}|\xb_0 ) + \vb_{\tau}(\wh\xb_{\tau}|\xb_0 )-\bfa f_{\bfa\theta}(\wh\xb_{\tau},\tau,\xb_0).
    \end{align*}
    And we note that by the Lipschitz property of $\xb_{\tau}$,
    \begin{align*}
        \Vert\vb_{\tau}(\xb_{\tau}|\xb_0) - \vb_{\tau}(\wh\xb_{\tau}|\xb_0)\Vert\leq B_0\Vert\xb_{\tau}-\wh\xb_{\tau}\Vert_2.
    \end{align*}
    And for the second term, using lemma \ref{apx:uniformbound} we can show that with probability at least $1-\delta$,
    \begin{align*}
    \sup_{\xb_0,\tau,\wh\xb_{\tau}}\Vert\vb_{\tau}(\wh\xb_{\tau}|\xb_0)-\bfa f_{\bfa\theta}(\wh\xb_{\tau},\tau,\xb_0)\Vert_2= o(\log(1/\delta)).
    \end{align*}
    Hence, we can use the derivative formula and Cauchy-Schwartz inequality to obtain that
    \begin{align*}
        \Big|\frac{d\Vert\xb_{\tau}-\wh\xb_{\tau}\Vert_2}{d\tau}\Big| &=\Big| \frac{(\xb_{\tau} - \wh\xb_{\tau})^\top}{\Vert\xb_{\tau} - \wh\xb_{\tau}\Vert_2}\frac{d(\xb_{\tau}-\wh\xb_{\tau})}{dt}\Big|\leq \Big\Vert\frac{d(\xb_{\tau} -\wh\xb_{\tau})}{d\tau}\Big\Vert_2\\
        &\leq B_0\Vert\xb_{\tau}-\wh\xb_{\tau}\Vert_2 + o(\log(1/\delta)).
    \end{align*}
    We then introduce $\delta_{\tau}:=\Vert\wh\xb_{\tau}-\xb_{\tau}\Vert_2$ and $\bfa g_t :=\exp(-B_0t)\delta_t$ to obtain that
    \begin{align*}
        \bfa g^\prime_t =\exp(-B_0t)\lef(\delta_t^\prime - B_0\delta_t\rig)=\exp(-B_0t)\cdot o\lef(\log(1/\delta)\rig).
    \end{align*}
    Note that $\bfa g_0=0$ since $\delta_0= 0$.
    And we can show that
    \begin{align*}
        \bfa g_t =\bfa g_{t} -\bfa g_0\leq \int_{0}^t\sqrt{\frac{\exp(-2B_0t)d\log(n/\delta)}{n}}dt = \frac{1-\exp(B_0t)}{B_0}\cdot o\lef(\log(1/\delta)\rig).
    \end{align*}
    Hence, we conclude that with probability at least $1-\delta$ we have for all $\lef(\xb_0,\tau,\xb_{\tau}\rig)\in K$,
    \begin{align*}
        \Vert\xb_{\tau}-\wh\xb_{\tau}\Vert_2=o\lef(\log(1/\delta)\rig).
    \end{align*}
    We further can show that the $1$-Wasserstein distance satisfies
    \begin{align*}
        W_1(\bb P_{\xb_1},\bb P_{\wh\xb_1}) &\leq \bb E[\Vert\xb_1 - \wh\xb_{1}\Vert_2]=\bb E[\Vert\xb_1-\wh\xb_1\Vert_2|K]\bb P(K) + \bb E[\Vert\xb_1-\wh\xb_1\Vert_2|K^c]\bb P(K^c)\\
        &\leq C\Big(\frac{(4B_0)^d}{\exp(-\zeta^2/2d)}\sqrt{\frac{d}{n}\log(n/\delta)}\Big)^{\frac{1}{d+1}} +C\sqrt d B_1\exp\lef(-\frac{\zeta^2}{2d}\rig).
    \end{align*}
    Hence, when we optimize over $\zeta$ and the Wasserstein distance can be upper bounded by
    \begin{align*}
        W_1(\bb P_{\xb_1},\bb P_{\wh\xb_1}) \lesssim \Big(\frac{\log(n/\delta)}{n}\Big)^{\frac{1}{d+2}} B_1^{-\frac{d+1}{d+2}}.
    \end{align*}

\section{Additional Experimental Details}\label{app:experiments}

This appendix provides the sliced Wasserstein-1 estimator (Appendix~\ref{app:exp:slicedw1}), the complete per-cell gradient-descent trajectories for the simulation experiments (Appendices~\ref{app:exp:traces}--\ref{app:exp:w1traces}), and the training dynamics and full reconstruction grids for the real-world experiments (Appendix~\ref{app:exp:realworld}).

All experiments were conducted on Princeton University's HPC cluster using NVIDIA H200 GPU nodes. The simulation experiments (Experiments $1-2$) consist of $1,500$ training runs (a $5 \times 3 \times 10$ grid of network width, step size, and dimension, replicated over $10$ random seeds), parallelized across $10$ SLURM array tasks each allocated $1$ GPU, $1$ CPU core, and $16$ GB RAM; total wall-clock time was under 10 GPU-hours. The real-data experiments (MNIST and Fashion-MNIST) consist of $12$ training runs per dataset, parallelized across $4$ array tasks each allocated $1$ GPU, $4$ CPU cores, and $32$ GB RAM; total wall-clock time was under $2$ GPU-hours per dataset.

\paragraph{Recap of the simulation setup.} The target $\bb P_1$ is a balanced two-component Gaussian mixture in $\bb R^d$ with means $\pm 2\bfa e_1$ and identity covariance, the conditional source is $\bb P(\xb_0\mid\xb_1)=N(\xb_1,\bfa I_d)$, and the schedule is the linear one $\sigma_\tau = 1-\tau$, $\mu_\tau = \tau$. The two-layer ReLU network $\bfa f_{\bfa\theta}$ is trained by full-batch gradient descent for $T=500$ iterations with $\kappa=1$, both layers trainable. Sampling under $\bfa\theta^{(t)}$ uses $T_{\rm Euler}=200$ Euler steps and the truncation radius $2\sqrt d B_1$ with $B_1=10$. We sweep five widths $m\in\{16,32,64,128,256\}$, three step sizes $\eta\in\{10^{-4},10^{-3},10^{-2}\}$ and ten ambient dimensions $d\in\{5,10,15,\dots,50\}$, and we evaluate both metrics every $10$ gradient-descent steps, yielding $51$ checkpoints $t\in\{0,10,\dots,500\}$ per cell. Each of the $150$ grid cells is independently repeated over $10$ random initializations; all figures show the mean with shaded $\pm1$ standard-deviation bands.

\subsection{Sliced Wasserstein-1 estimator}\label{app:exp:slicedw1}

Theorem~\ref{thmsampbound} controls population multivariate $W_1$. We use an empirical sliced-$W_1$ estimator because repeatedly solving multivariate discrete optimal transport over 150 configurations and many checkpoints is impractical. For empirical measures $\widehat P,\widehat Q$, we report
\[
\widehat{\operatorname{SW}}_1(\widehat P,\widehat Q)
=\frac1K\sum_{k=1}^K W_1\big((\bfa u_k^\top)_{\widehat P},(\bfa u_k^\top)_{\widehat Q} \big),
\]
where $K=128$ directions are sampled uniformly from the unit sphere, each one-dimensional distance is computed exactly by sorting, and directions are fixed throughout each run.

This diagnostic is related to, but not equal to, full $W_1$. Projection is 1-Lipschitz, so $\operatorname{SW}_1(P,Q)\leq W_1(P,Q)$. We call it empirical sliced Wasserstein-1 and do not claim a dimension-free, equal-scale approximation.

\subsection{Per-cell trajectories of the training loss}\label{app:exp:traces}

Figure~\ref{fig:app:loss-vs-t} shows the training loss $L(\bfa\theta^{(t)})$ along the gradient-descent trajectory for every $(m,\eta)$ pair in the grid. Within each panel, the ten ambient dimensions $d\in\{5,10,\dots,50\}$ are overlaid as separate curves whose color encodes $d$ via the colorbar. Three observations recur across all 15 panels and corroborate Theorem~\ref{theorem1}: (i) the loss decays geometrically in $t$ on the log-scale, (ii) larger $d$ produces strictly larger terminal loss at every fixed $(m,\eta)$, and (iii) the rate of decay is monotone in $\eta$.

\begin{figure}[htbp]
  \centering
  \begin{subfigure}[t]{\gridw}
    \includegraphics[width=\linewidth]{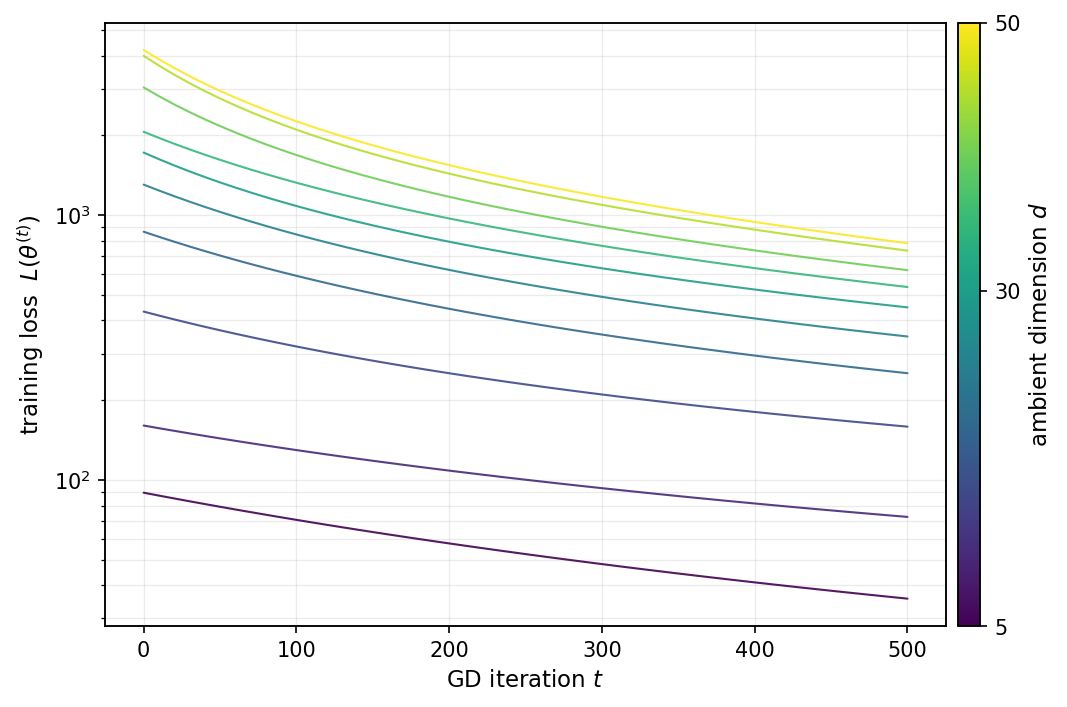}
    \caption{$m=16,\eta=10^{-4}$.}
  \end{subfigure}
  \begin{subfigure}[t]{\gridw}
    \includegraphics[width=\linewidth]{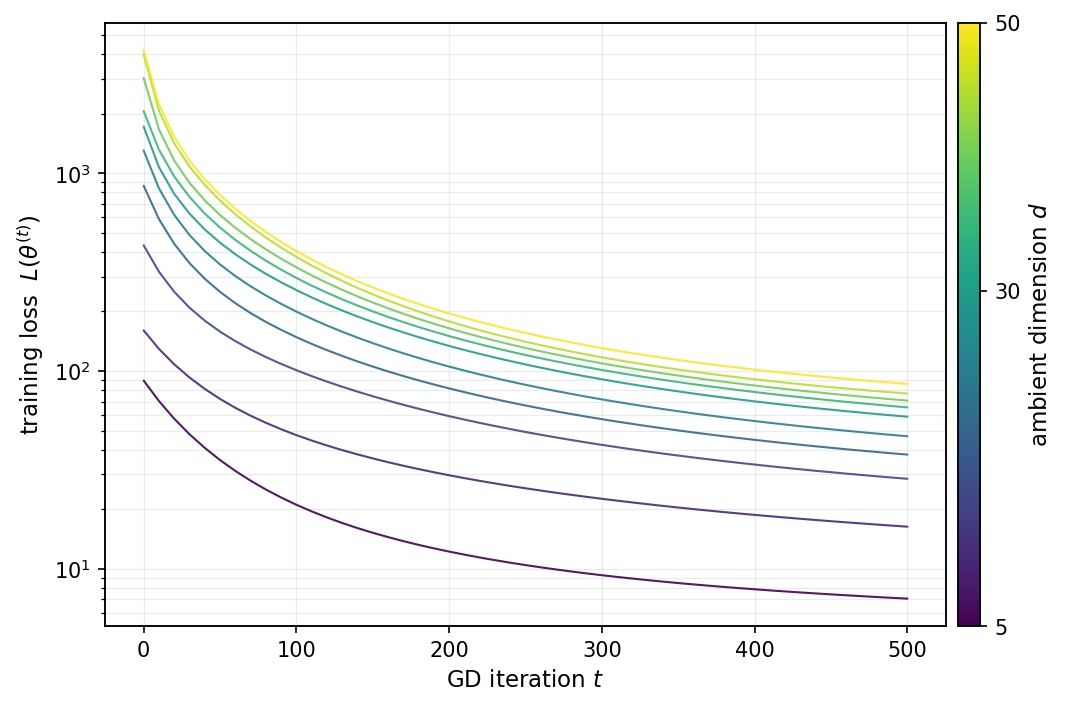}
    \caption{$m=16,\eta=10^{-3}$.}
  \end{subfigure}
  \begin{subfigure}[t]{\gridw}
    \includegraphics[width=\linewidth]{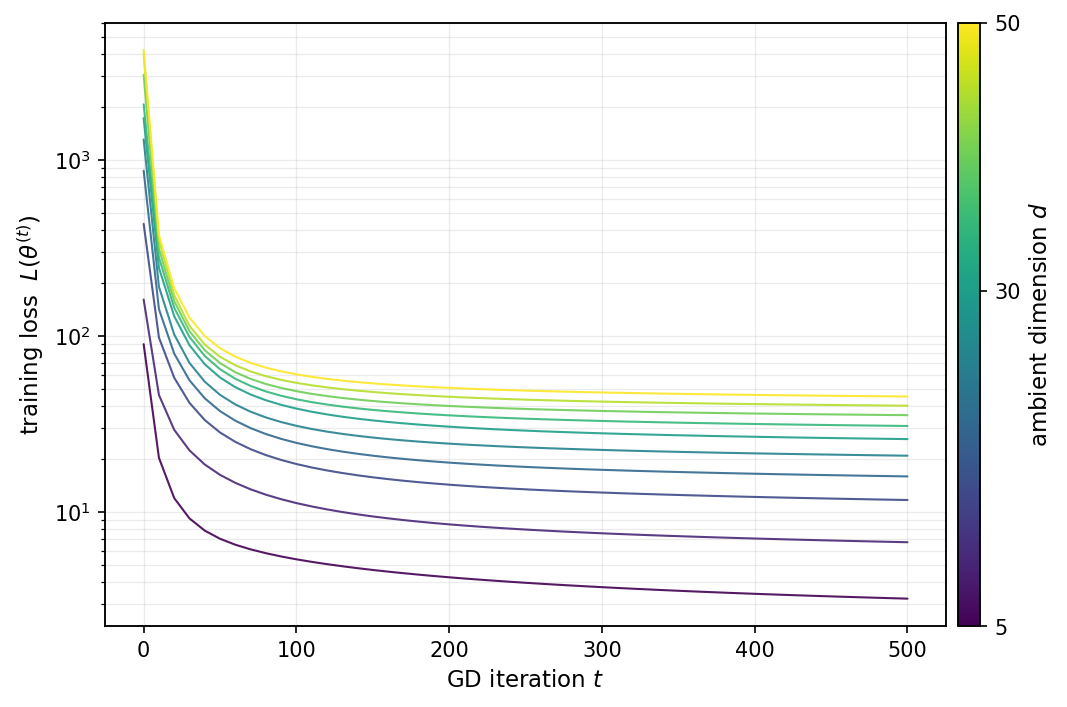}
    \caption{$m=16,\eta=10^{-2}$.}
  \end{subfigure}\\[0.5em]
  \begin{subfigure}[t]{\gridw}
    \includegraphics[width=\linewidth]{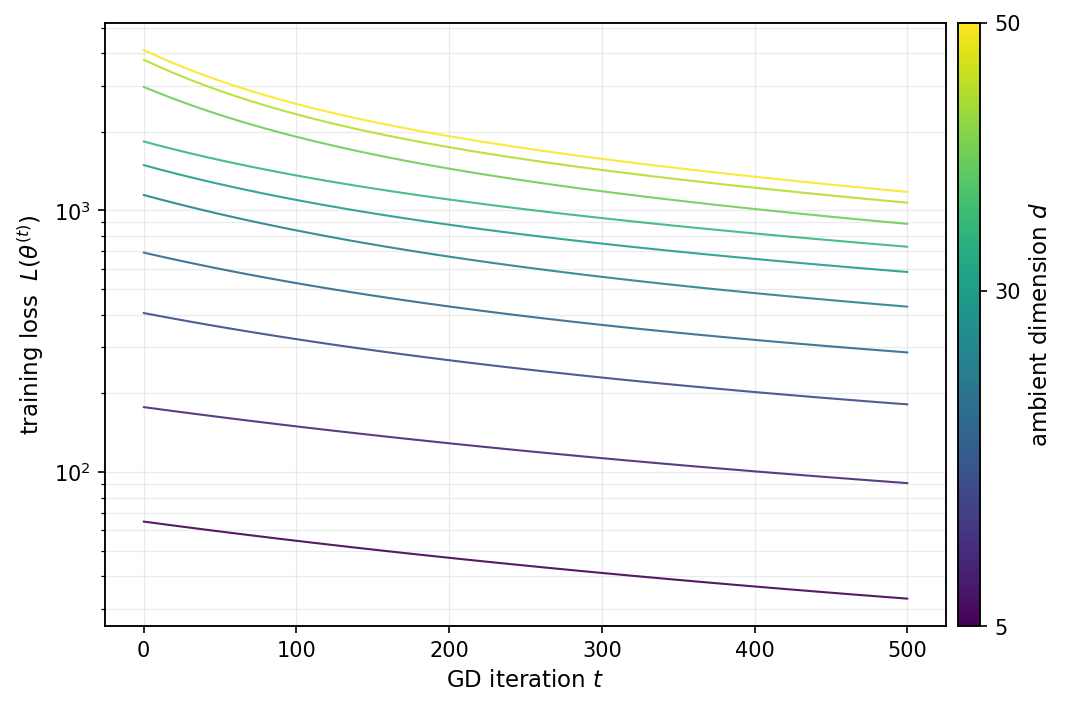}
    \caption{$m=32,\eta=10^{-4}$.}
  \end{subfigure}
  \begin{subfigure}[t]{\gridw}
    \includegraphics[width=\linewidth]{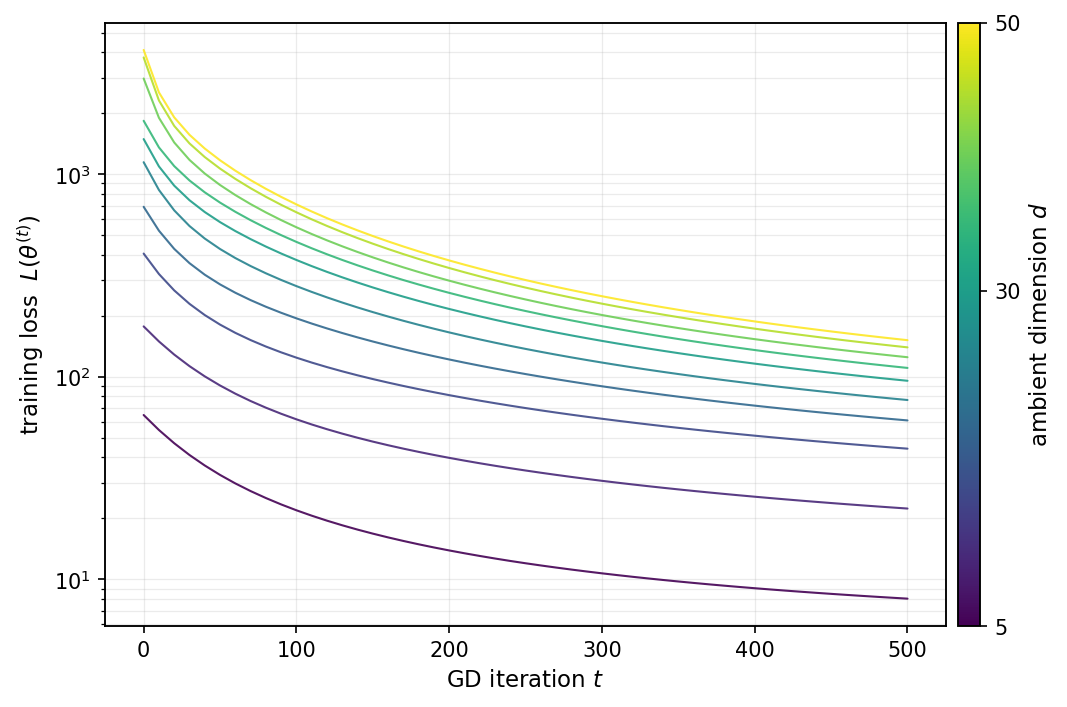}
    \caption{$m=32,\eta=10^{-3}$.}
  \end{subfigure}
  \begin{subfigure}[t]{\gridw}
    \includegraphics[width=\linewidth]{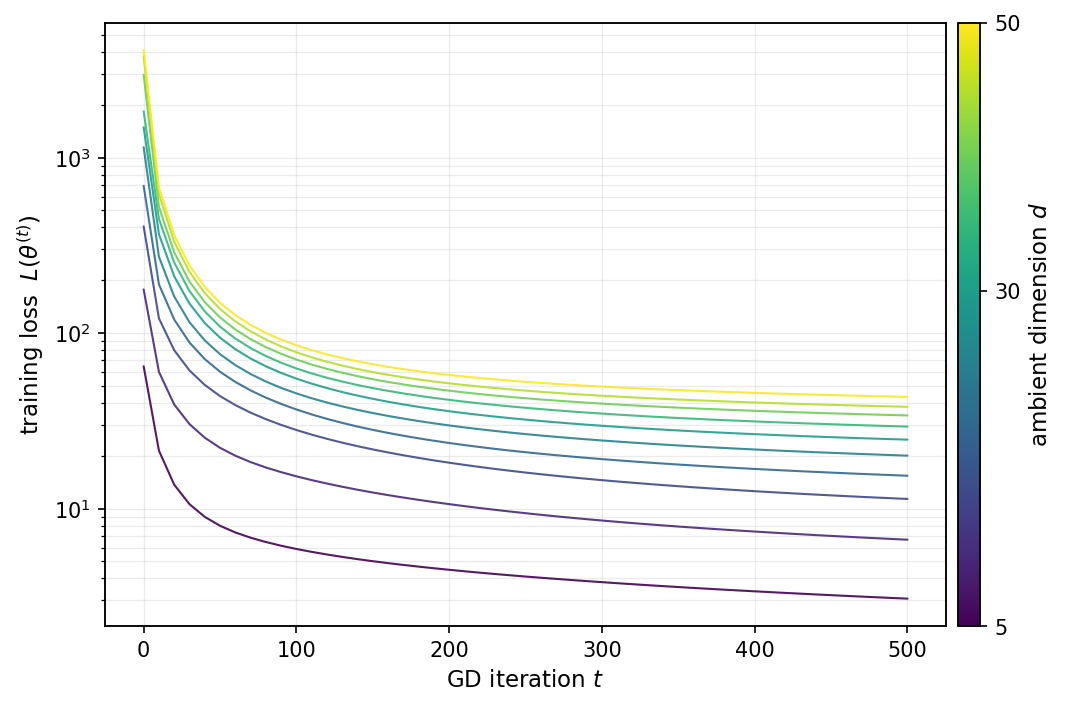}
    \caption{$m=32,\eta=10^{-2}$.}
  \end{subfigure}\\[0.5em]
  \begin{subfigure}[t]{\gridw}
    \includegraphics[width=\linewidth]{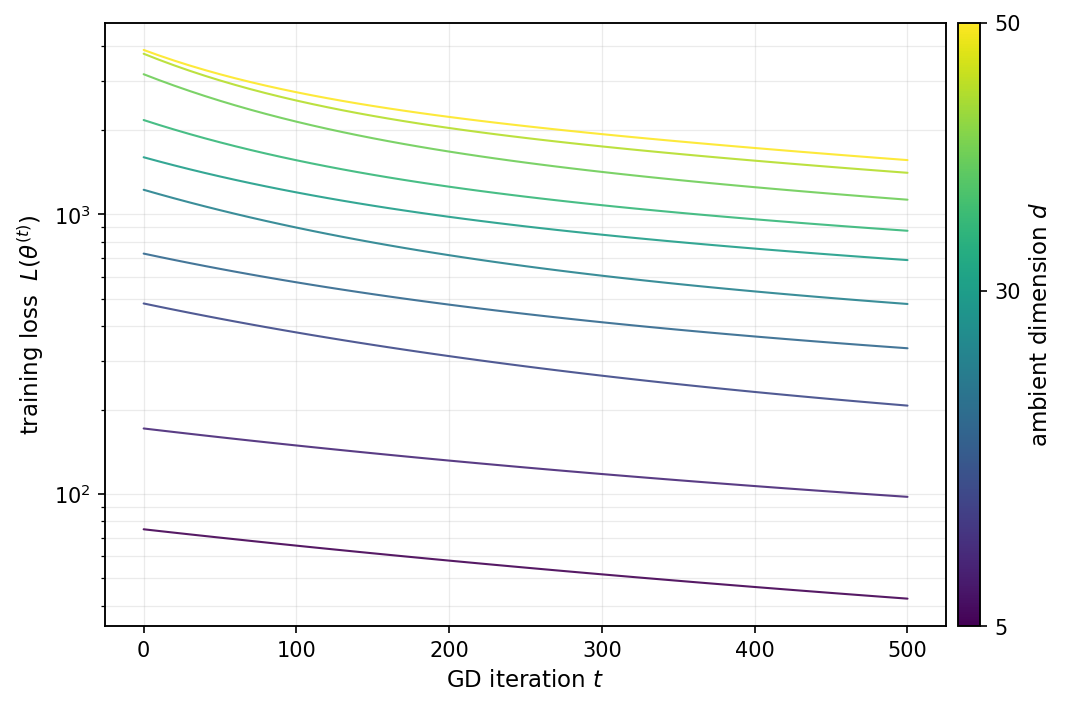}
    \caption{$m=64,\eta=10^{-4}$.}
  \end{subfigure}
  \begin{subfigure}[t]{\gridw}
    \includegraphics[width=\linewidth]{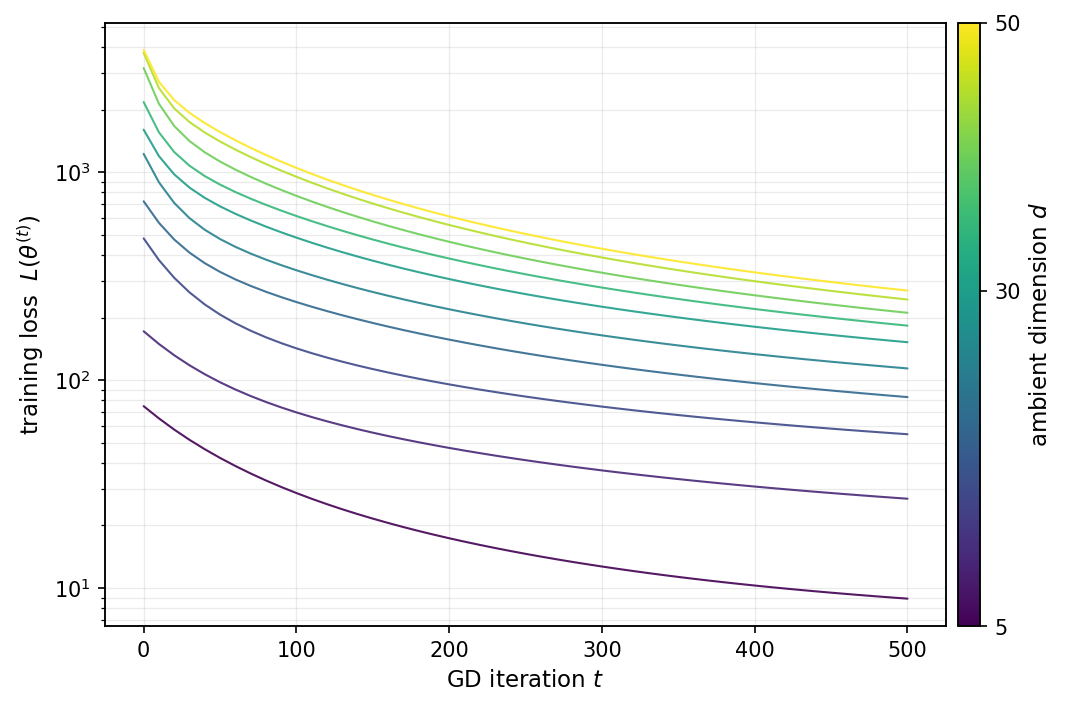}
    \caption{$m=64,\eta=10^{-3}$.}
  \end{subfigure}
  \begin{subfigure}[t]{\gridw}
    \includegraphics[width=\linewidth]{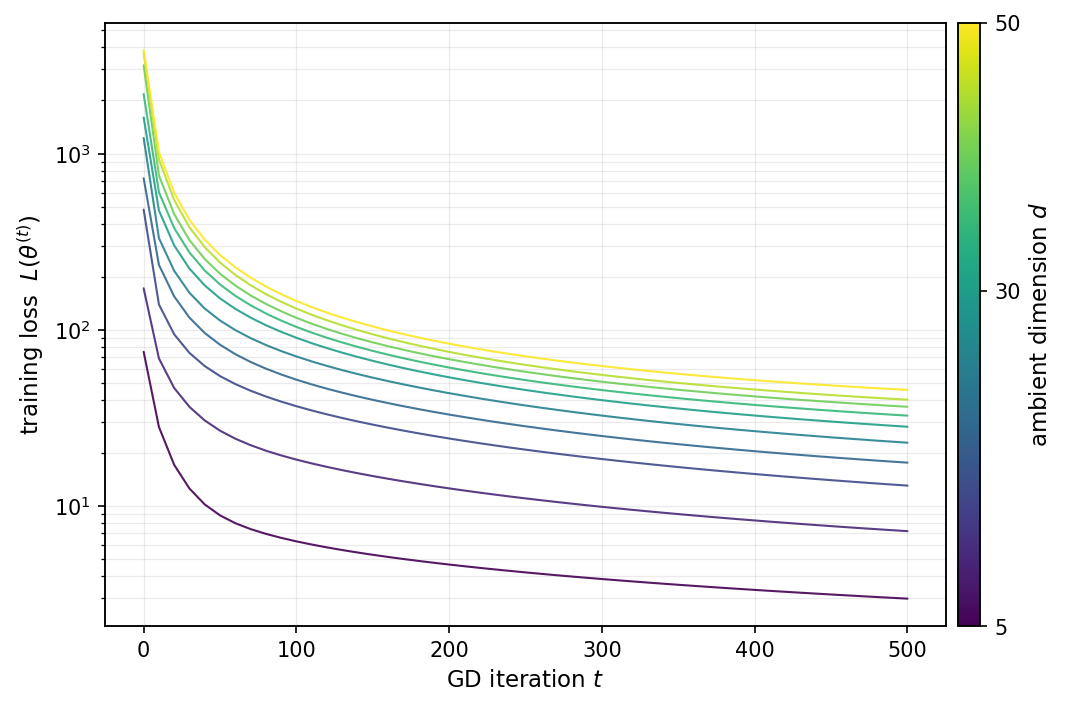}
    \caption{$m=64,\eta=10^{-2}$.}
  \end{subfigure}\\[0.5em]
  \begin{subfigure}[t]{\gridw}
    \includegraphics[width=\linewidth]{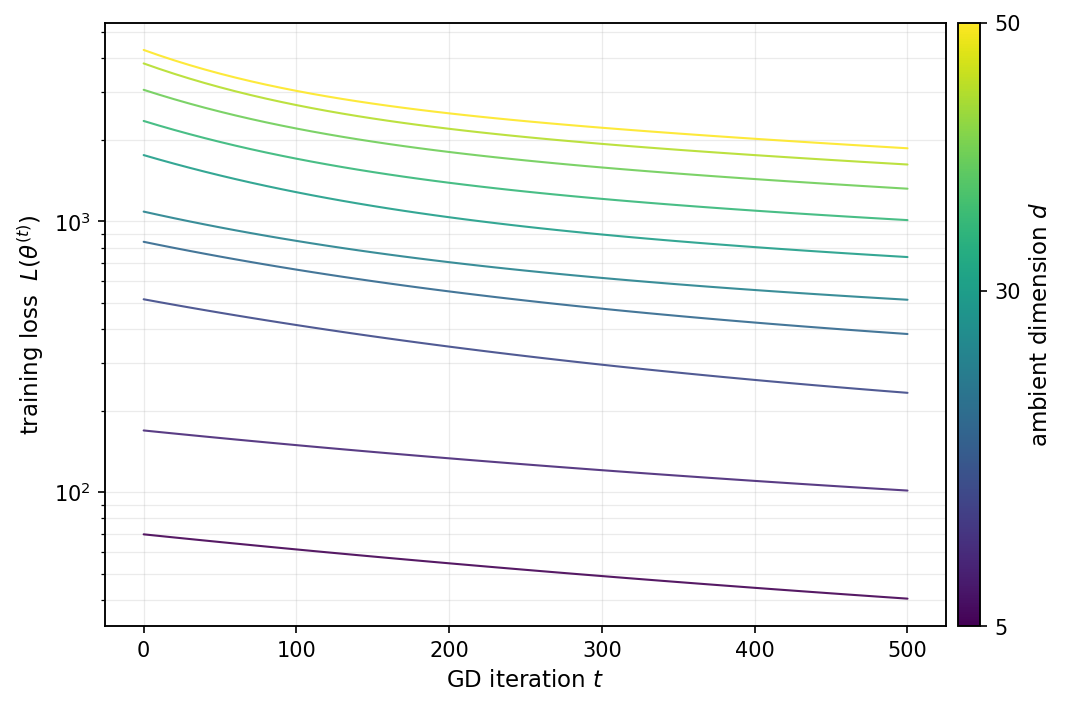}
    \caption{$m=128,\eta=10^{-4}$.}
  \end{subfigure}
  \begin{subfigure}[t]{\gridw}
    \includegraphics[width=\linewidth]{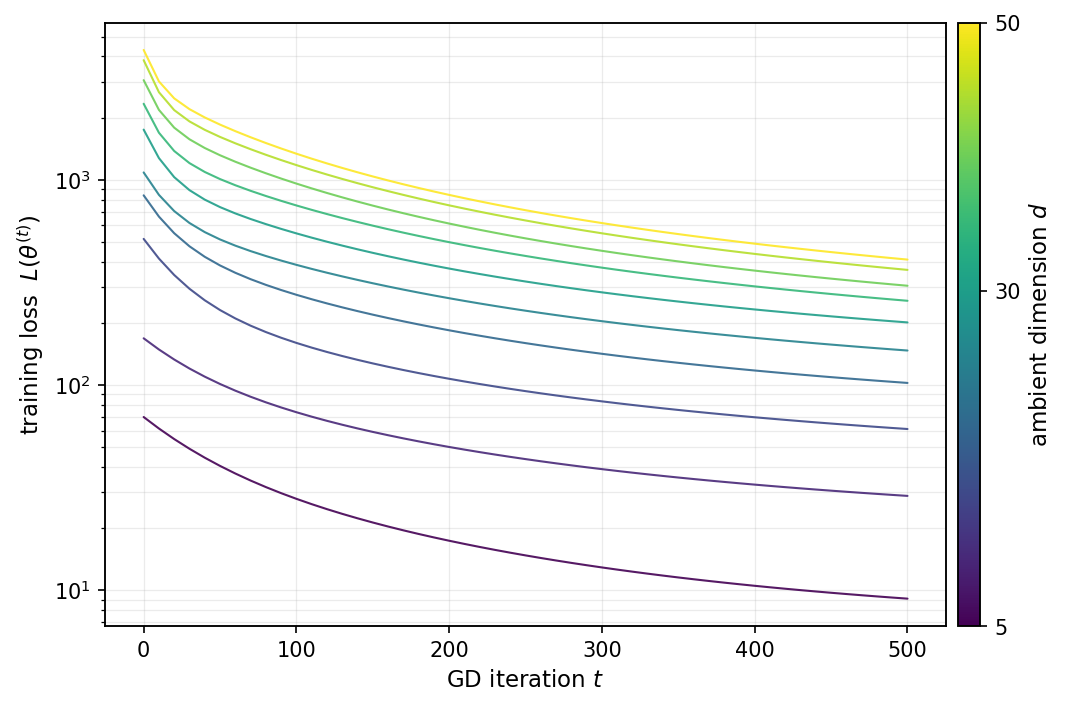}
    \caption{$m=128,\eta=10^{-3}$.}
  \end{subfigure}
  \begin{subfigure}[t]{\gridw}
    \includegraphics[width=\linewidth]{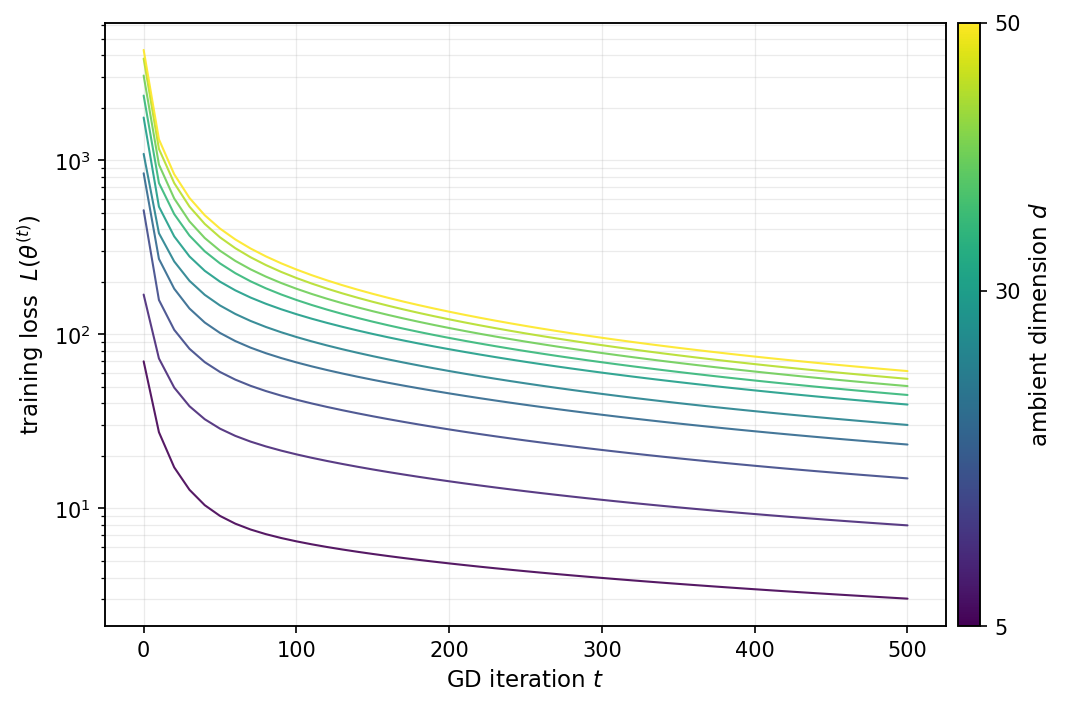}
    \caption{$m=128,\eta=10^{-2}$.}
  \end{subfigure}\\[0.5em]
  \begin{subfigure}[t]{\gridw}
    \includegraphics[width=\linewidth]{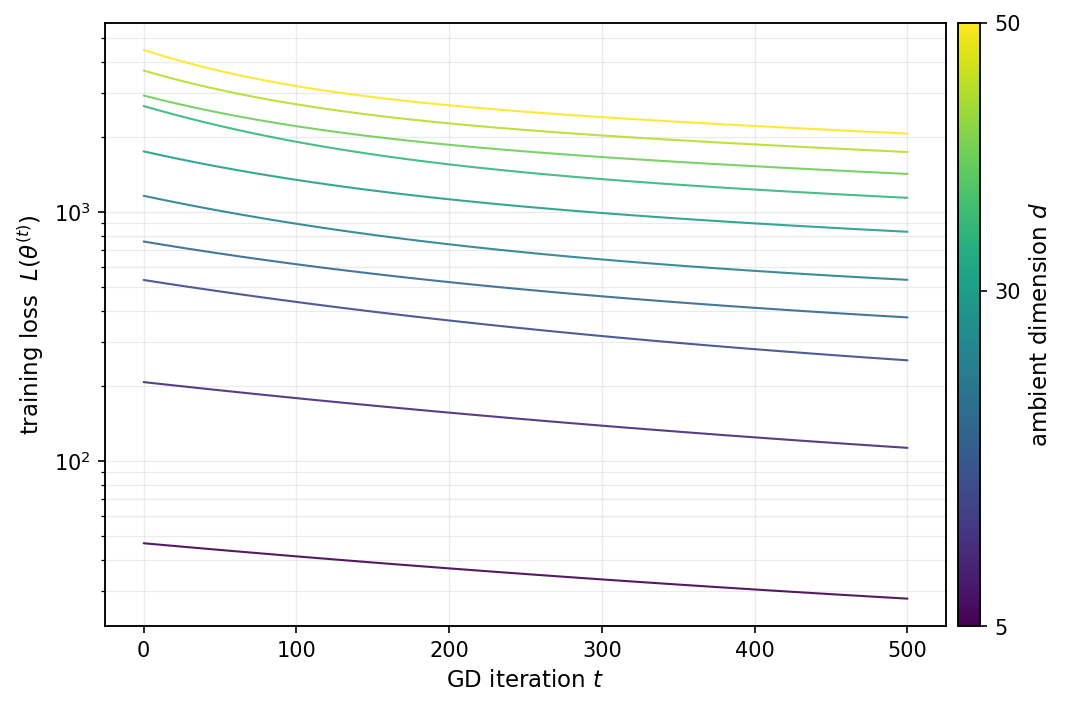}
    \caption{$m=256,\eta=10^{-4}$.}
  \end{subfigure}
  \begin{subfigure}[t]{\gridw}
    \includegraphics[width=\linewidth]{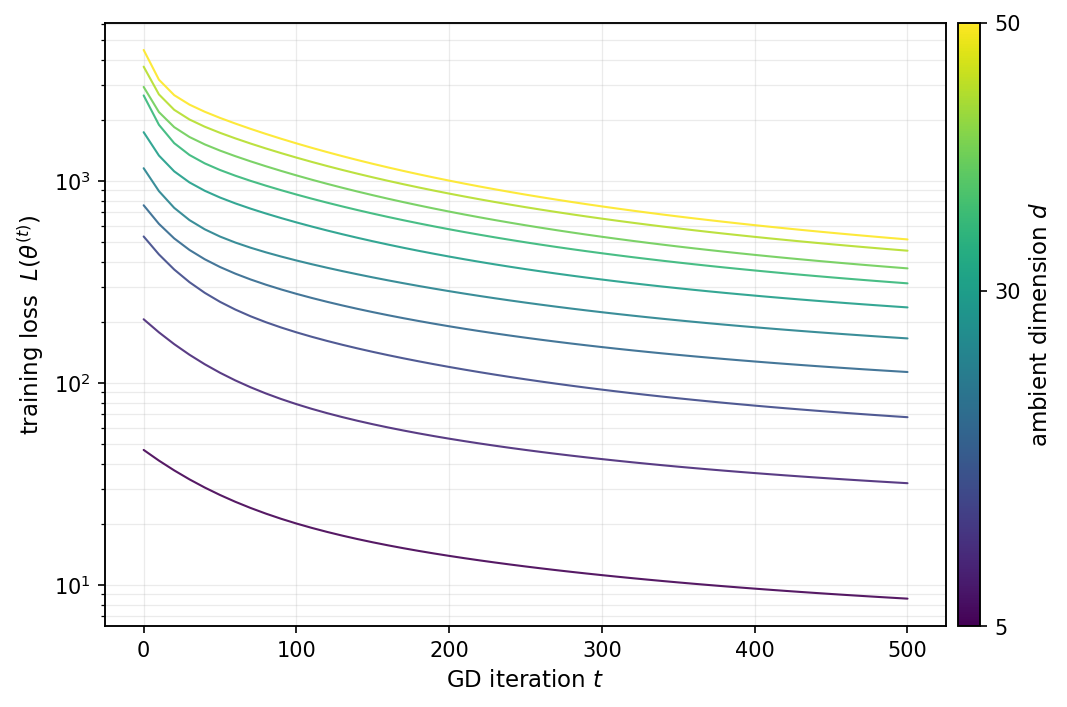}
    \caption{$m=256,\eta=10^{-3}$.}
  \end{subfigure}
  \begin{subfigure}[t]{\gridw}
    \includegraphics[width=\linewidth]{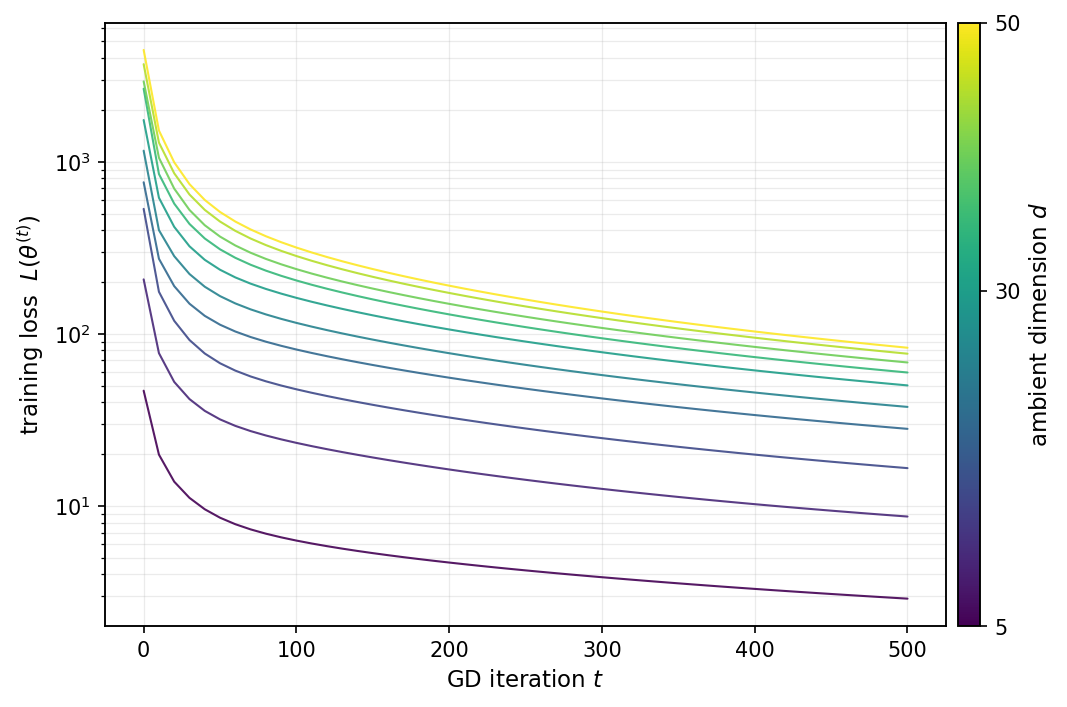}
    \caption{$m=256,\eta=10^{-2}$.}
  \end{subfigure}
  \caption{Training loss $L(\bfa\theta^{(t)})$ along the gradient-descent trajectory for every $(m,\eta)$ in the grid, $n_{\rm train}=500$, $n_{\rm test}=5000$. Within each panel, color encodes the ambient dimension $d\in\{5,10,\dots,50\}$ (color bar); curves show the mean over $10$ independent random seeds. All 15 panels exhibit geometric decay in $t$ on the log-scale, in line with Theorem~\ref{theorem1}.}
  \label{fig:app:loss-vs-t}
\end{figure}

\subsection{Per-cell trajectories of the sliced Wasserstein-1 distance}\label{app:exp:w1traces}

Figure~\ref{fig:app:w1-vs-t} reports the same per-cell view for the sliced Wasserstein-1 distance $\overline{W_1}(\bb P_1,\bb P_{\wh\xb_1^{(t)}})$ between the data distribution and the distribution of samples drawn by Algorithm~\ref{alg:learning} under the current iterate $\bfa\theta^{(t)}$. The same monotone behaviors are visible: $\overline{W_1}$ drops along the gradient-descent trajectory at every $(m,\eta,d)$ and terminates at a value that is strictly larger for larger $d$. This is the empirical correspondence to the dimension-dependent rate $C(d)\big(\log(n/\delta)/n\big)^{1/(2(d+2))} B_1^{-(d+1)/(d+2)}$ of Theorem~\ref{thmsampbound}.

\begin{figure}[htbp]
  \centering
  \begin{subfigure}[t]{\gridw}
    \includegraphics[width=\linewidth]{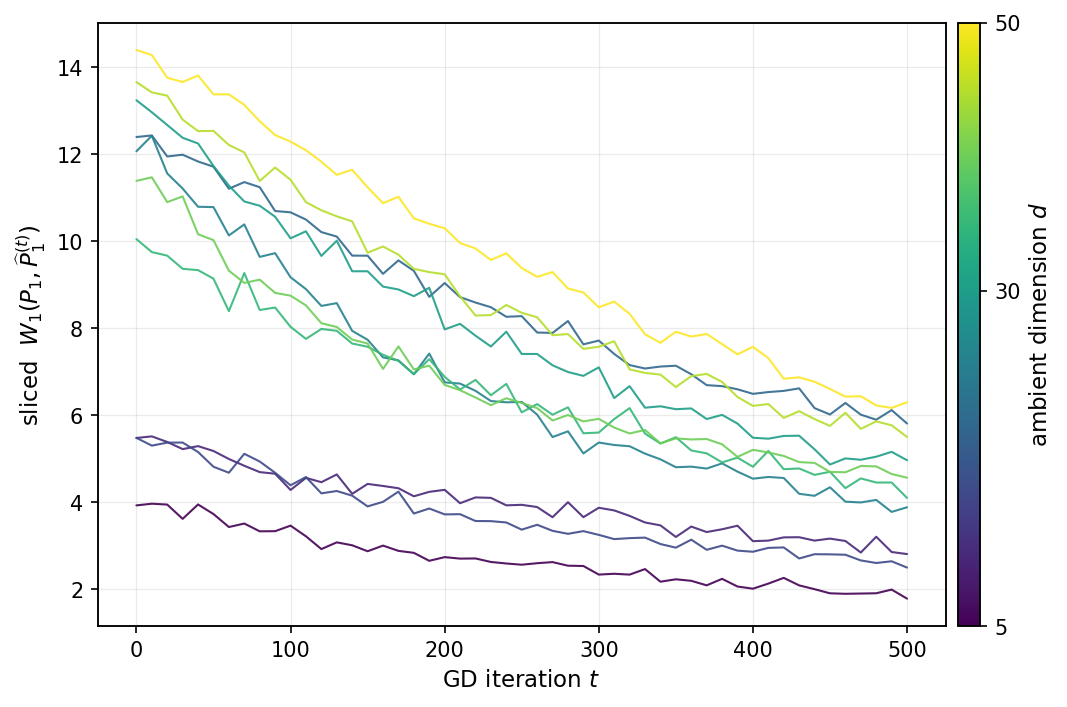}
    \caption{$m=16,\eta=10^{-4}$.}
  \end{subfigure}
  \begin{subfigure}[t]{\gridw}
    \includegraphics[width=\linewidth]{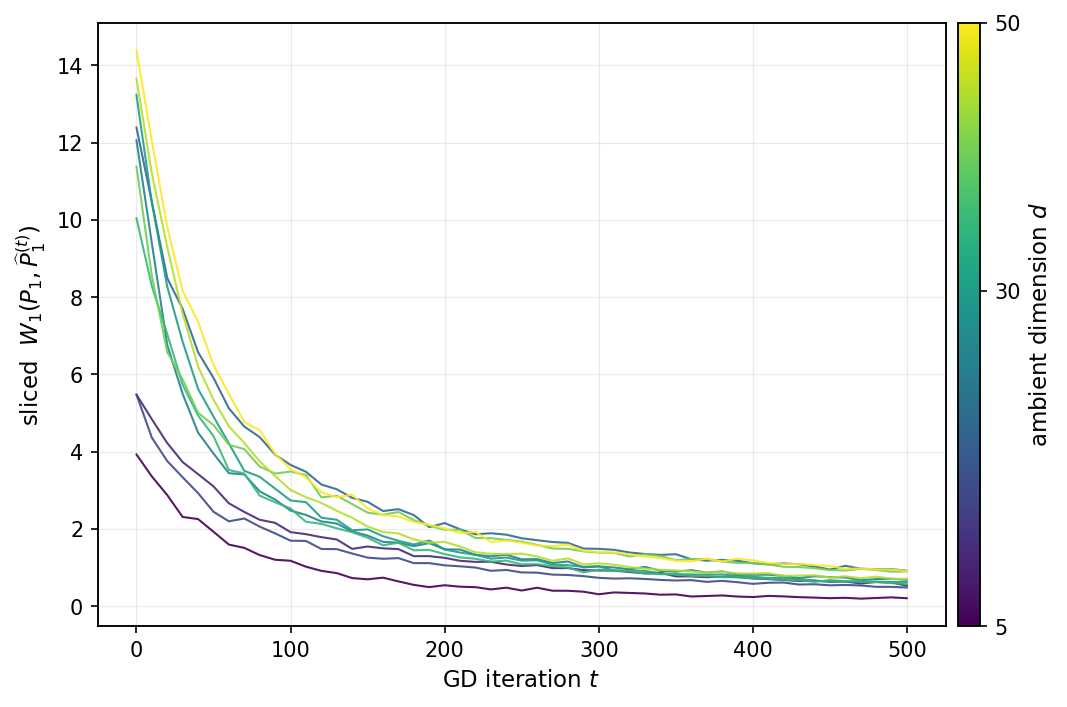}
    \caption{$m=16,\eta=10^{-3}$.}
  \end{subfigure}
  \begin{subfigure}[t]{\gridw}
    \includegraphics[width=\linewidth]{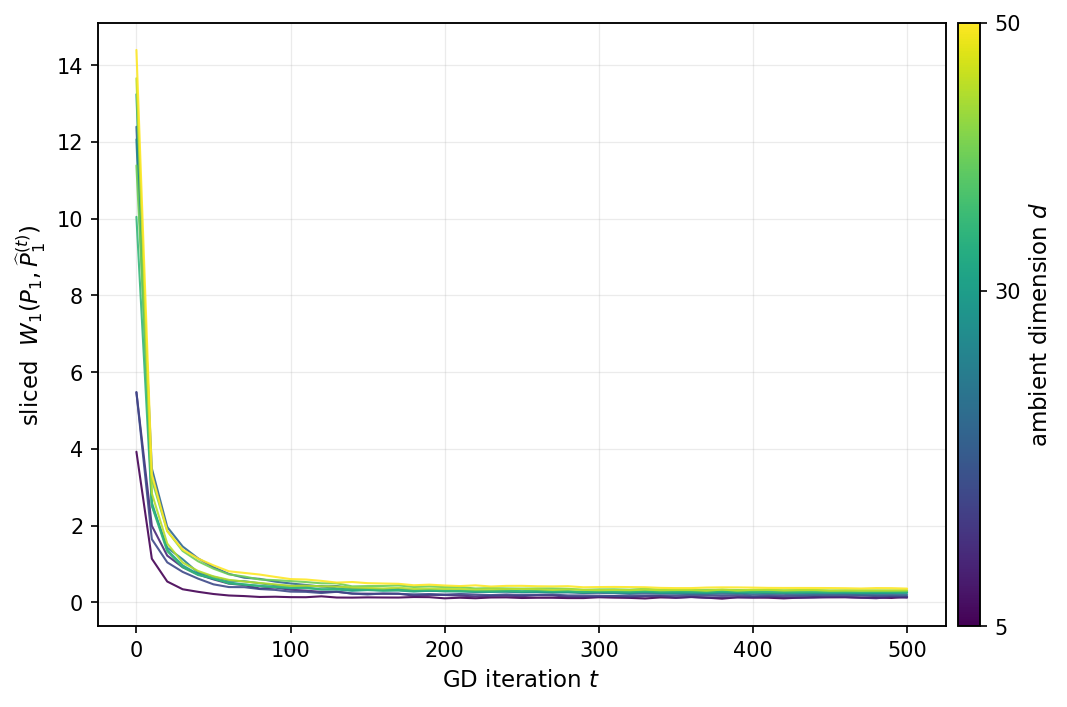}
    \caption{$m=16,\eta=10^{-2}$.}
  \end{subfigure}\\[0.5em]
  \begin{subfigure}[t]{\gridw}
    \includegraphics[width=\linewidth]{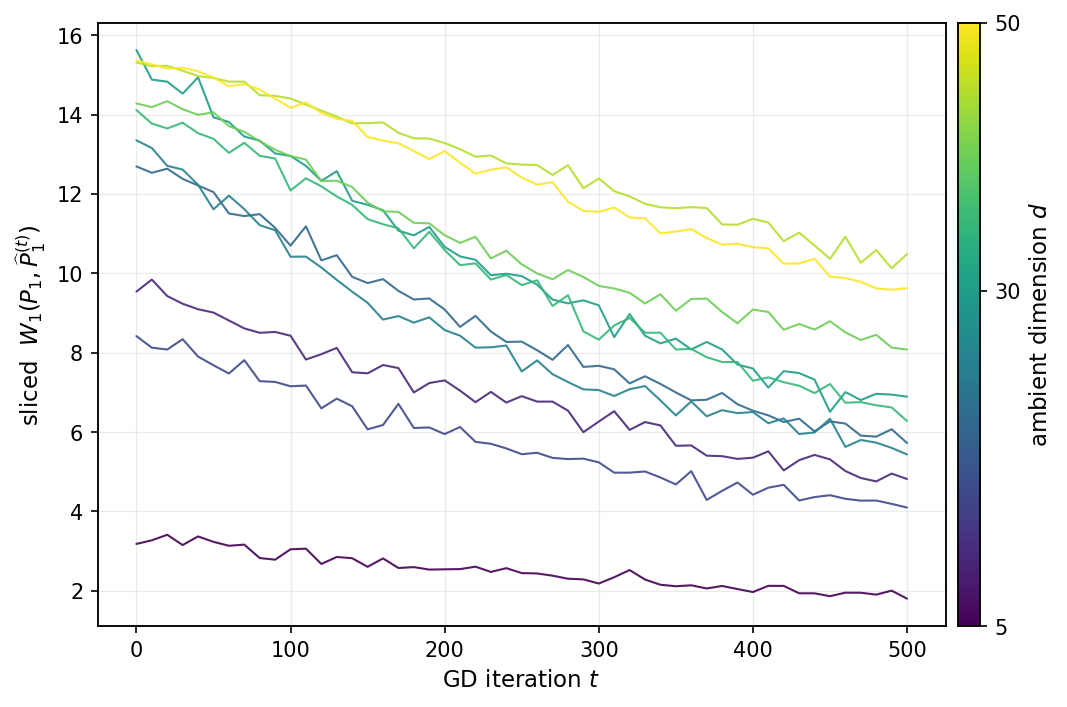}
    \caption{$m=32,\eta=10^{-4}$.}
  \end{subfigure}
  \begin{subfigure}[t]{\gridw}
    \includegraphics[width=\linewidth]{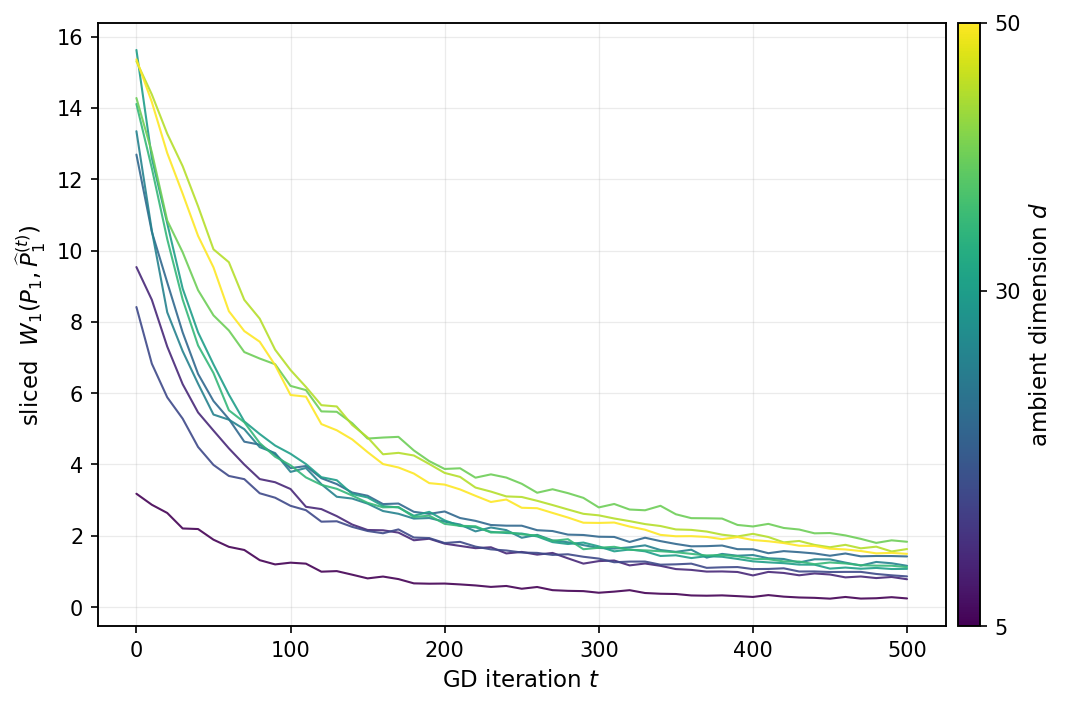}
    \caption{$m=32,\eta=10^{-3}$.}
  \end{subfigure}
  \begin{subfigure}[t]{\gridw}
    \includegraphics[width=\linewidth]{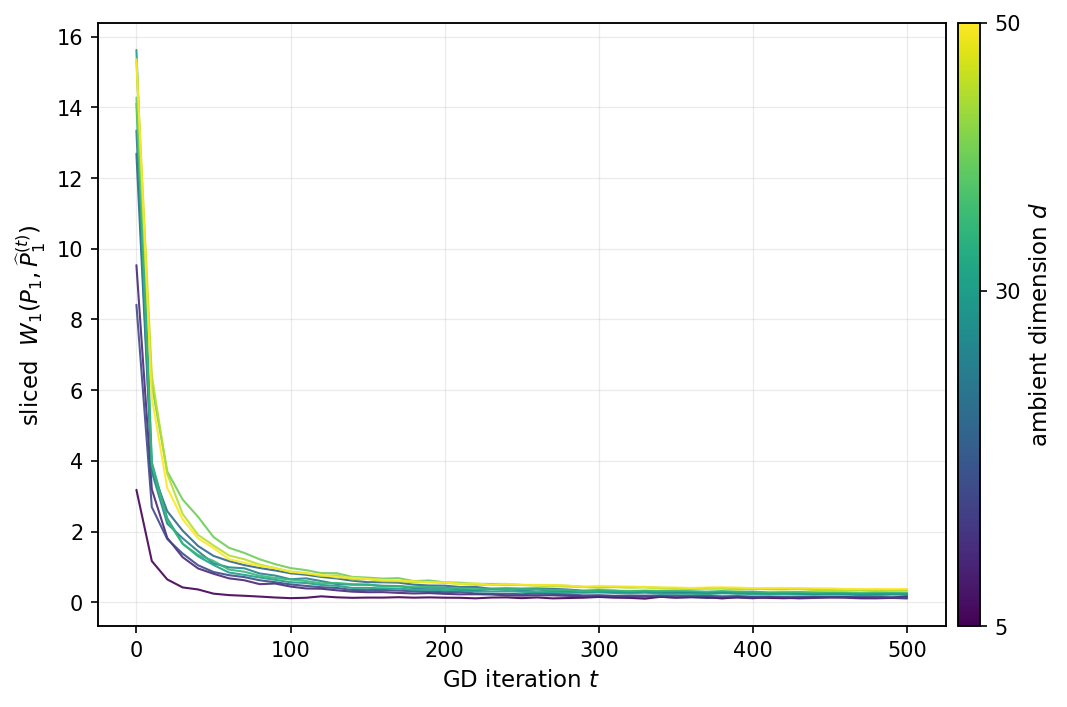}
    \caption{$m=32,\eta=10^{-2}$.}
  \end{subfigure}\\[0.5em]
  \begin{subfigure}[t]{\gridw}
    \includegraphics[width=\linewidth]{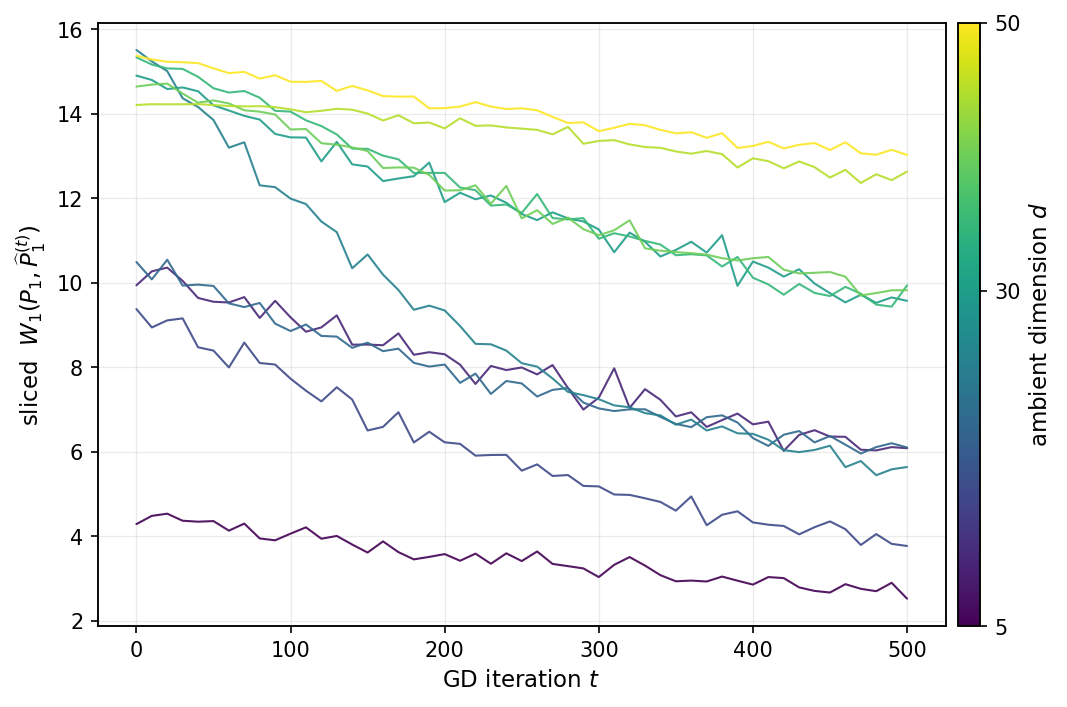}
    \caption{$m=64,\eta=10^{-4}$.}
  \end{subfigure}
  \begin{subfigure}[t]{\gridw}
    \includegraphics[width=\linewidth]{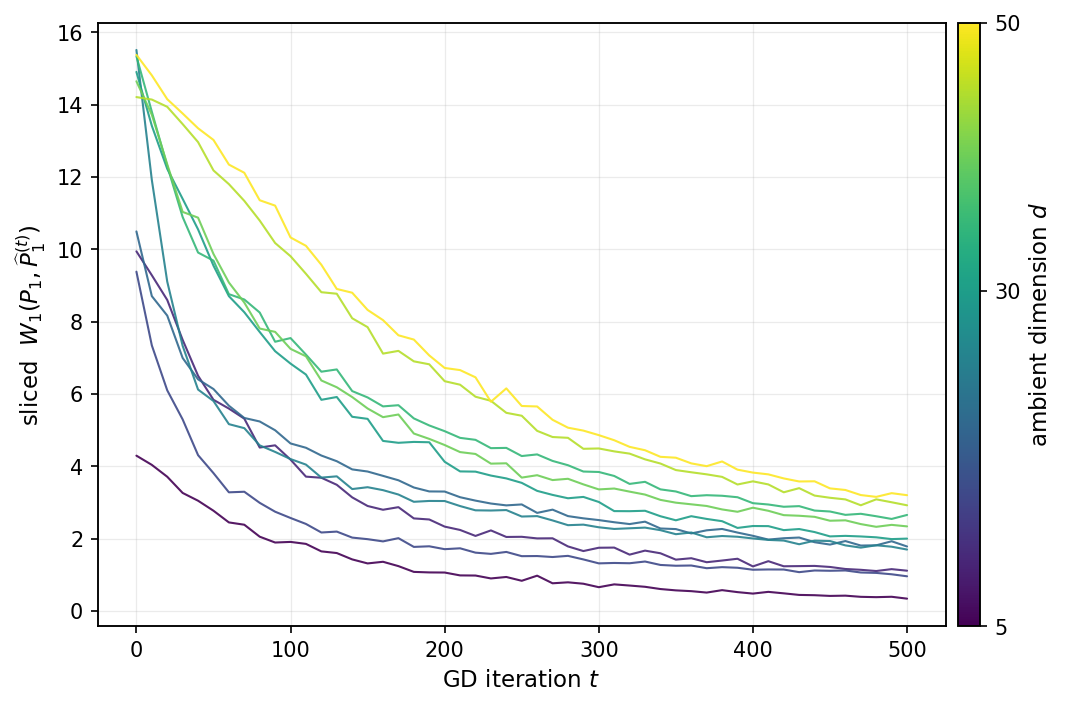}
    \caption{$m=64,\eta=10^{-3}$.}
  \end{subfigure}
  \begin{subfigure}[t]{\gridw}
    \includegraphics[width=\linewidth]{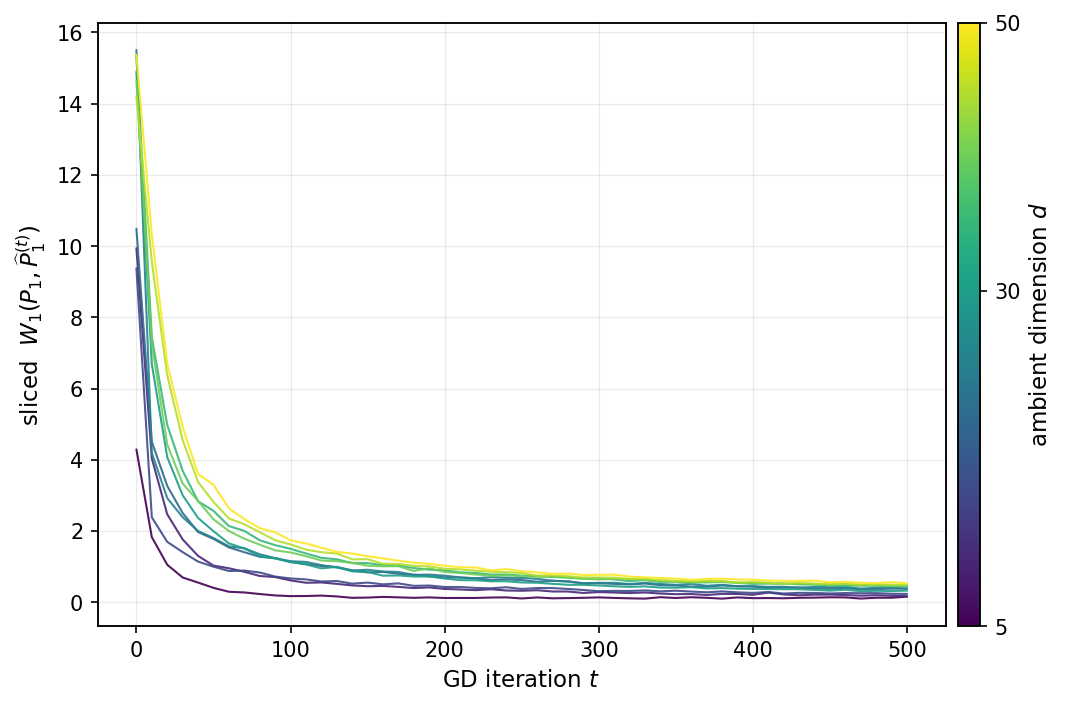}
    \caption{$m=64,\eta=10^{-2}$.}
  \end{subfigure}\\[0.5em]
  \begin{subfigure}[t]{\gridw}
    \includegraphics[width=\linewidth]{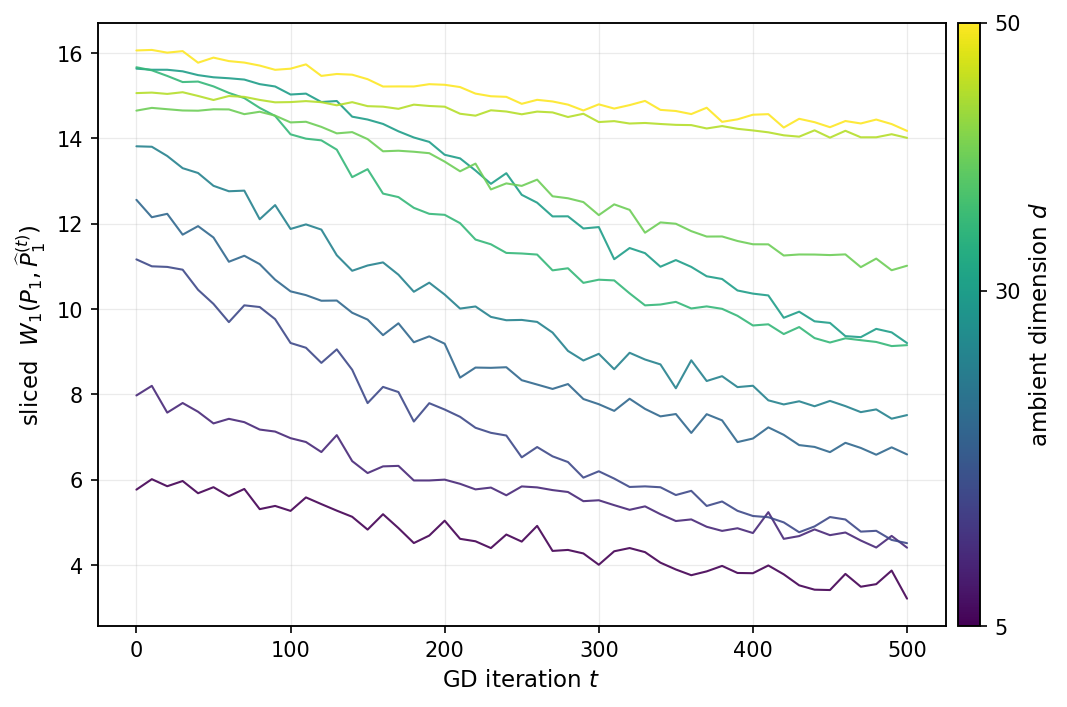}
    \caption{$m=128,\eta=10^{-4}$.}
  \end{subfigure}
  \begin{subfigure}[t]{\gridw}
    \includegraphics[width=\linewidth]{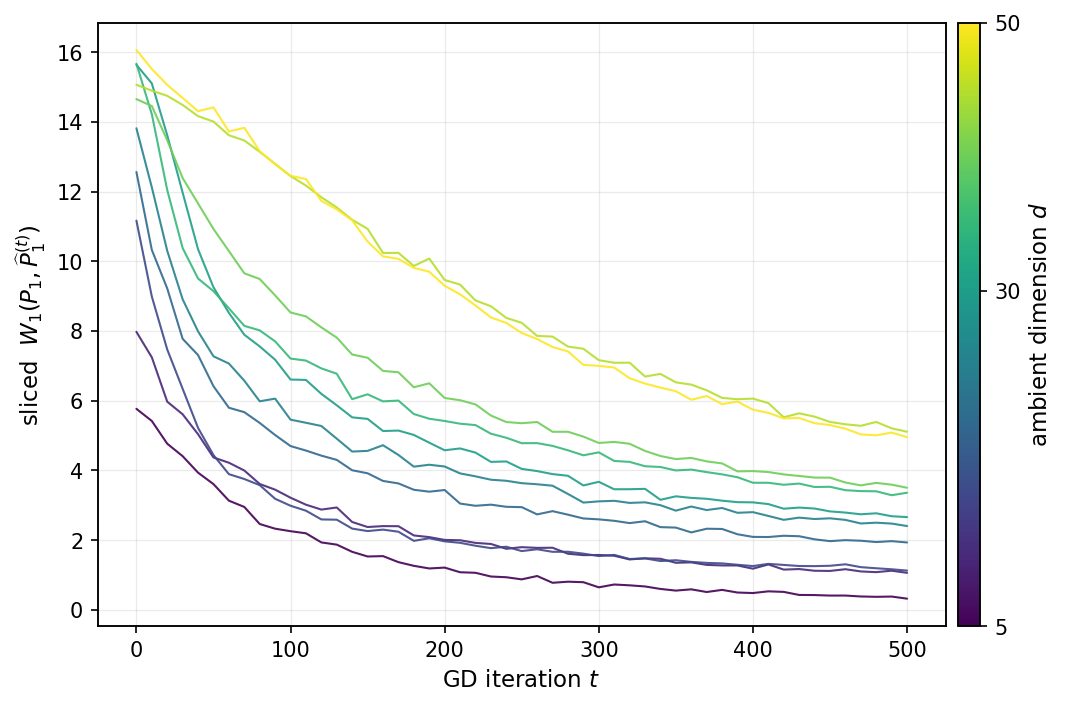}
    \caption{$m=128,\eta=10^{-3}$.}
  \end{subfigure}
  \begin{subfigure}[t]{\gridw}
    \includegraphics[width=\linewidth]{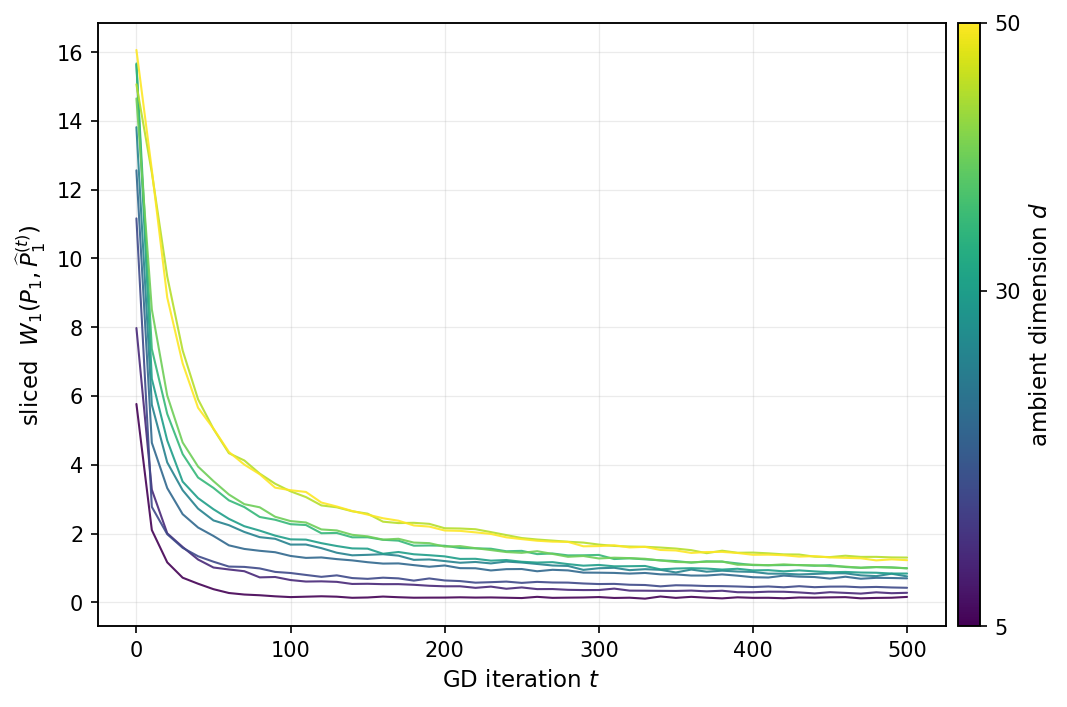}
    \caption{$m=128,\eta=10^{-2}$.}
  \end{subfigure}\\[0.5em]
  \begin{subfigure}[t]{\gridw}
    \includegraphics[width=\linewidth]{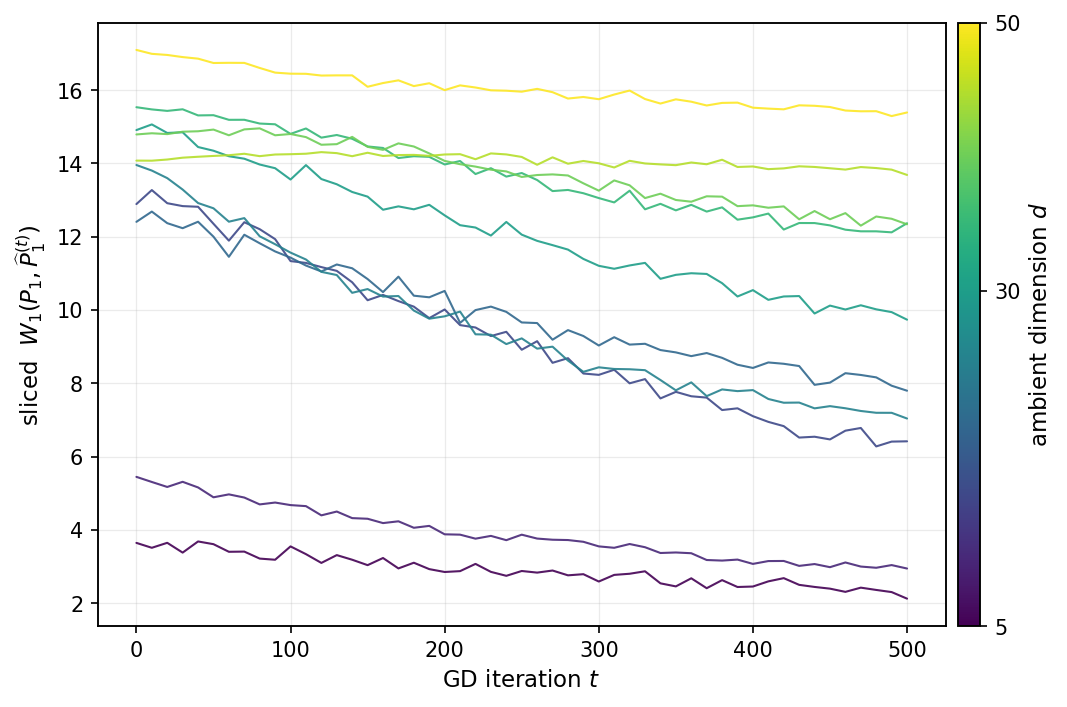}
    \caption{$m=256,\eta=10^{-4}$.}
  \end{subfigure}
  \begin{subfigure}[t]{\gridw}
    \includegraphics[width=\linewidth]{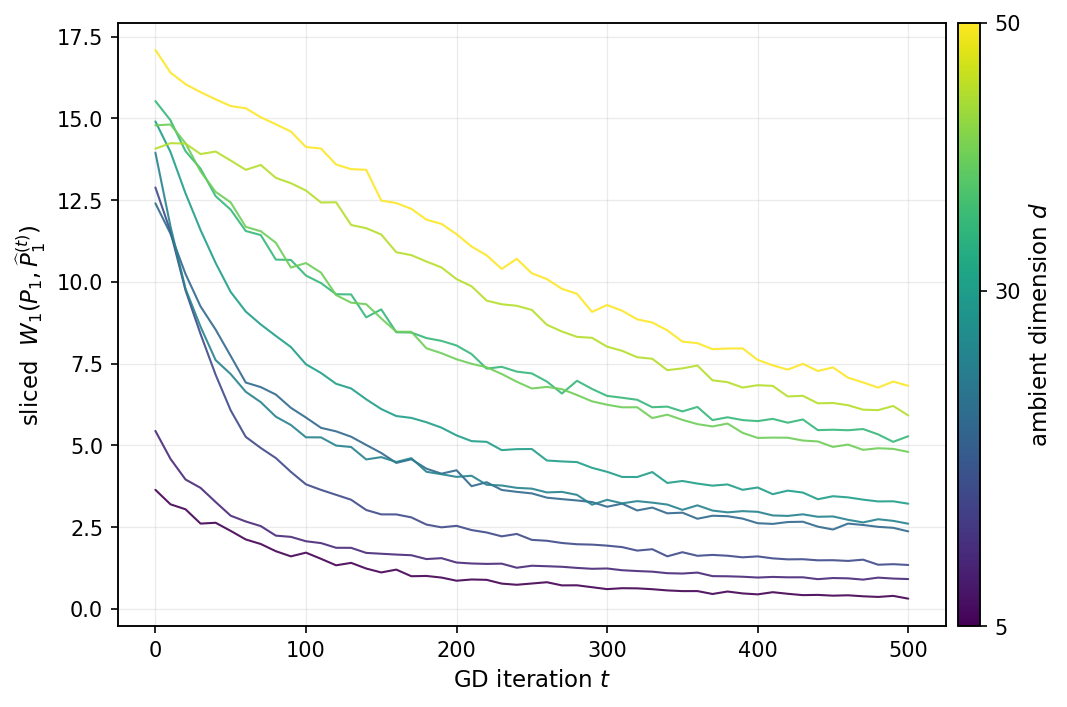}
    \caption{$m=256,\eta=10^{-3}$.}
  \end{subfigure}
  \begin{subfigure}[t]{\gridw}
    \includegraphics[width=\linewidth]{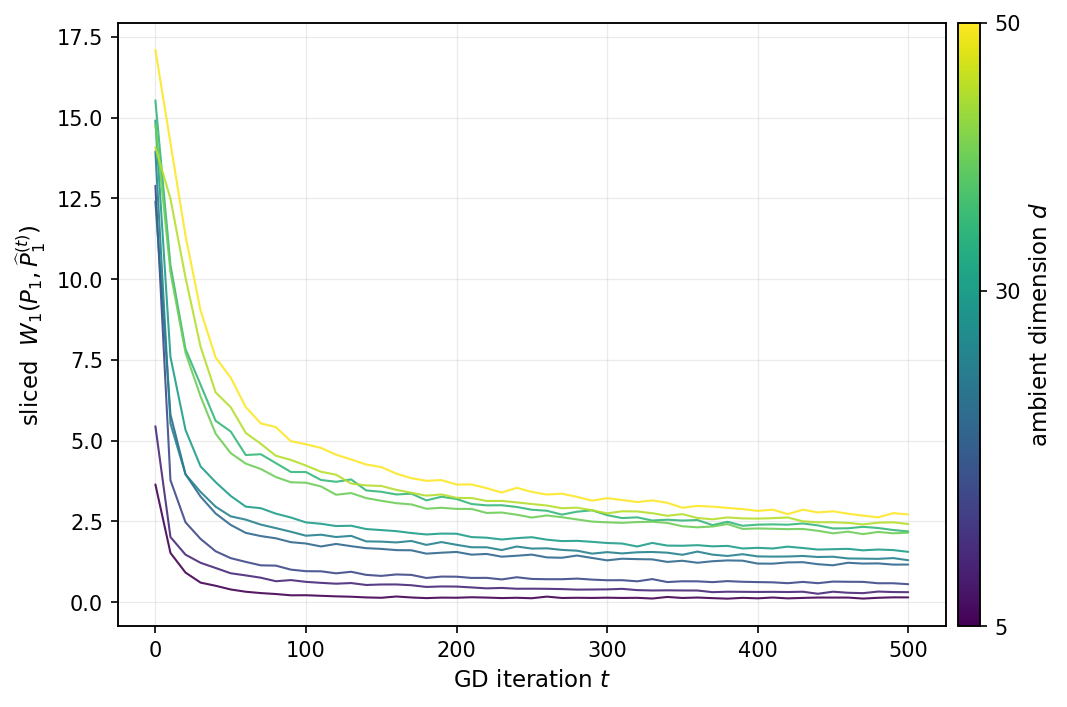}
    \caption{$m=256,\eta=10^{-2}$.}
  \end{subfigure}
  \caption{Sliced Wasserstein-1 distance $\overline{W_1}(\bb P_1,\bb P_{\wh\xb_1^{(t)}})$ along the gradient-descent trajectory for every $(m,\eta)$ in the grid, $n_{\rm train}=500$, $n_{\rm test}=5000$. Within each panel, color encodes the ambient dimension $d\in\{5,10,\dots,50\}$ (color bar); curves show the mean over $10$ independent random seeds. The distance drops monotonically along the trajectory and saturates at a value increasing in $d$, in line with Theorem~\ref{thmsampbound}.}
  \label{fig:app:w1-vs-t}
\end{figure}

\subsection{Real-World Datasets: MNIST and Fashion-MNIST}\label{app:exp:realworld}

The setup is described in Section~\ref{sect:7:realworld}: $d=7$ PCA components fitted on the full training split, $n_{\rm train}=500$, $n_{\rm test}=2000$, linear schedule, $\bb P(\zb_0\mid\zb_1)=N(\zb_1,\bfa I_7)$, $T=500$ GD iterations, $T_{\rm Euler}=200$ Euler steps, $B_1=10$, $m\in\{128,256,512,1024\}$, $\eta\in\{10^{-4},10^{-3},10^{-2}\}$, $12$ cells per dataset.

\subsubsection{Training dynamics}\label{app:exp:realworld:dynamics}

Figure~\ref{fig:app:realworld:dynamics} shows the training loss $L(\bfa\theta^{(t)})$ and sliced $\overline{W_1}(\bb P_1,\bb P_{\wh\zb_1^{(t)}})$ for all $12$ cells on both MNIST and Fashion-MNIST. On both datasets the loss decays geometrically (Theorem~\ref{theorem1}) and $\overline{W_1}$ tracks it downward monotonically (Theorem~\ref{thmsampbound}), with curves grouping by $\eta$ rather than $m$.

\begin{figure}[htbp]
  \centering
  \begin{subfigure}[t]{0.48\linewidth}
    \centering
    \includegraphics[width=\linewidth]{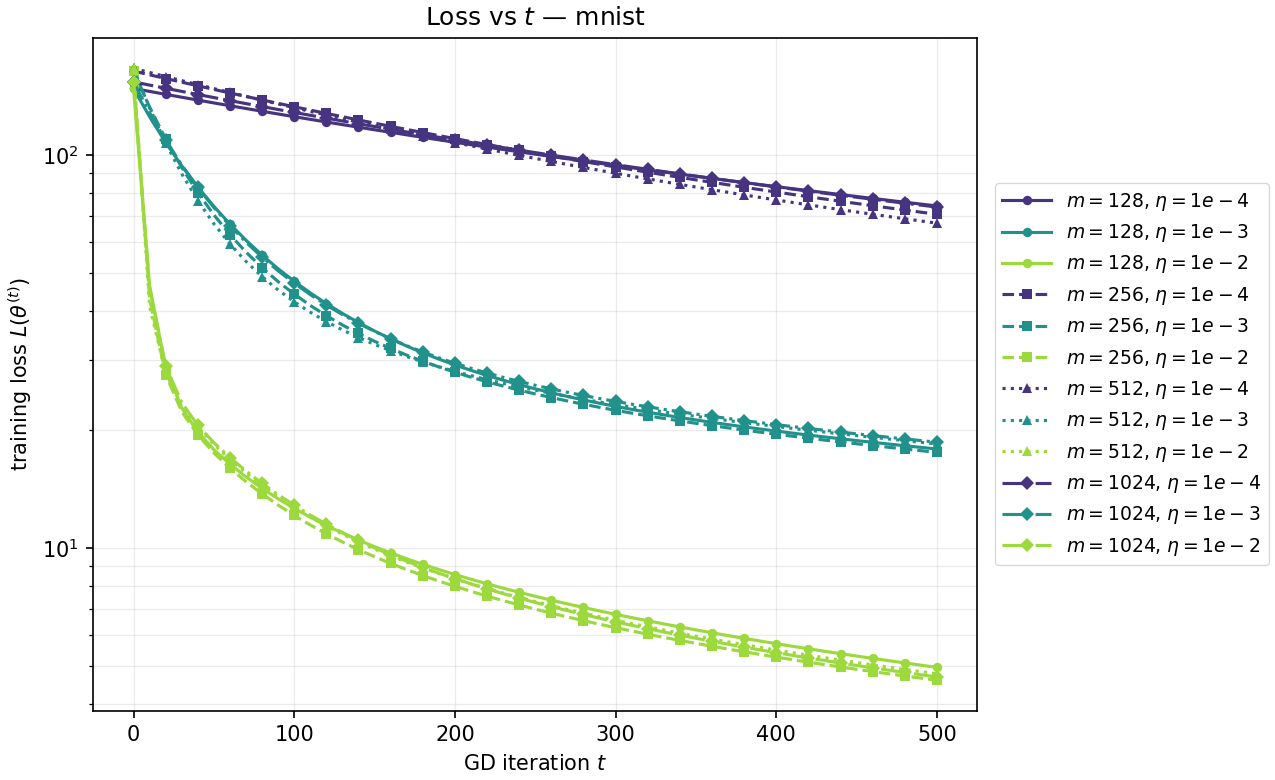}
    \caption{Training loss $L(\bfa\theta^{(t)})$, MNIST.}
    \label{fig:app:realworld:loss_mnist}
  \end{subfigure}
  \hfill
  \begin{subfigure}[t]{0.48\linewidth}
    \centering
    \includegraphics[width=\linewidth]{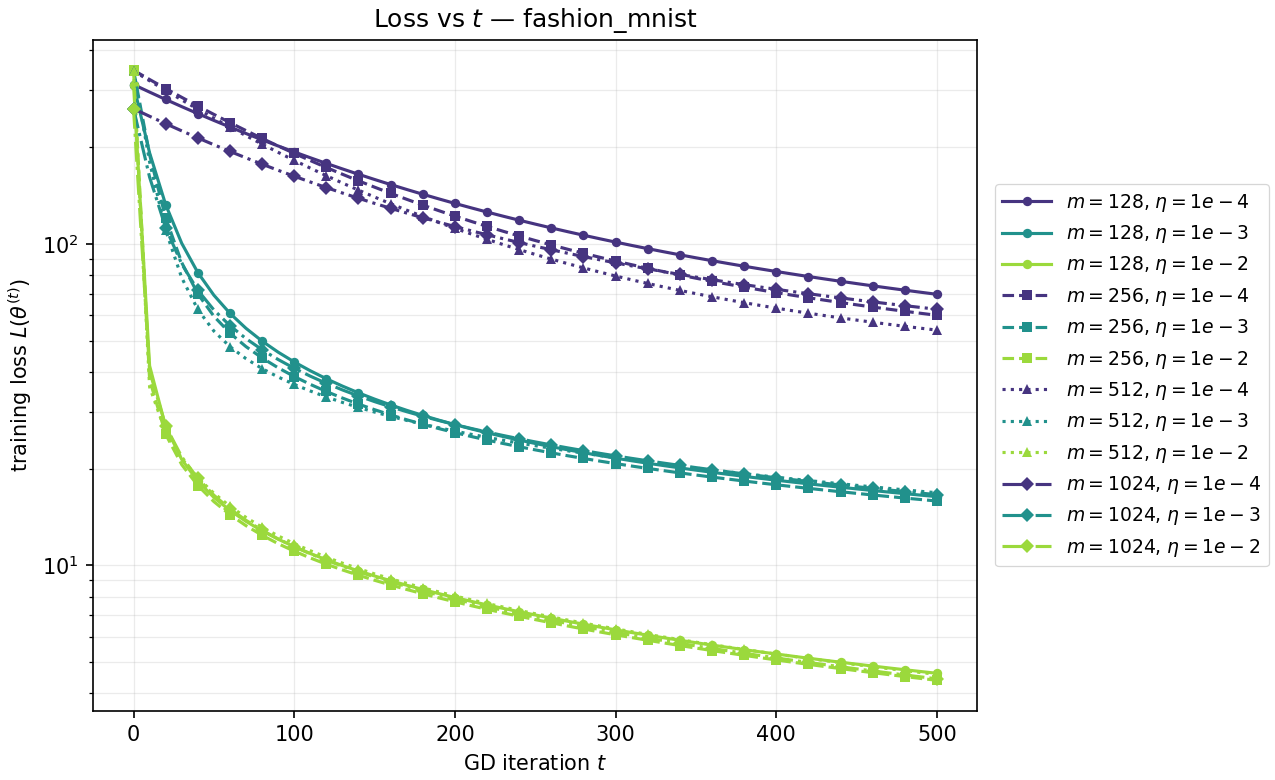}
    \caption{Training loss $L(\bfa\theta^{(t)})$, Fashion-MNIST.}
    \label{fig:app:realworld:loss_fmnist}
  \end{subfigure}\\[0.5em]
  \begin{subfigure}[t]{0.48\linewidth}
    \centering
    \includegraphics[width=\linewidth]{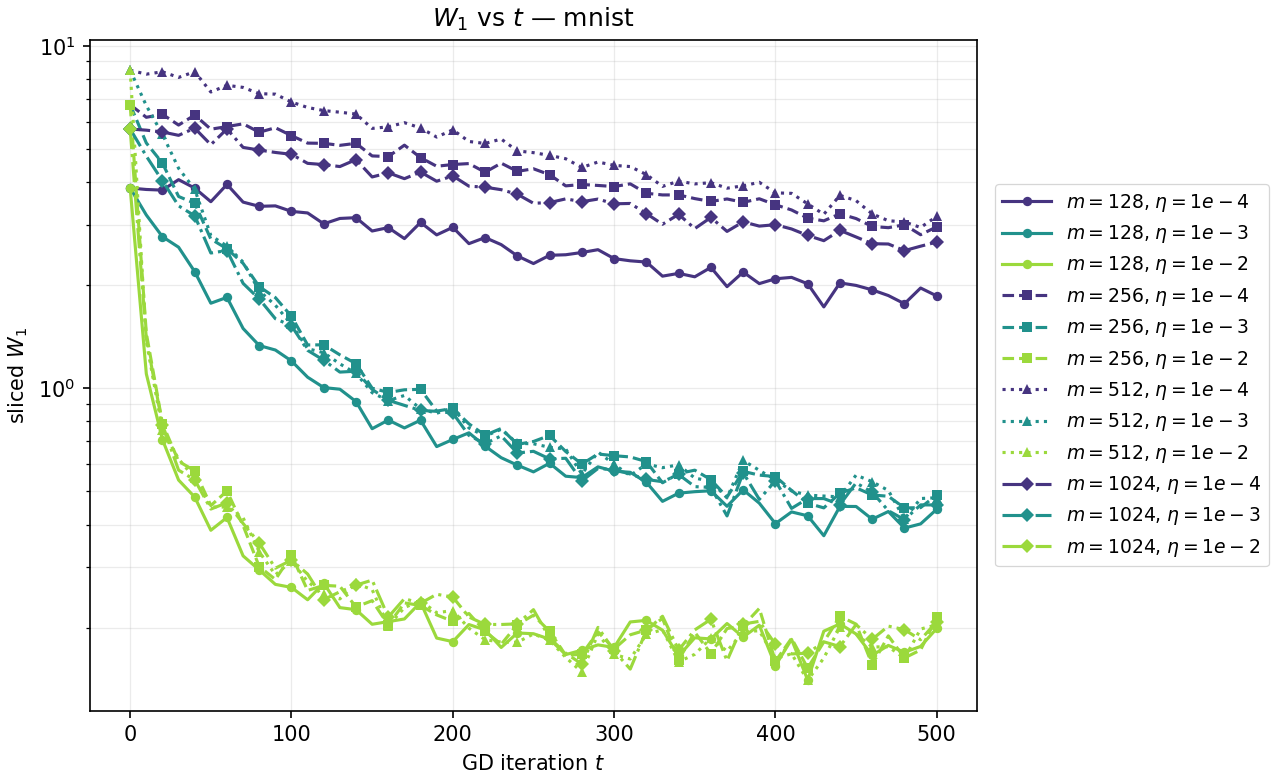}
    \caption{Sliced $\overline{W_1}(\bb P_1,\bb P_{\wh\zb_1^{(t)}})$, MNIST.}
    \label{fig:app:realworld:w1_mnist}
  \end{subfigure}
  \hfill
  \begin{subfigure}[t]{0.48\linewidth}
    \centering
    \includegraphics[width=\linewidth]{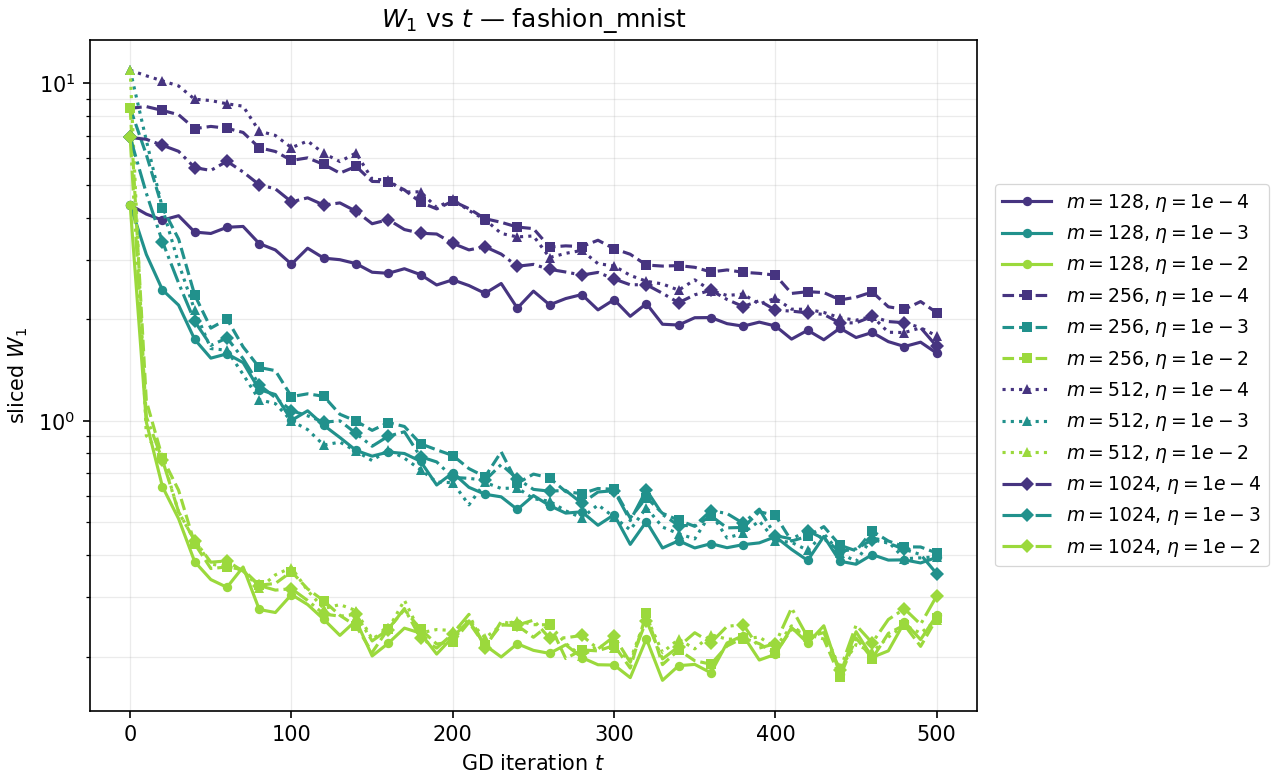}
    \caption{Sliced $\overline{W_1}(\bb P_1,\bb P_{\wh\zb_1^{(t)}})$, Fashion-MNIST.}
    \label{fig:app:realworld:w1_fmnist}
  \end{subfigure}
  \caption{Training dynamics on MNIST (left) and Fashion-MNIST (right), $n_{\rm train}=500$, $n_{\rm test}=2000$, $d=7$ PCA components. Each of the $12$ curves is one $(m,\eta)$ pair; color encodes $\eta$ ($10^{-4}$ darkest, $10^{-2}$ lightest) and linestyle encodes $m\in\{128,256,512,1024\}$. Top: training loss decays geometrically (Theorem~\ref{theorem1}). Bottom: sliced $\overline{W_1}$ drops monotonically (Theorem~\ref{thmsampbound}).}
  \label{fig:app:realworld:dynamics}
\end{figure}

\subsubsection{Generated sample grids}\label{app:exp:realworld:recon}

Figures~\ref{fig:app:realworld:recon:mnist} and~\ref{fig:app:realworld:recon:fmnist} display $8\times8$ grids of images generated by Algorithm~\ref{alg:learning} at the terminal iterate for every $(m,\eta)$ cell on MNIST and Fashion-MNIST, respectively.  Sample quality improves markedly with larger $\eta$: at $\eta=10^{-2}$, generated MNIST images are legible handwritten digits and Fashion-MNIST images show recognizable garment silhouettes; at $\eta=10^{-4}$, MNIST samples are blurry and class-incoherent, while Fashion-MNIST samples retain only the coarsest high-contrast structure.  Across widths $m$ at fixed $\eta$, quality differences are small, consistent with the over-parameterized regime in which loss and $\overline{W_1}$ curves group by $\eta$ rather than $m$ (Figure~\ref{fig:app:realworld:dynamics}).

\begin{figure}[htbp]
  \centering
  \begin{subfigure}[t]{\gridw}
    \includegraphics[width=\linewidth]{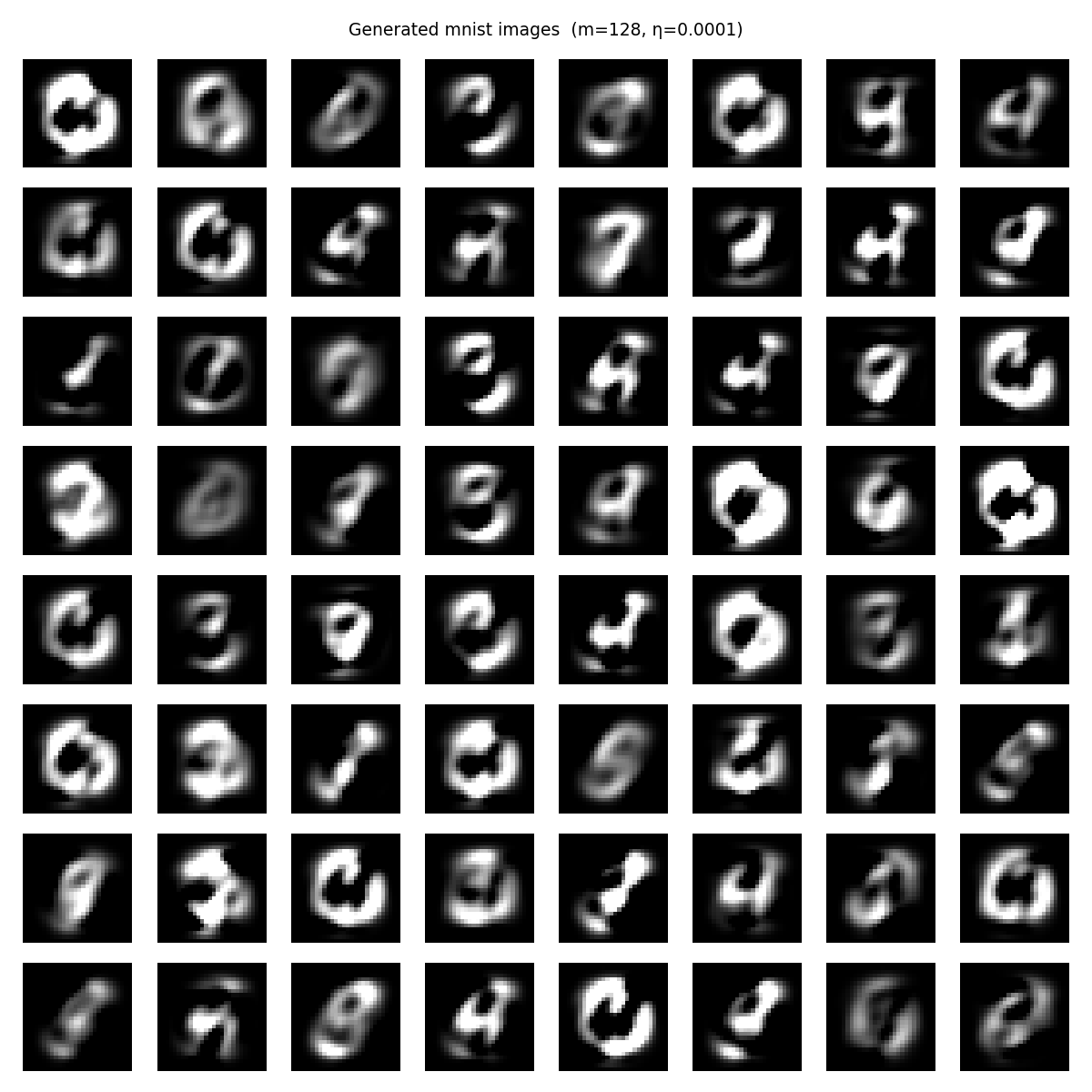}
    \caption{$m=128,\eta=10^{-4}$.}
  \end{subfigure}
  \begin{subfigure}[t]{\gridw}
    \includegraphics[width=\linewidth]{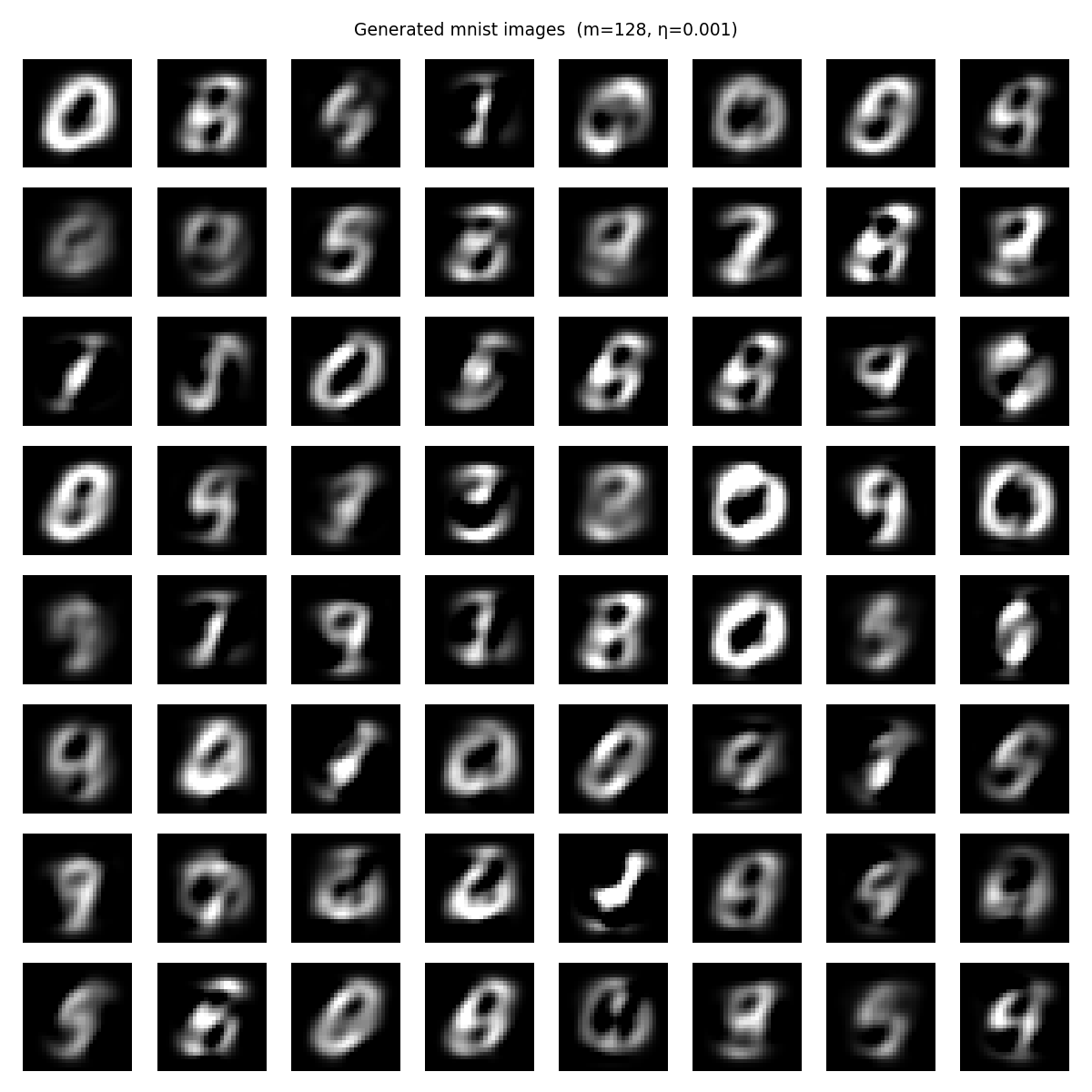}
    \caption{$m=128,\eta=10^{-3}$.}
  \end{subfigure}
  \begin{subfigure}[t]{\gridw}
    \includegraphics[width=\linewidth]{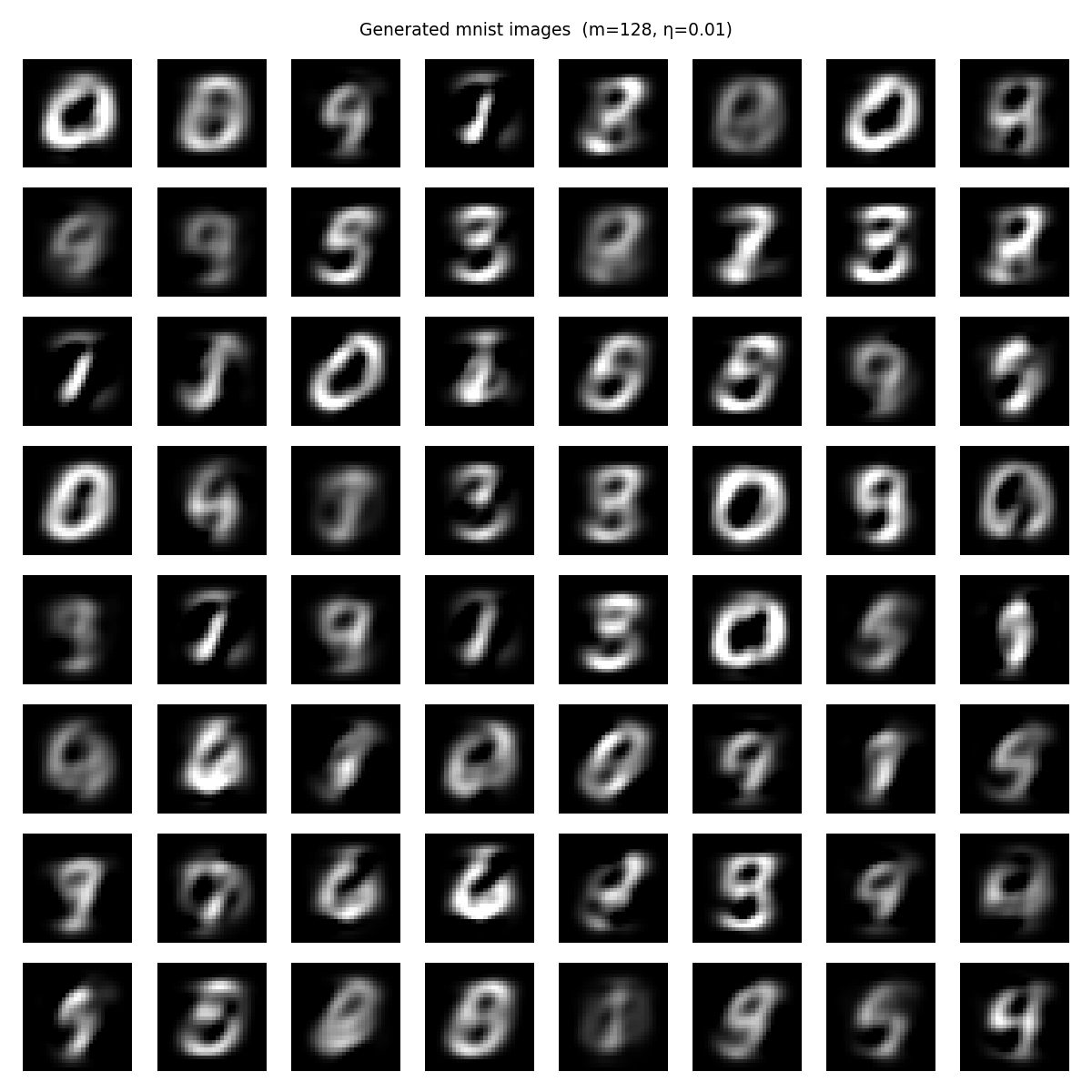}
    \caption{$m=128,\eta=10^{-2}$.}
  \end{subfigure}\\[0.5em]
  \begin{subfigure}[t]{\gridw}
    \includegraphics[width=\linewidth]{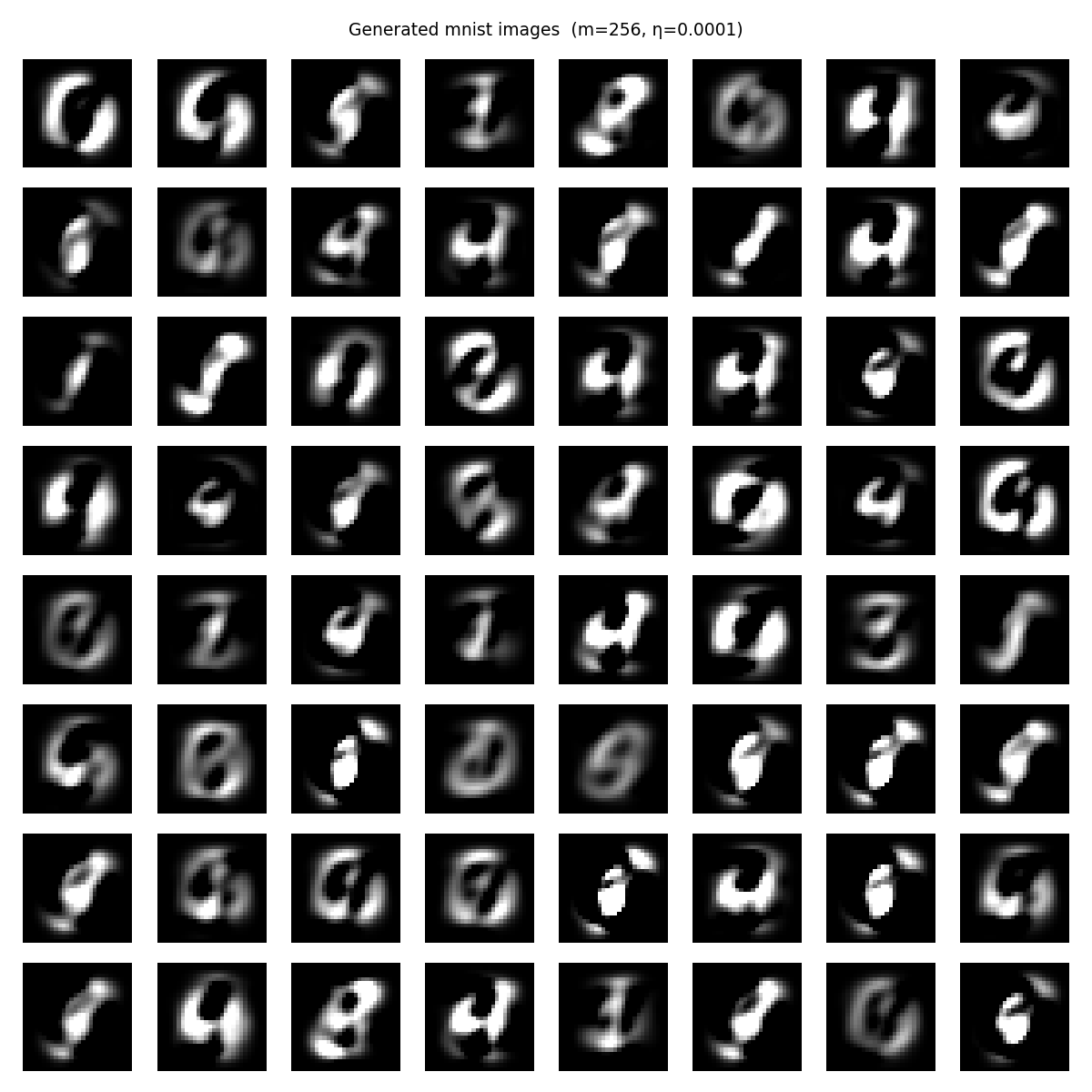}
    \caption{$m=256,\eta=10^{-4}$.}
  \end{subfigure}
  \begin{subfigure}[t]{\gridw}
    \includegraphics[width=\linewidth]{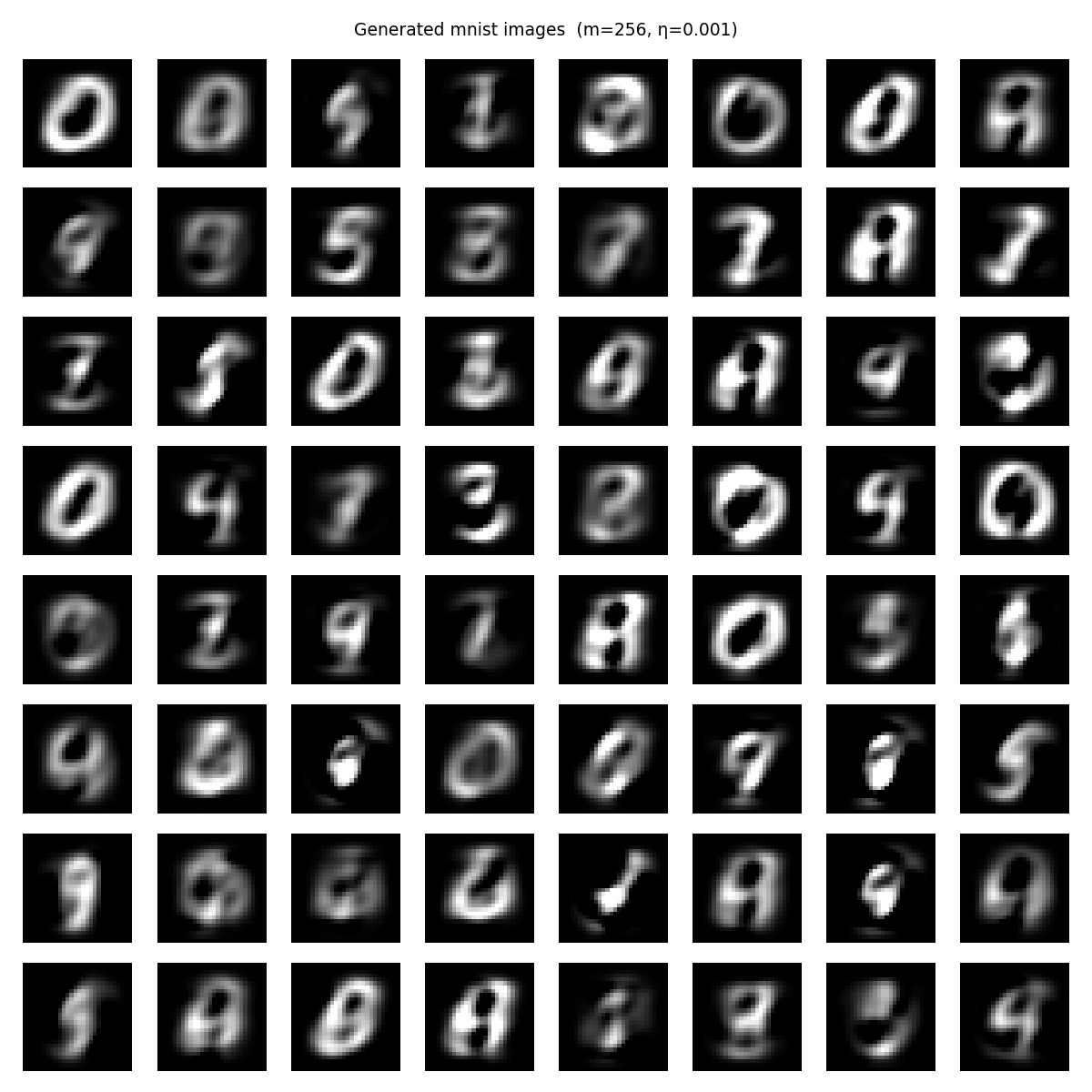}
    \caption{$m=256,\eta=10^{-3}$.}
  \end{subfigure}
  \begin{subfigure}[t]{\gridw}
    \includegraphics[width=\linewidth]{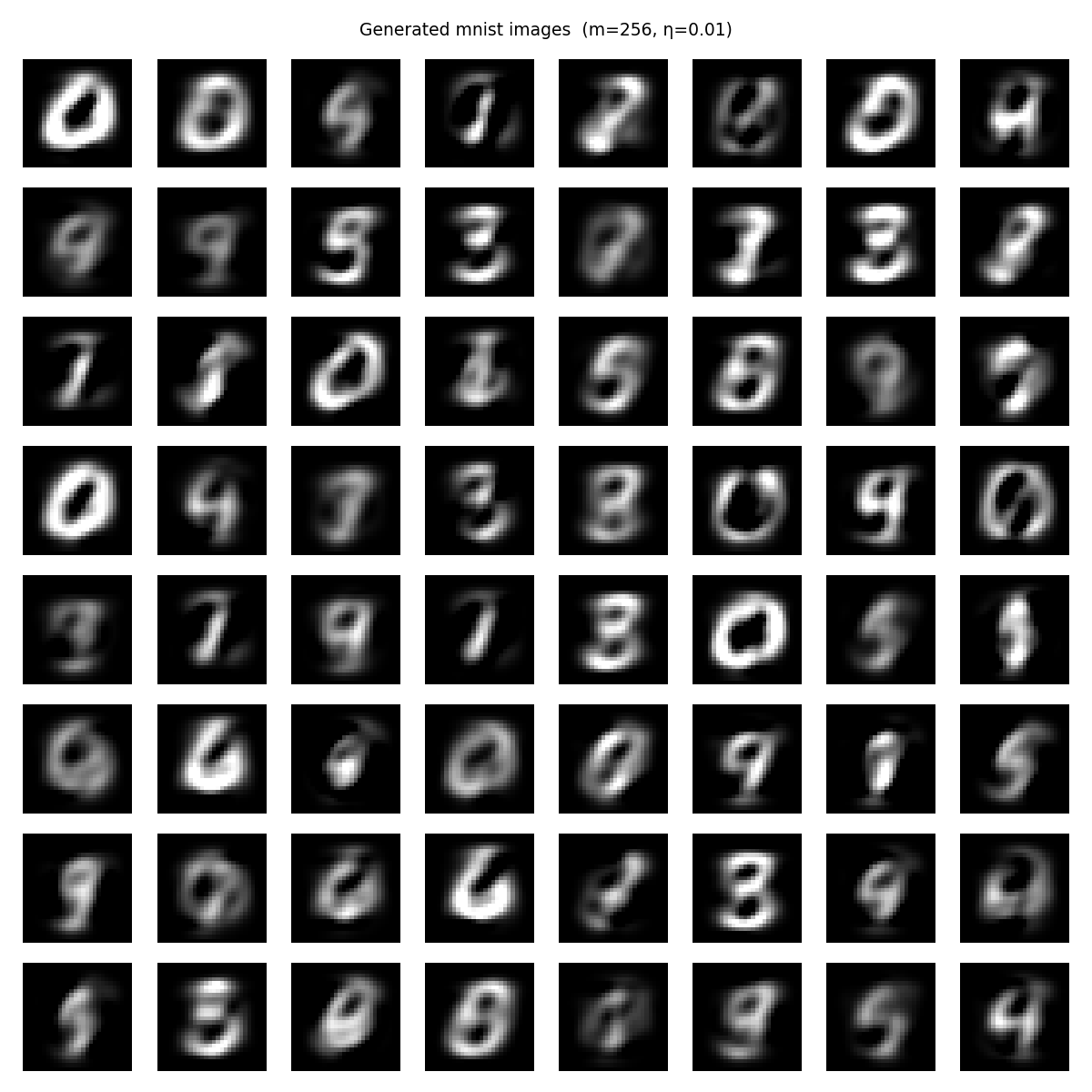}
    \caption{$m=256,\eta=10^{-2}$.}
  \end{subfigure}\\[0.5em]
  \begin{subfigure}[t]{\gridw}
    \includegraphics[width=\linewidth]{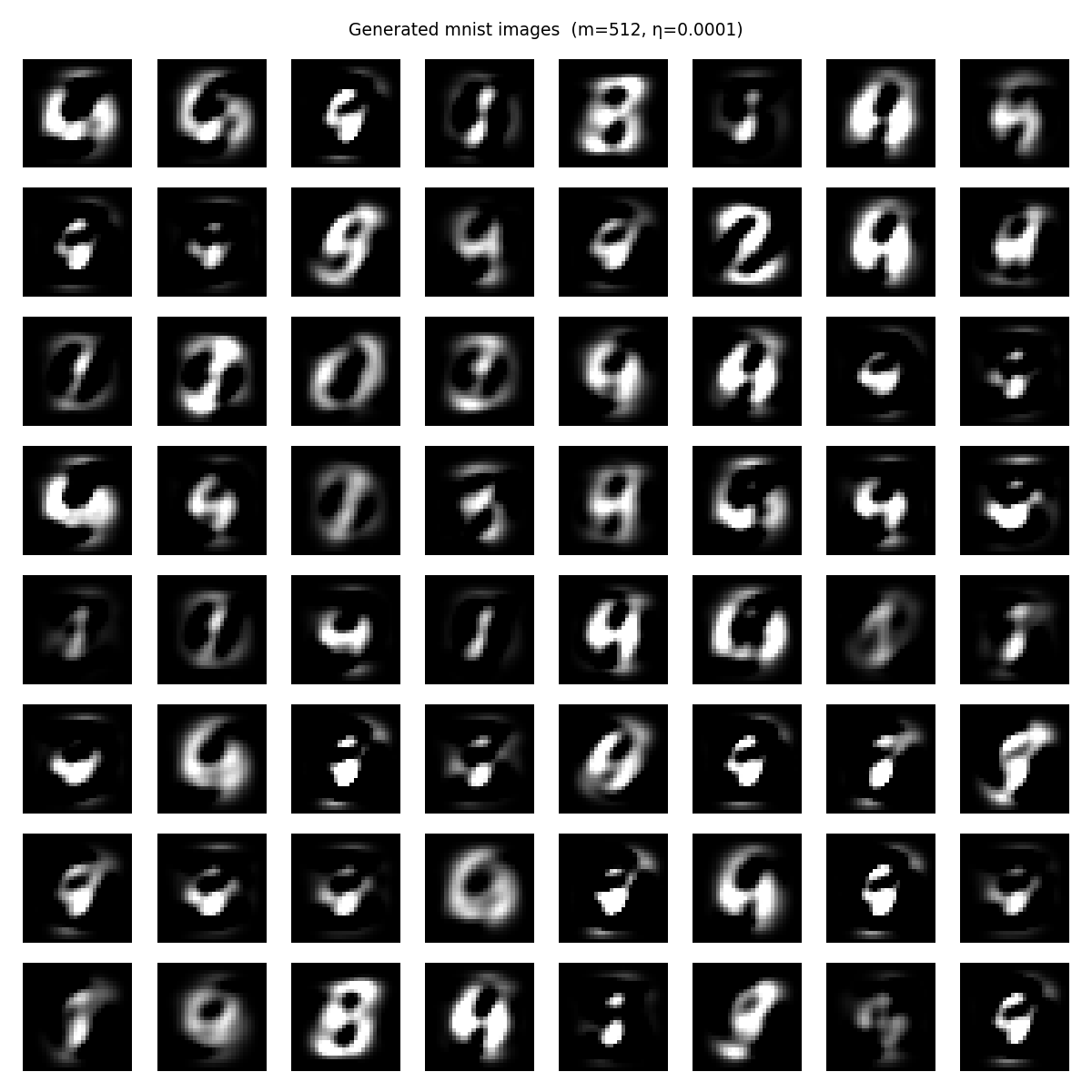}
    \caption{$m=512,\eta=10^{-4}$.}
  \end{subfigure}
  \begin{subfigure}[t]{\gridw}
    \includegraphics[width=\linewidth]{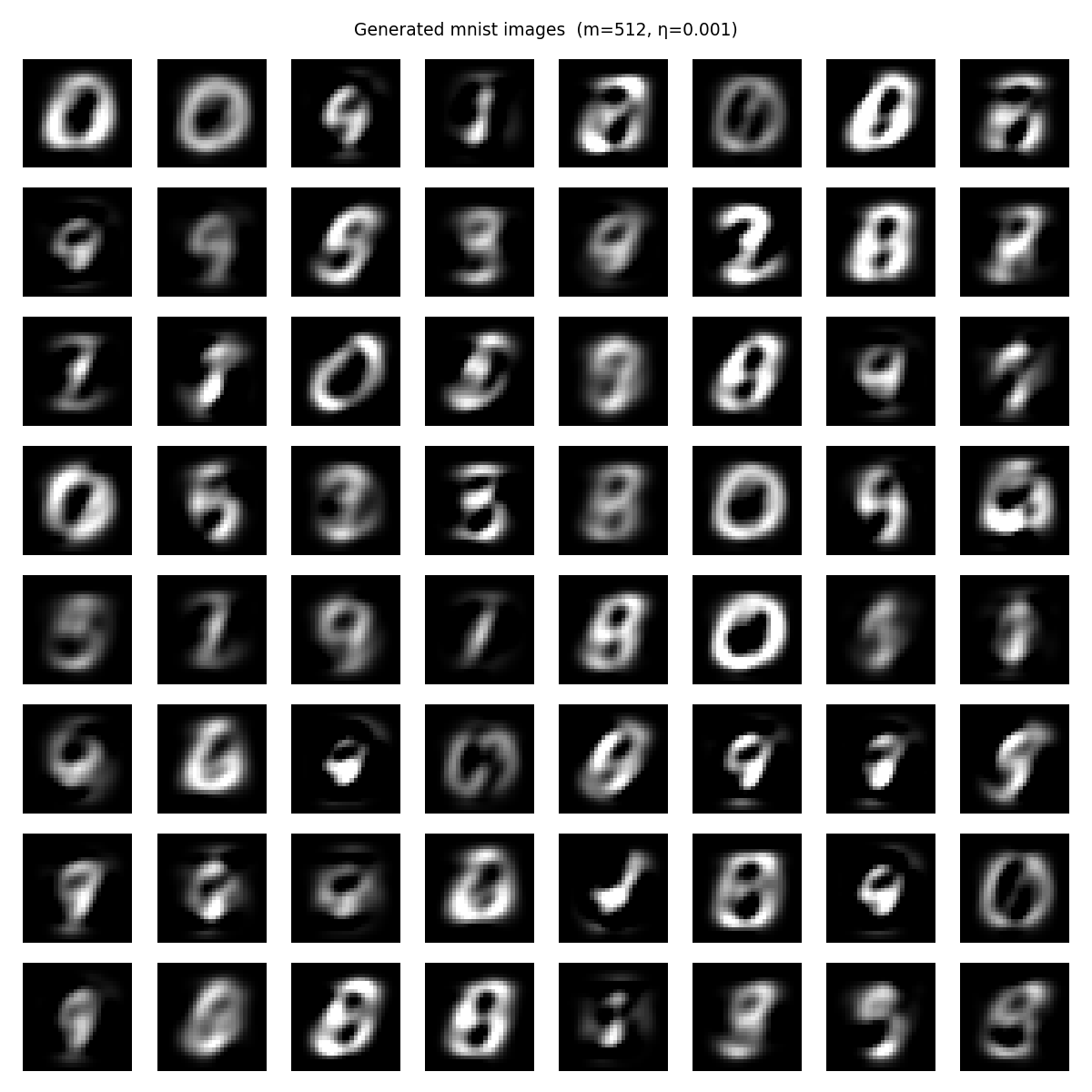}
    \caption{$m=512,\eta=10^{-3}$.}
  \end{subfigure}
  \begin{subfigure}[t]{\gridw}
    \includegraphics[width=\linewidth]{mnist/reconstruction/m512_eta1e-2_ntr500_nte2000.png}
    \caption{$m=512,\eta=10^{-2}$.}
  \end{subfigure}\\[0.5em]
  \begin{subfigure}[t]{\gridw}
    \includegraphics[width=\linewidth]{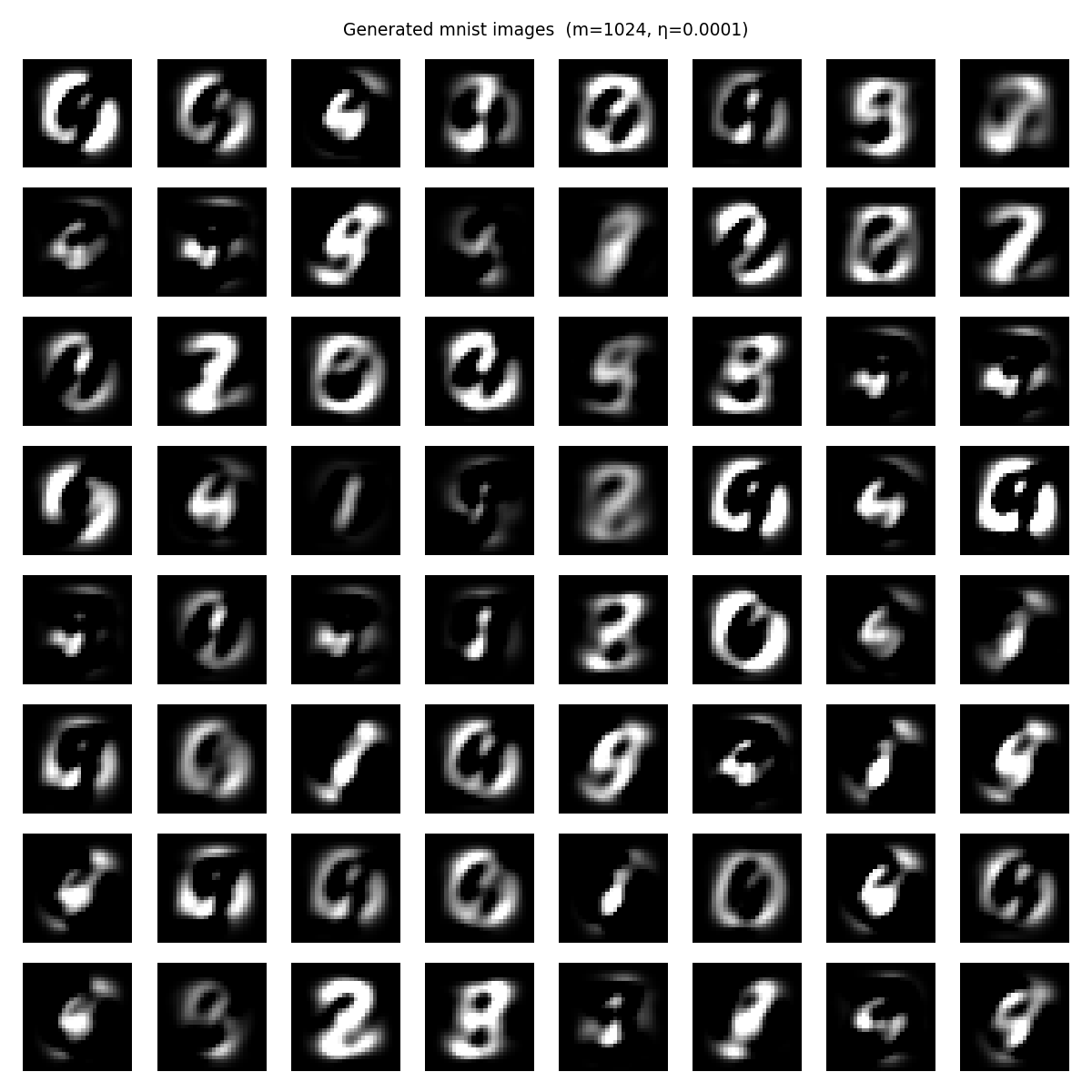}
    \caption{$m=1024,\eta=10^{-4}$.}
  \end{subfigure}
  \begin{subfigure}[t]{\gridw}
    \includegraphics[width=\linewidth]{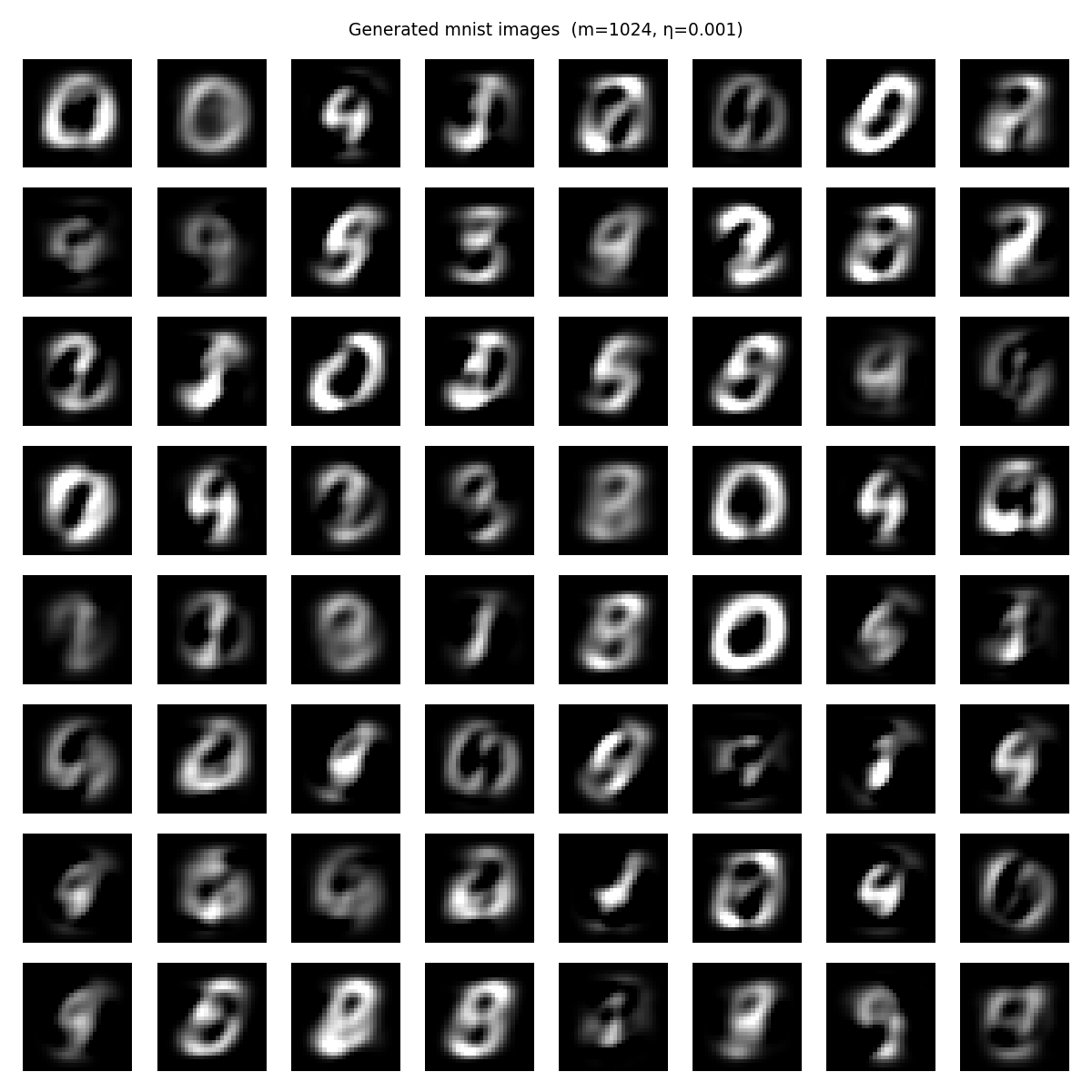}
    \caption{$m=1024,\eta=10^{-3}$.}
  \end{subfigure}
  \begin{subfigure}[t]{\gridw}
    \includegraphics[width=\linewidth]{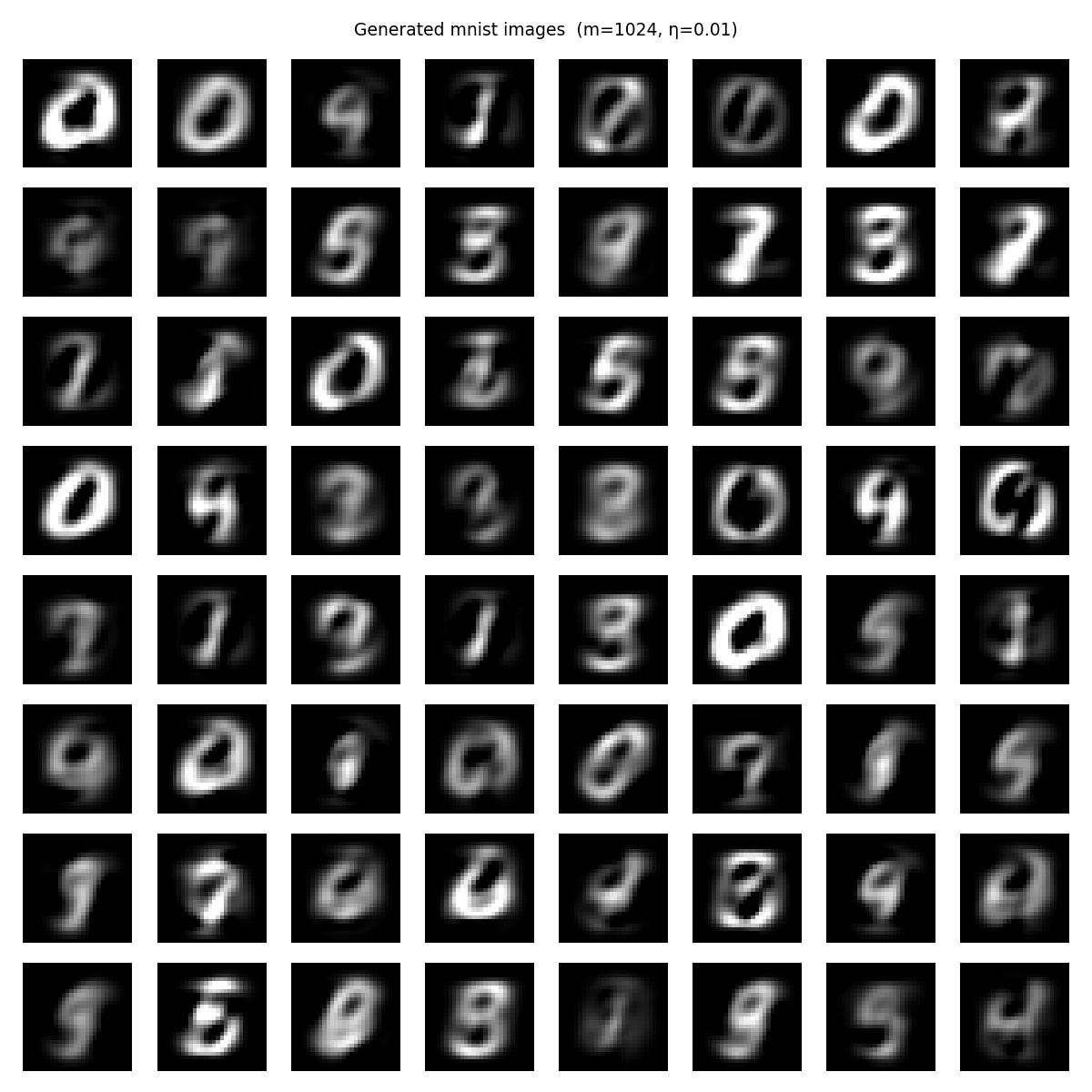}
    \caption{$m=1024,\eta=10^{-2}$.}
  \end{subfigure}
  \caption{$8\times8$ grids of images generated by Algorithm~\ref{alg:learning} at the terminal iterate $t=T=500$ for every $(m,\eta)$ cell on MNIST, $n_{\rm train}=500$.  Rows correspond to widths $m\in\{128,256,512,1024\}$; columns to step sizes $\eta\in\{10^{-4},10^{-3},10^{-2}\}$.  Sample quality improves with $\eta$; the effect of $m$ at fixed $\eta$ is small.}
  \label{fig:app:realworld:recon:mnist}
\end{figure}

\begin{figure}[htbp]
  \centering
  \begin{subfigure}[t]{\gridw}
    \includegraphics[width=\linewidth]{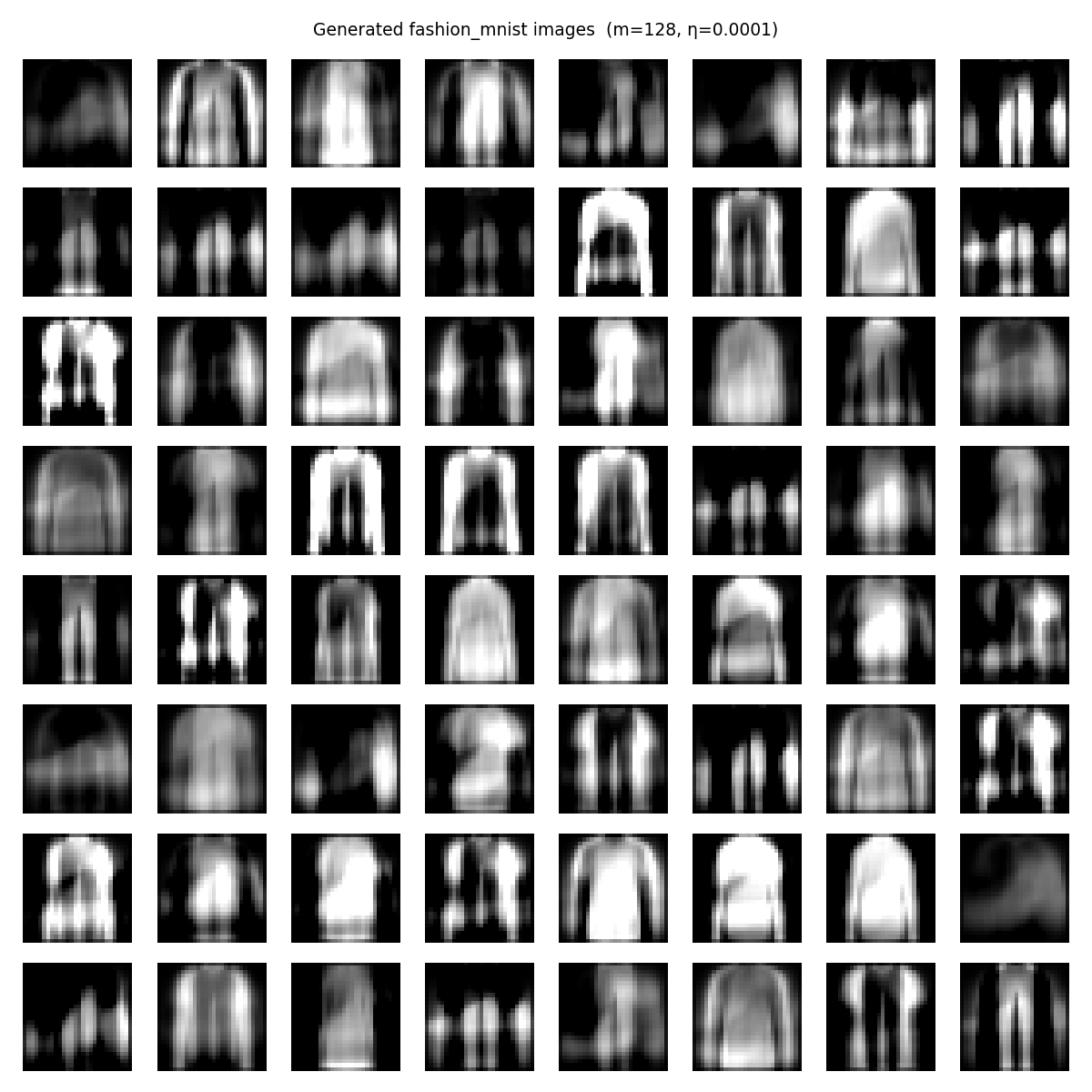}
    \caption{$m=128,\eta=10^{-4}$.}
  \end{subfigure}
  \begin{subfigure}[t]{\gridw}
    \includegraphics[width=\linewidth]{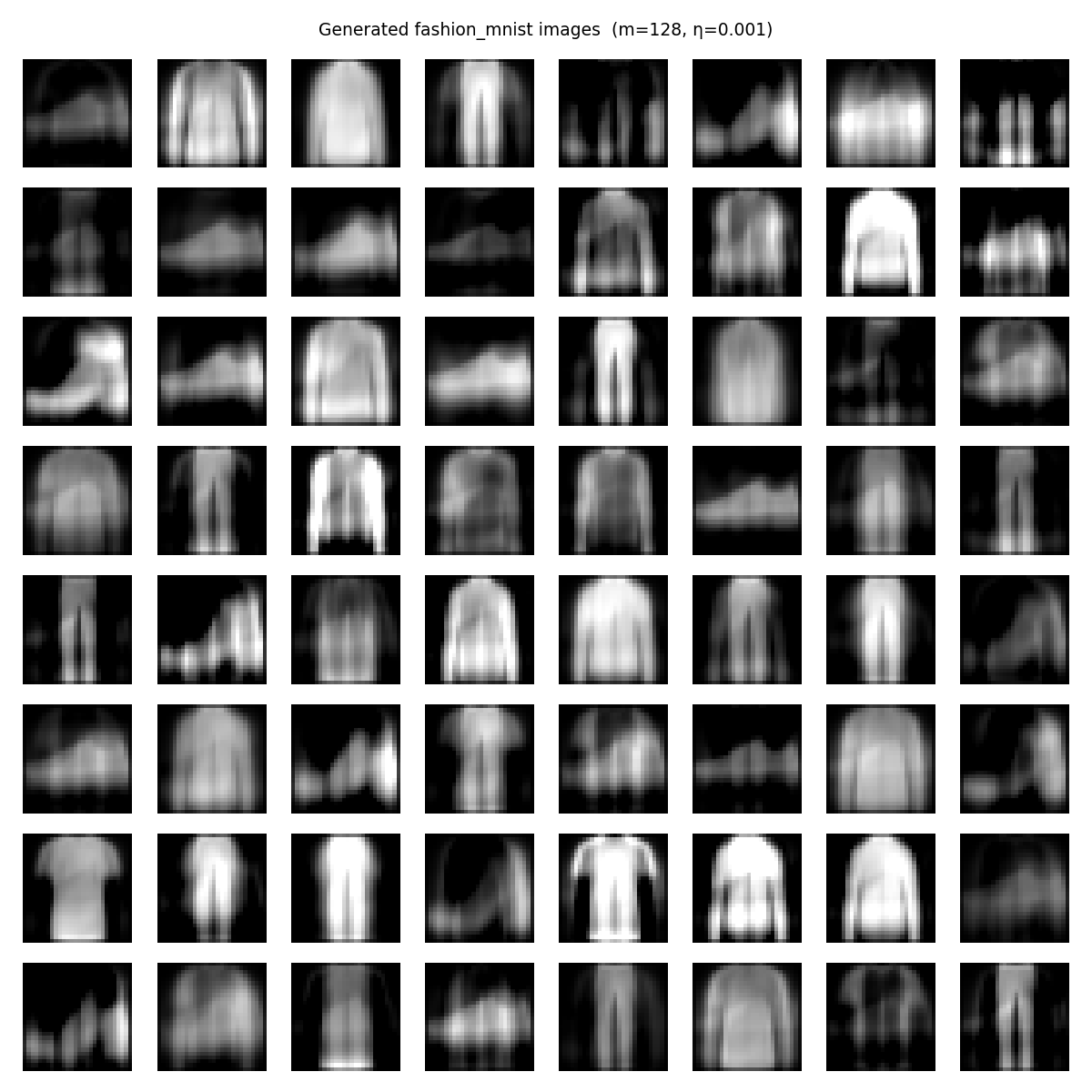}
    \caption{$m=128,\eta=10^{-3}$.}
  \end{subfigure}
  \begin{subfigure}[t]{\gridw}
    \includegraphics[width=\linewidth]{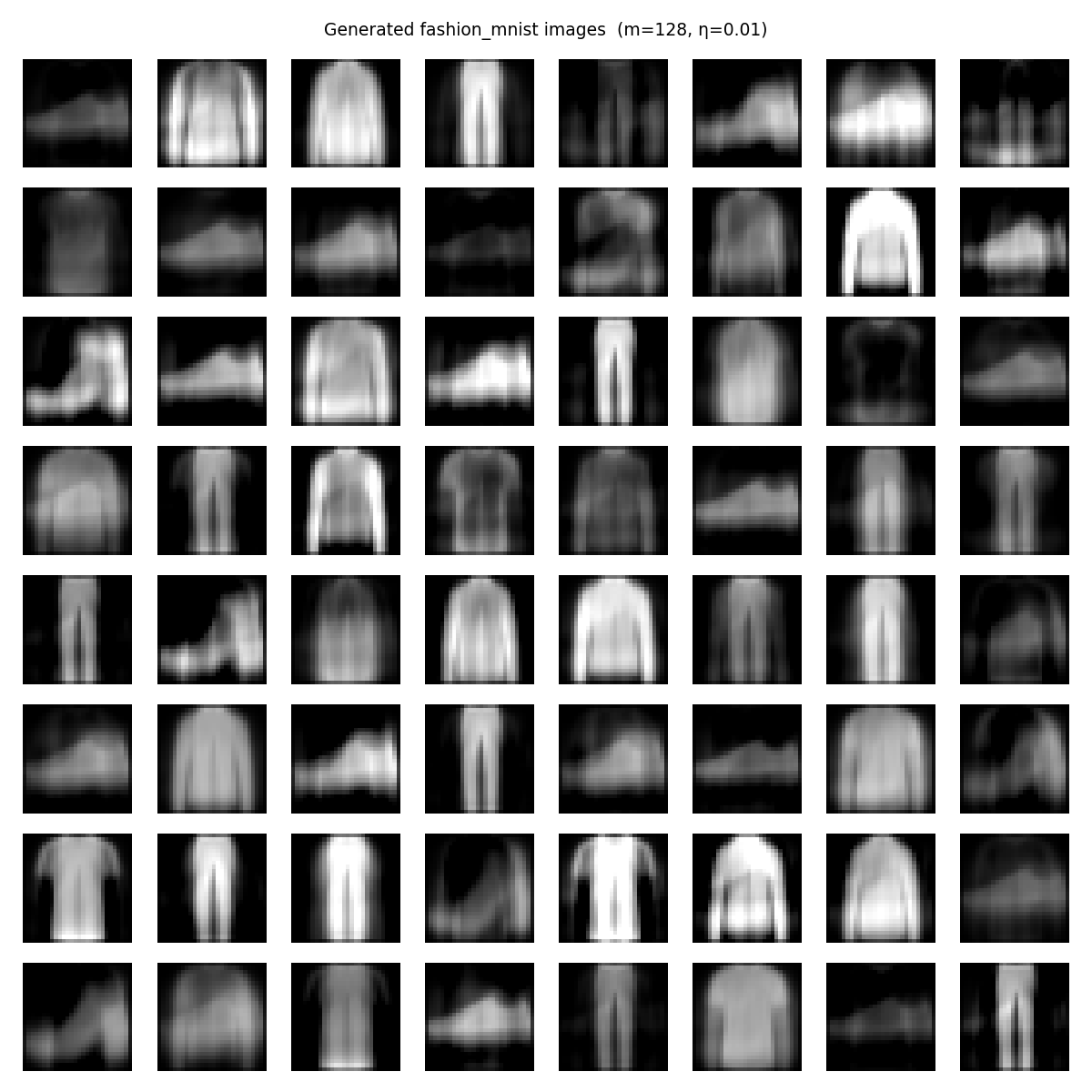}
    \caption{$m=128,\eta=10^{-2}$.}
  \end{subfigure}\\[0.5em]
  \begin{subfigure}[t]{\gridw}
    \includegraphics[width=\linewidth]{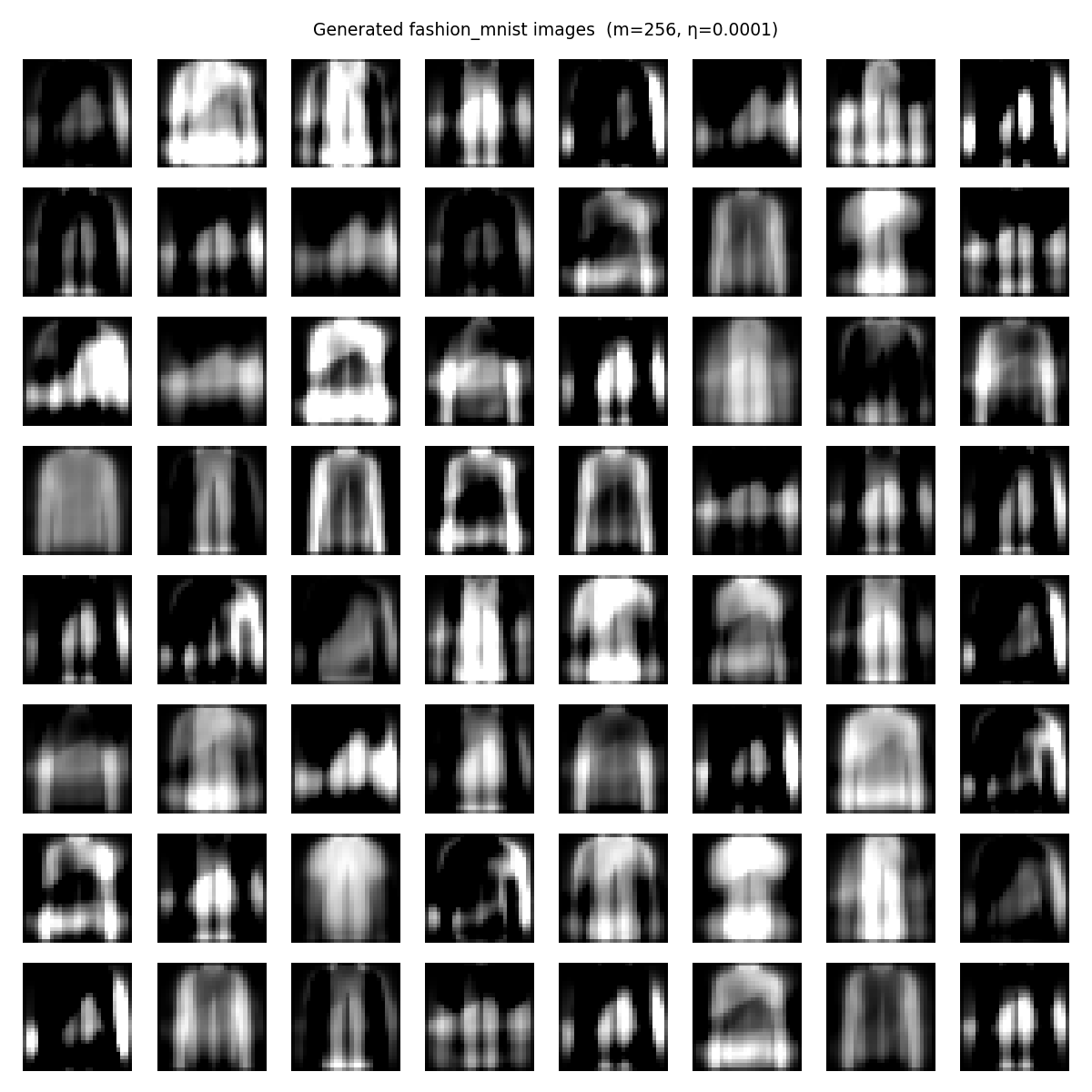}
    \caption{$m=256,\eta=10^{-4}$.}
  \end{subfigure}
  \begin{subfigure}[t]{\gridw}
    \includegraphics[width=\linewidth]{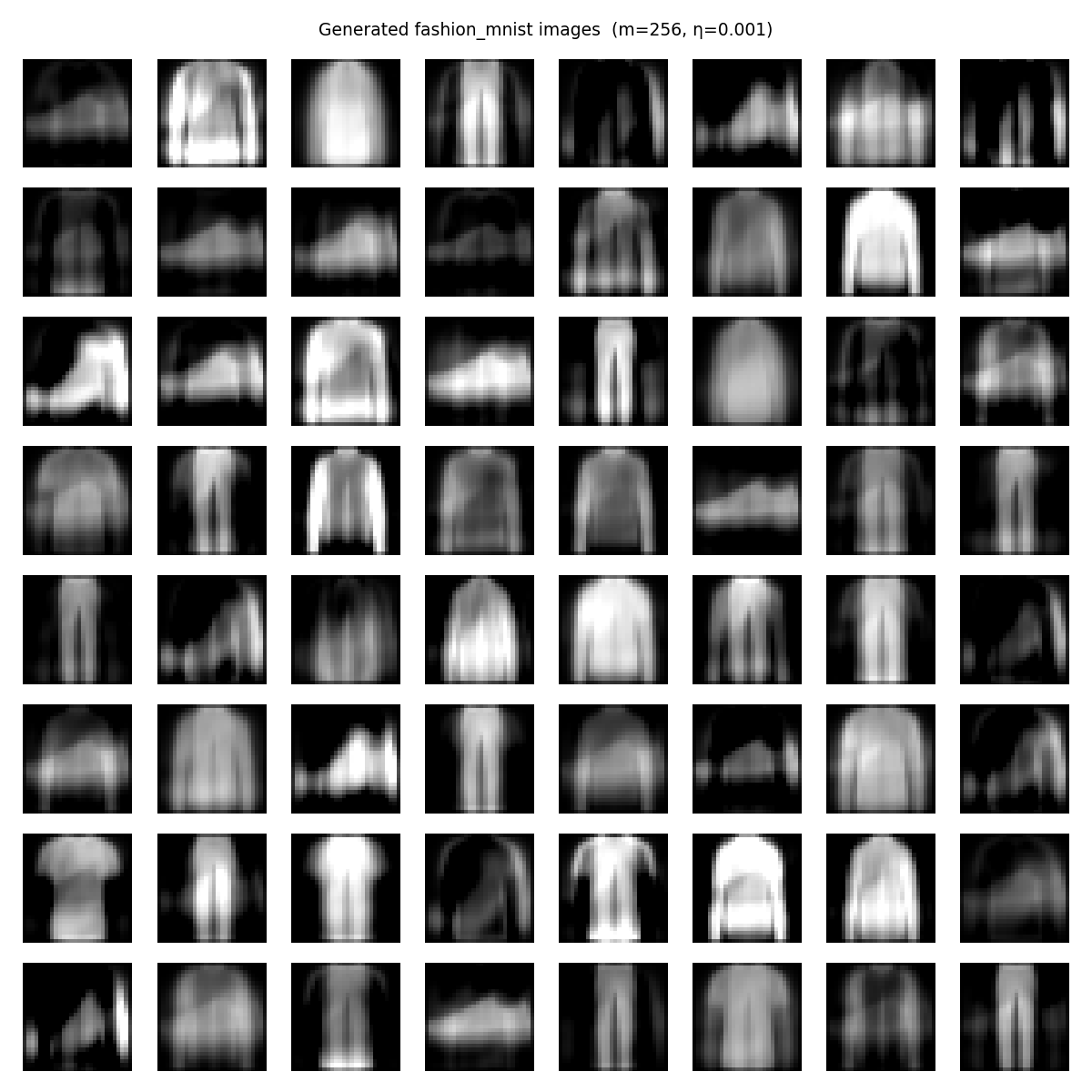}
    \caption{$m=256,\eta=10^{-3}$.}
  \end{subfigure}
  \begin{subfigure}[t]{\gridw}
    \includegraphics[width=\linewidth]{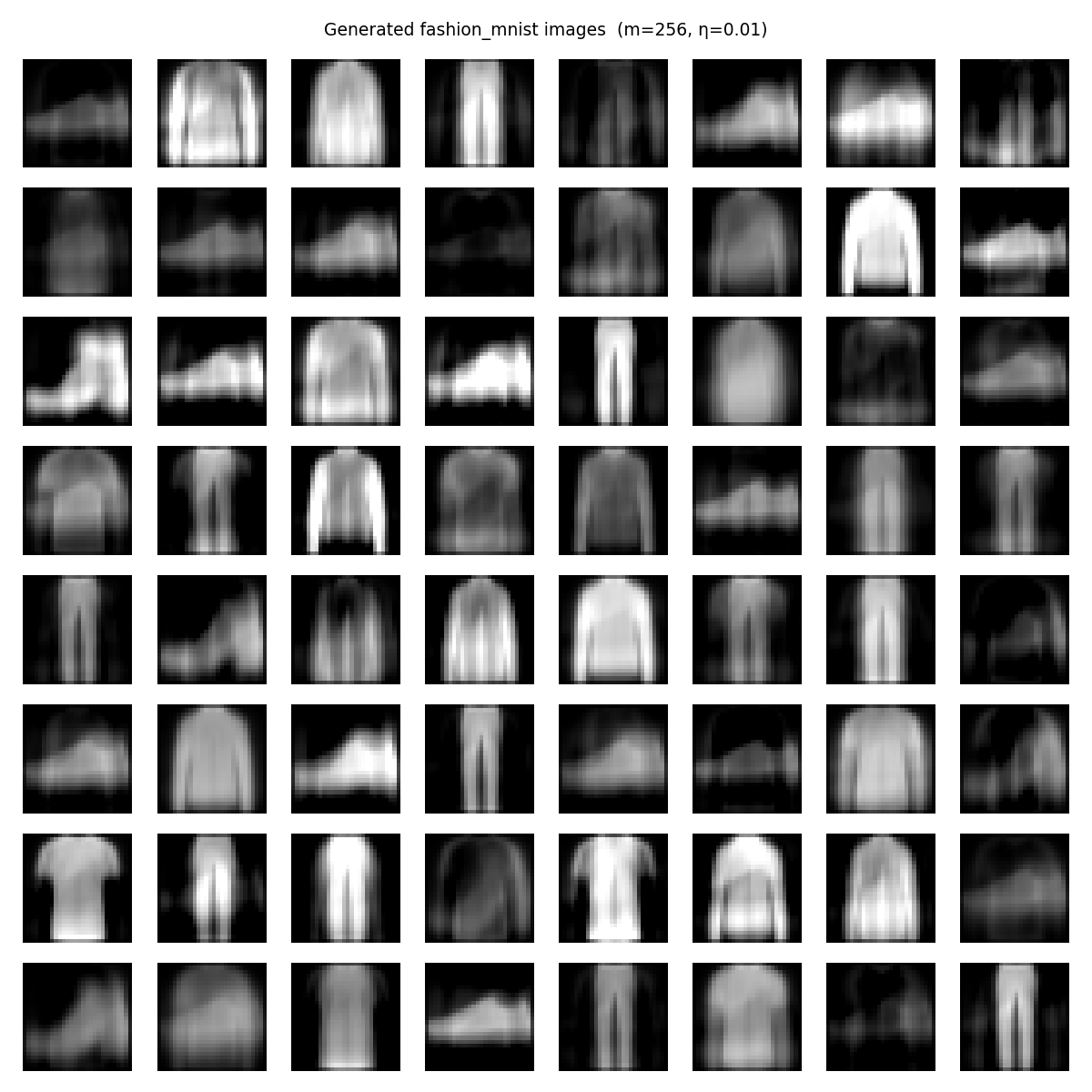}
    \caption{$m=256,\eta=10^{-2}$.}
  \end{subfigure}\\[0.5em]
  \begin{subfigure}[t]{\gridw}
    \includegraphics[width=\linewidth]{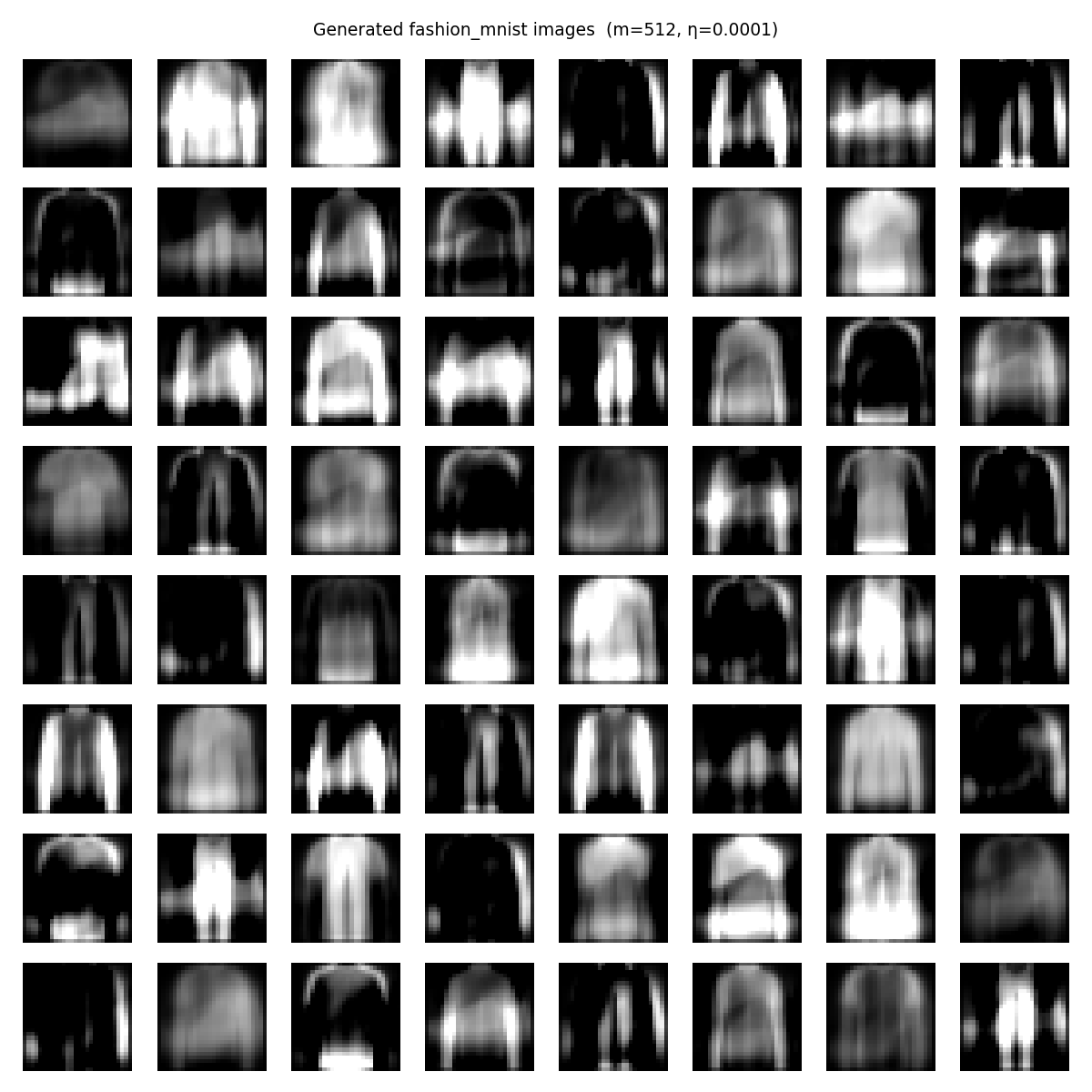}
    \caption{$m=512,\eta=10^{-4}$.}
  \end{subfigure}
  \begin{subfigure}[t]{\gridw}
    \includegraphics[width=\linewidth]{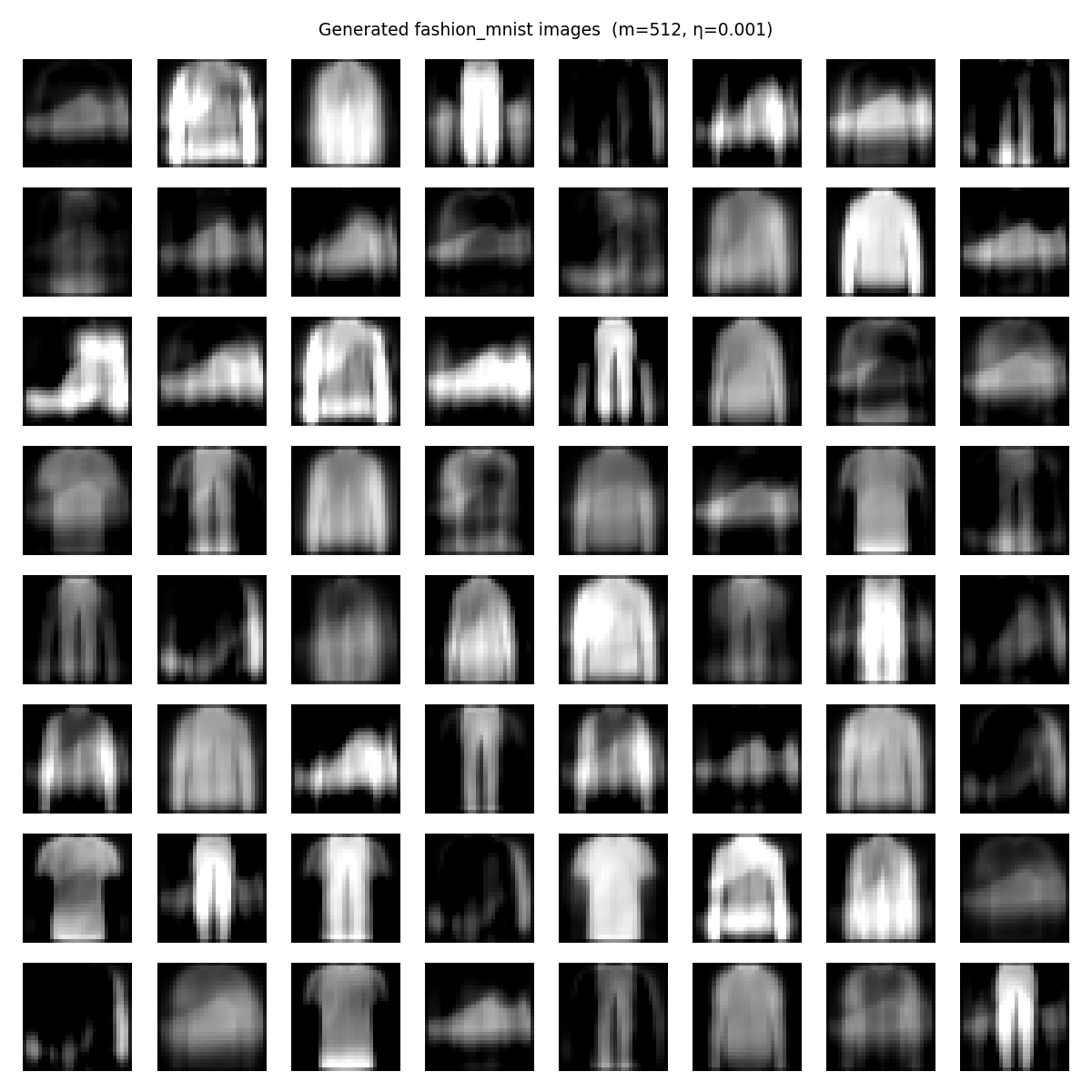}
    \caption{$m=512,\eta=10^{-3}$.}
  \end{subfigure}
  \begin{subfigure}[t]{\gridw}
    \includegraphics[width=\linewidth]{fashion_mnist/reconstruction/m512_eta1e-2_ntr500_nte2000.png}
    \caption{$m=512,\eta=10^{-2}$.}
  \end{subfigure}\\[0.5em]
  \begin{subfigure}[t]{\gridw}
    \includegraphics[width=\linewidth]{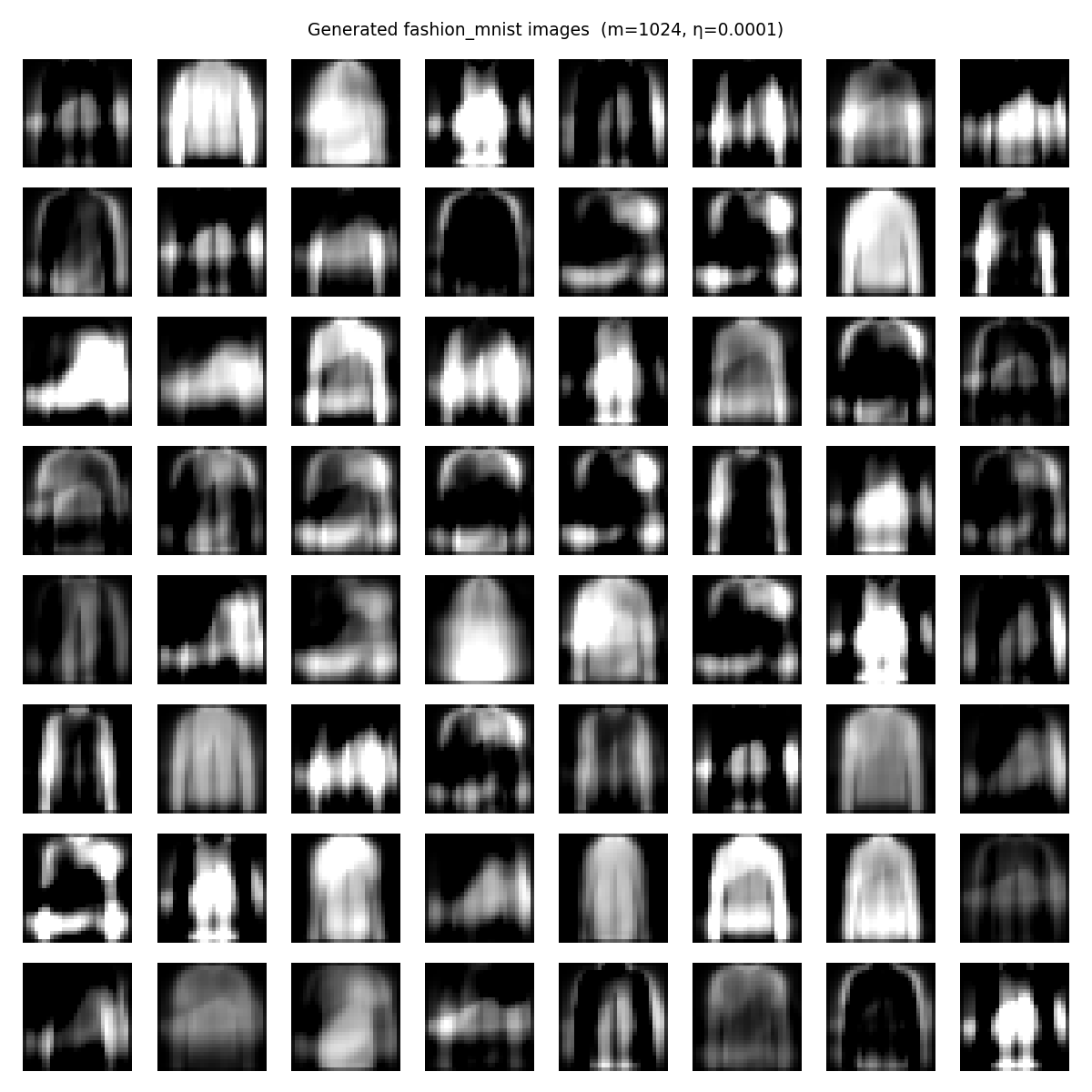}
    \caption{$m=1024,\eta=10^{-4}$.}
  \end{subfigure}
  \begin{subfigure}[t]{\gridw}
    \includegraphics[width=\linewidth]{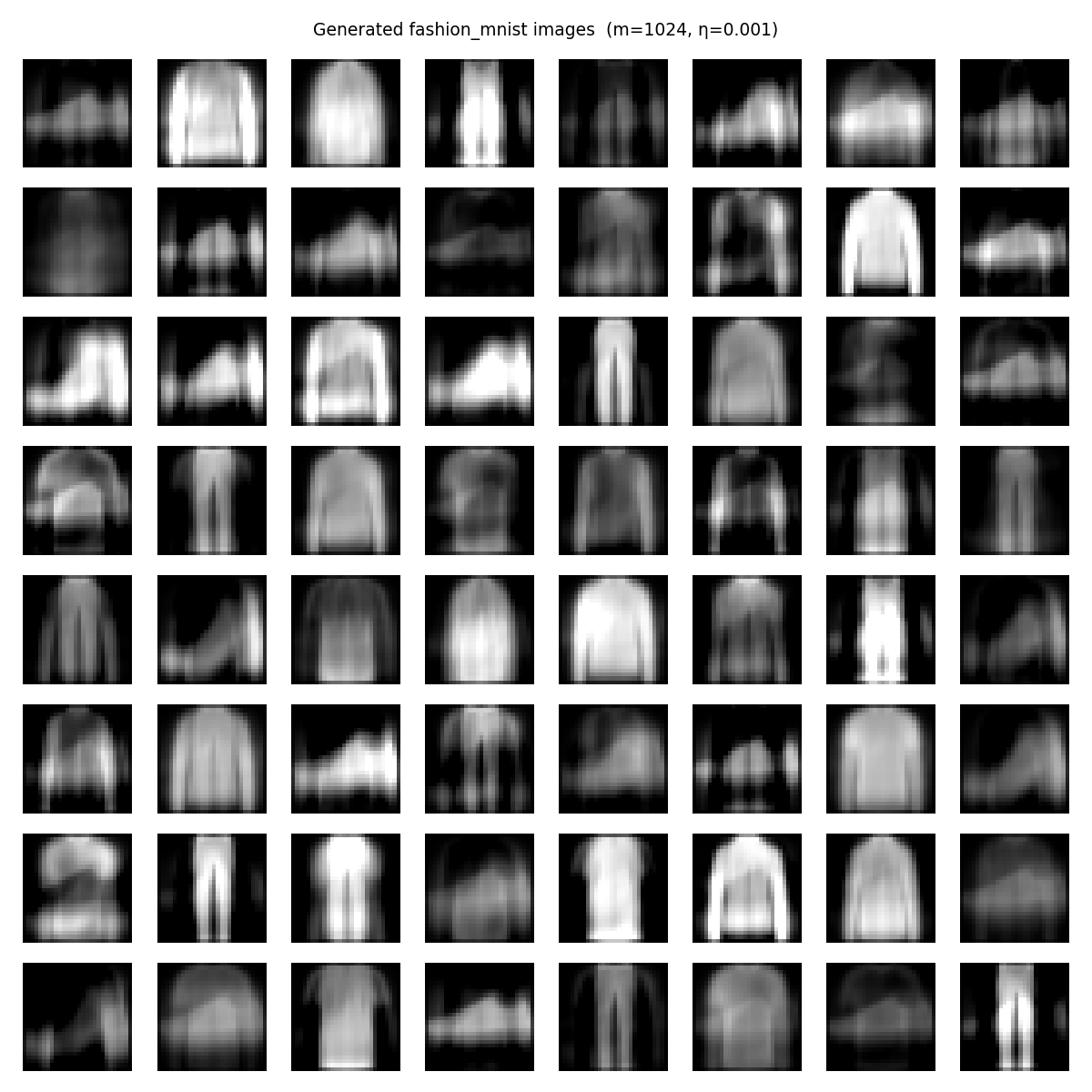}
    \caption{$m=1024,\eta=10^{-3}$.}
  \end{subfigure}
  \begin{subfigure}[t]{\gridw}
    \includegraphics[width=\linewidth]{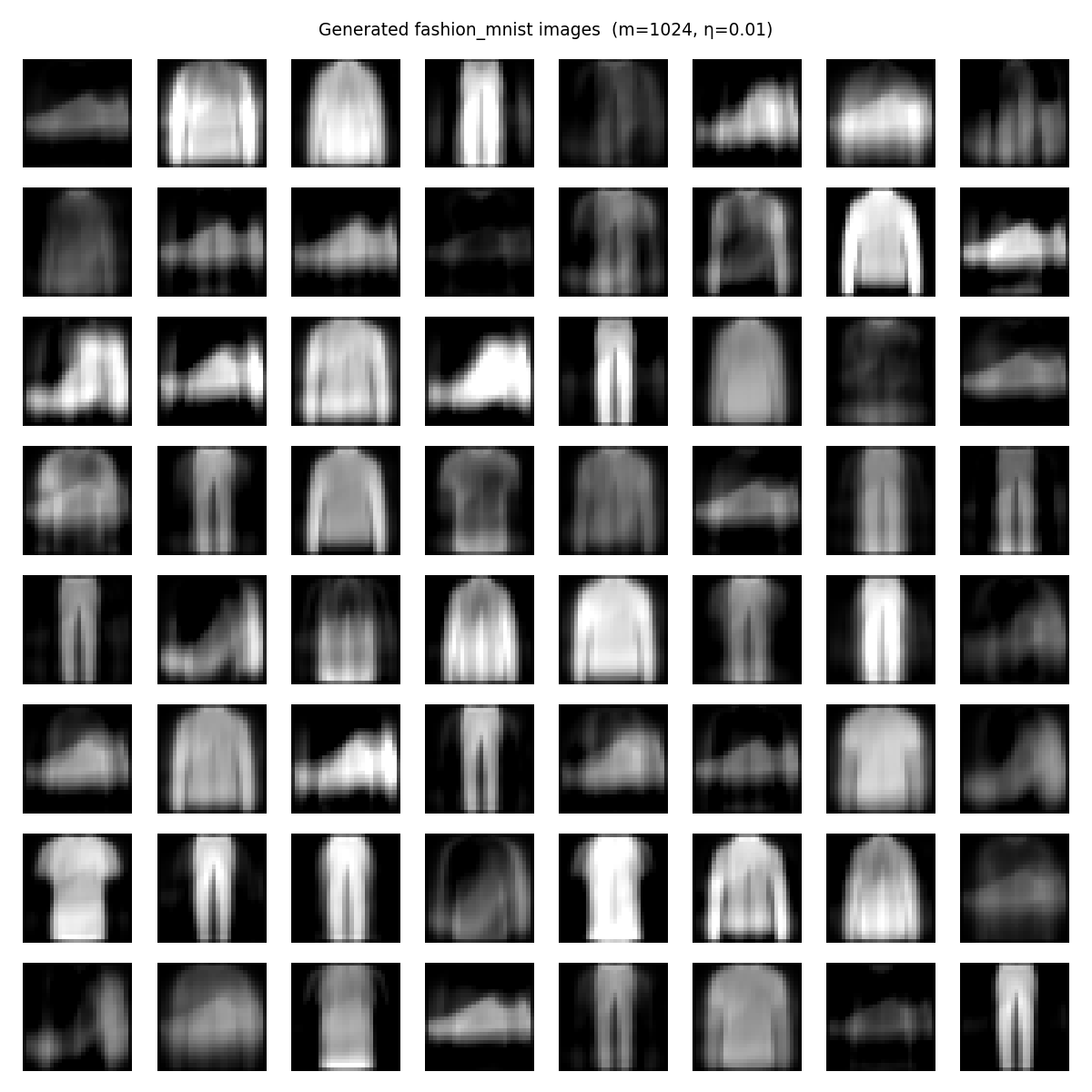}
    \caption{$m=1024,\eta=10^{-2}$.}
  \end{subfigure}
  \caption{$8\times8$ grids of images generated by Algorithm~\ref{alg:learning} at the terminal iterate $t=T=500$ for every $(m,\eta)$ cell on Fashion-MNIST, $n_{\rm train}=500$.  Layout identical to Figure~\ref{fig:app:realworld:recon:mnist}.  At $\eta=10^{-2}$ garment categories are distinguishable across all widths; at $\eta=10^{-4}$ only coarse high-contrast structure is captured.}
  \label{fig:app:realworld:recon:fmnist}
\end{figure}

\clearpage
\section{Additional experimental results}\label{app:additional-results}

\subsection{NTK spectrum and contraction}
The estimated smallest NTK eigenvalue was positive over five samples in every setting:
\begin{table}[H]\centering
\caption{ Estimated NTK eigenvalues; standard deviations are in parentheses.}
\label{tab:ntk-spectrum}
\begin{tabular}{rrrr}\toprule
$d$&$n$&$\widehat\lambda_{\min}$ (sd)&$\widehat\lambda_{\max}$\\\midrule
5&100&0.322 (0.079)&320\\
5&500&0.126 (0.044)&1445\\
10&100&1.395 (0.296)&311\\
10&500&0.652 (0.067)&1467\\
20&100&6.502 (0.458)&343\\
20&500&3.015 (0.399)&1506\\
\bottomrule\end{tabular}\end{table}

We also trained widths $m=256,\ldots,16384$ with $\eta=c/\widehat\lambda_{\max}$, $c\in\{0.25,0.5,1\}$. The contraction conclusion of Theorem~\ref{theorem1} held in all 630 runs; the largest seed-averaged per-run maximum ratio to the contraction bound was $0.947<1$. Observed contraction was faster than the sufficient bound and largely insensitive to width over this $64\times$ range.

These finite-width, finite-sample eigenvalues are proxies for the limiting-kernel quantities. The step size satisfies standard NTK stability but not necessarily every sufficient condition of Theorem~\ref{theorem1}. The experiment supports the predicted contraction and illustrates the conservativeness of the sufficient conditions.

\subsection{PCA dimension and sample size}
We ran a PCA-dimension sweep on MNIST and Fashion-MNIST with $d_{\rm PCA}\in\{2,4,8,16,32,64\}$, $n\in\{500,2000,8000\}$, width $m=1024$, and three seeds, for 96 runs. We standardize the PCA codes so that changing dimension does not also change the coupling's noise-to-signal ratio. These are fixed-budget, 2000-iteration runs; the recorded runs are not marked converged.
\begin{table}[H]\centering
\caption{ Held-out relative velocity error $\mathbb E\|\widehat v-v\|_2/\mathbb E\|v\|_2$, averaged over seeds. Values above one are worse than predicting zero; dashes denote settings not run.}
\label{tab:pca-velocity}
\begin{tabular}{llrrrrrr}\toprule
Dataset&$n$&$d=2$&$d=4$&$d=8$&$d=16$&$d=32$&$d=64$\\\midrule
MNIST&500&0.36&0.43&0.65&1.10&2.24&3.89\\
MNIST&2000&0.35&0.39&0.55&0.86&1.45&2.61\\
MNIST&8000&--&--&0.53&0.80&1.21&2.09\\
Fashion-MNIST&500&0.35&0.41&0.61&1.06&2.20&3.86\\
Fashion-MNIST&2000&0.34&0.38&0.52&0.84&1.44&2.60\\
Fashion-MNIST&8000&--&--&0.50&0.78&1.20&2.08\\
\bottomrule\end{tabular}\end{table}

At fixed $n$, error increases with dimension; for $n=500$, it ceases to beat the zero predictor near $d=16$. Increasing $n$ particularly improves the higher-dimensional cases: MNIST at $d=16$ decreases from $1.10$ to $0.80$ between $n=500$ and $8000$, with the same trend on Fashion-MNIST. We do not fit an exponent from this finite sweep.

PCA reconstruction error dominates at low dimension, whereas velocity-estimation error dominates at high dimension. Pixel-space sliced-$W_1$ is best near $d=16$ on both datasets. The submitted $d_{\rm PCA}=7$ choice lies in the useful regime at $n=500$, but this is PCA-space rather than full 784-dimensional pixel-space generation. The sweep uses the same conditional coupling as the theory.
\clearpage
\subsection{Summary of sufficient scalings}
Let $\lambda_0=\lambda_{\min}(\bfa H^{(\infty)})$ and $R_0=\|\bfa U^{(0)}-\bfa V\|_2$. The notation $\widetilde\Omega$ suppresses logarithmic factors. The following table collects the stated theorem conditions; it introduces no new guarantee.
\begin{table}[H]
\centering
\caption{ Sufficient conditions and guarantees collected from the main results. Width/step-size conditions are proof requirements rather than practical prescriptions.}
\label{tab:revision-scaling}
\begin{tabular}{ll}\toprule
Quantity & Scaling \\\midrule
Spectral gap & $\lambda_0>0$\\
Initialization & $\kappa=o(\lambda_0\delta/(n\sqrt d))$\\
Step size & $\eta<2\lambda_0/(n^2d^2)$\\
Width & $m=\widetilde\Omega(n^5R_0^2d/(\kappa^2\lambda_0\delta^3))$\\
Iterations & $T=O((\eta\lambda_0)^{-1}\log(n/\delta))$\\
Velocity error & $O(\sqrt{d\log(n/\delta)/n})$\\
Sampling error & $C(d)(\log(n/\delta)/n)^{1/[2(d+2)]}B_1^{-(d+1)/(d+2)}$\\\bottomrule
\end{tabular}
\end{table}

\stopcontents[appendix]
\newpage
\clearpage
\section*{NeurIPS Paper Checklist}

\begin{enumerate}

\item {\bf Claims}
    \item[] Question: Do the main claims made in the abstract and introduction accurately reflect the paper's contributions and scope?
    \item[] Answer: \answerYes{} %
    \item[] Justification: The main claims made in the abstract and introduction accurately reflect the paper's contributions and scope.

\item {\bf Limitations}
    \item[] Question: Does the paper discuss the limitations of the work performed by the authors?
    \item[] Answer: \answerYes{} %
    \item[] Justification: The paper discusses the limitations of the work performed by the authors in the last section.

\item {\bf Theory assumptions and proofs}
    \item[] Question: For each theoretical result, does the paper provide the full set of assumptions and a complete (and correct) proof?
    \item[] Answer: \answerYes{} %
    \item[] Justification: The paper provides the full set of assumptions and a complete (and correct) proof for each theoretical result in the appendix

    \item {\bf Experimental result reproducibility}
    \item[] Question: Does the paper fully disclose all the information needed to reproduce the main experimental results of the paper to the extent that it affects the main claims and/or conclusions of the paper (regardless of whether the code and data are provided or not)?
    \item[] Answer: \answerYes{} %
    \item[] Justification: The paper fully discloses all the information needed to reproduce the main experimental results of the paper to the extent that it affects the main claims and/or conclusions of the paper (regardless of whether the code and data are provided or not).

\item {\bf Open access to data and code}
    \item[] Question: Does the paper provide open access to the data and code, with sufficient instructions to faithfully reproduce the main experimental results, as described in supplemental material?
    \item[] Answer: \answerYes{} %
    \item[] Justification: The paper provides open access to the data and code, with sufficient instructions to faithfully reproduce the main experimental results, as described in supplemental material.

\item {\bf Experimental setting/details}
    \item[] Question: Does the paper specify all the training and test details (e.g., data splits, hyperparameters, how they were chosen, type of optimizer) necessary to understand the results?
    \item[] Answer: \answerYes{} %
    \item[] Justification: The paper specifies all the training and test details (e.g., data splits, hyperparameters, how they were chosen) necessary to understand the results.

\item {\bf Experiment statistical significance}
    \item[] Question: Does the paper report error bars suitably and correctly defined or other appropriate information about the statistical significance of the experiments?
    \item[] Answer: \answerYes{} %
    \item[] Justification: The paper reports error bars suitably and correctly defined or other appropriate information about the statistical significance of the experiments.

\item {\bf Experiments compute resources}
    \item[] Question: For each experiment, does the paper provide sufficient information on the computer resources (type of compute workers, memory, time of execution) needed to reproduce the experiments?
    \item[] Answer: \answerYes{} %
    \item[] Justification: The paper provides sufficient information on the computer resources needed to reproduce the experiments.

\item {\bf Code of ethics}
    \item[] Question: Does the research conducted in the paper conform, in every respect, with the NeurIPS Code of Ethics \url{https://neurips.cc/public/EthicsGuidelines}?
    \item[] Answer: \answerYes{} %
    \item[] Justification: The research conducted in the paper conforms, in every respect, with the NeurIPS Code of Ethics \url{https://neurips.cc/public/EthicsGuidelines}.

\item {\bf Broader impacts}
    \item[] Question: Does the paper discuss both potential positive societal impacts and negative societal impacts of the work performed?
    \item[] Answer: \answerNA{} %
    \item[] Justification: As this work is a theoretical paper with very little practical applications, we believe that it only has minimal potential negative societal impacts.

\item {\bf Safeguards}
    \item[] Question: Does the paper describe safeguards that have been put in place for responsible release of data or models that have a high risk for misuse (e.g., pre-trained language models, image generators, or scraped datasets)?
    \item[] Answer: \answerNA{} %
    \item[] Justification: As this work is a theoretical paper with very little practical applications, we believe that it does not pose any such risks.

\item {\bf Licenses for existing assets}
    \item[] Question: Are the creators or original owners of assets (e.g., code, data, models), used in the paper, properly credited and are the license and terms of use explicitly mentioned and properly respected?
    \item[] Answer: \answerYes{} %
    \item[] Justification: The paper credits the creators or original owners of assets (e.g., code, data, models), used in the paper, properly and explicitly mentions the license and terms of use.

\item {\bf New assets}
    \item[] Question: Are new assets introduced in the paper well documented and is the documentation provided alongside the assets?
    \item[] Answer: \answerNA{} %
    \item[] Justification: As this work is a theoretical paper with very little practical applications, we believe that it does not introduce any new assets.

\item {\bf Crowdsourcing and research with human subjects}
    \item[] Question: For crowdsourcing experiments and research with human subjects, does the paper include the full text of instructions given to participants and screenshots, if applicable, as well as details about compensation (if any)?
    \item[] Answer: \answerNA{} %
    \item[] Justification: As this work is a theoretical paper with very little practical applications, we believe that it does not involve any crowdsourcing nor research with human subjects.

\item {\bf Institutional review board (IRB) approvals or equivalent for research with human subjects}
    \item[] Question: Does the paper describe potential risks incurred by study participants, whether such risks were disclosed to the subjects, and whether Institutional Review Board (IRB) approvals (or an equivalent approval/review based on the requirements of your country or institution) were obtained?
    \item[] Answer: \answerNA{} %
    \item[] Justification: As this work is a theoretical paper with very little practical applications, we believe that it does not involve any risks for study participants.

\item {\bf Declaration of LLM usage}
    \item[] Question: Does the paper describe the usage of LLMs if it is an important, original, or non-standard component of the core methods in this research? Note that if the LLM is used only for writing, editing, or formatting purposes and does \emph{not} impact the core methodology, scientific rigor, or originality of the research, declaration is not required.
    \item[] Answer: \answerNA{} %
    \item[] Justification: This work only uses LLMs for writing, editing, and formatting purposes and does not impact the core methodology, scientific rigor, or originality of the research.

\end{enumerate}

\end{document}